\documentclass[nohyperref]{article}

\usepackage{microtype}
\usepackage{graphicx}
\usepackage{subfigure}
\usepackage{booktabs} 

\usepackage{hyperref}

\usepackage[accepted]{icml2022}

\usepackage{amsmath}
\usepackage{amssymb}
\usepackage{mathtools}
\usepackage{amsthm}
\usepackage{calc}
\usepackage[capitalize,noabbrev]{cleveref}

\newtheorem{theorem}{Theorem}[section]

\newtheorem{lemma}[theorem]{Lemma}

\usepackage[textsize=tiny]{todonotes}

\usepackage{hyperref}       
\usepackage{url}            
\usepackage{booktabs}       
\usepackage{amsfonts}       
\usepackage{nicefrac}       
\usepackage{microtype}      
\usepackage{xcolor}         
\usepackage{graphicx}
\usepackage{subfigure}
\usepackage{caption}
\usepackage{booktabs}
\usepackage{amsmath}
\usepackage{bbm}
\usepackage{amssymb}
\usepackage{mathrsfs}
\usepackage{multirow}
\usepackage{enumitem}
\usepackage{makecell}
\usepackage{wrapfig}

\usepackage{amsthm}
\usepackage{listings}
\usepackage{tikz}

\usetikzlibrary{shapes.geometric, calc,arrows.meta,positioning}

\tikzset{
    rect/.style={rectangle, rounded corners, minimum width=1cm, minimum height=1cm,text centered, draw=black, fill=orange!30},
    every label/.style={draw=none},
}
\tikzstyle{arrow} = [thick,<->,>=stealth]

\newcommand{\ours}{{TransRate }}

\newcommand{\trR}{{\text{TrR}}}

\DeclareMathOperator*{\argmax}{arg\,max}

\DeclareMathOperator*{\logdet}{log\,det}

\newtheorem{defn}{Definition}
\newtheorem{prop}{Proposition}

\newcommand{\kai}[1]{{\color{black}#1}}
\newcommand{\yingicml}[1]{{\color{black}#1}}

\icmltitlerunning{Frustratingly Easy Transferability Estimation}

\begin{document}

\twocolumn[
\icmltitle{Frustratingly Easy Transferability Estimation}



\icmlsetsymbol{equal}{*}

\begin{icmlauthorlist}
\icmlauthor{Long-Kai Huang}{tencent}
\icmlauthor{Junzhou Huang}{tencent}
\icmlauthor{Yu Rong}{tencent}
\icmlauthor{Qiang Yang}{hkust}
\icmlauthor{Ying Wei}{cityu}
\end{icmlauthorlist}

\icmlaffiliation{tencent}{Tencent AI Lab}
\icmlaffiliation{hkust}{Hong Kong University of Science and Technology}
\icmlaffiliation{cityu}{City University of Hong Kong}

\icmlcorrespondingauthor{Ying Wei}{yingwei@cityu.edu.hk}

\icmlkeywords{Machine Learning, ICML}

\vskip 0.3in
]

\printAffiliationsAndNotice{}

\begin{abstract}
Transferability estimation has been an essential tool in selecting a pre-trained model and the layers 
\kai{in} it \yingicml{for transfer learning,}
so as
to maximize the performance on a target task and prevent negative transfer.
Existing estimation algorithms either require intensive training on target tasks or have difficulties in evaluating the transferability between layers.
\yingicml{To this end,} we propose a simple, efficient, and effective transferability measure named TransRate. 
\kai{Through a single pass over} \yingicml{examples of a target task,}
TransRate measures the transferability as the mutual information between 
features of target examples extracted by a pre-trained model and 
\kai{their labels.}
We overcome the challenge of efficient mutual information estimation by resorting to coding rate 
that serves as an effective alternative to entropy.
\yingicml{From the perspective of feature representation, the resulting TransRate evaluates both completeness (whether features contain sufficient information of a target task) and compactness (whether features of each class are compact enough for good generalization) of pre-trained features. Theoretically, we have analyzed the close connection of} TransRate 
to the performance after transfer learning. Despite its extraordinary simplicity in 10 lines of codes, 
TransRate performs remarkably well in extensive evaluations on 
\kai{32} 
pre-trained models and 16 downstream tasks.

\end{abstract}


\section{Introduction}\label{sec:intro}

Transfer learning from standard large datasets (e.g., ImageNet) and corresponding pre-trained models (e.g., ResNet-50) has become a de-facto method for real-world deep learning applications where limited annotated data is 
\yingicml{accessible}.
Unfortunately, the performance gain by transfer learning could vary a lot, 
even with the possibility of negative transfer~\cite{pan2009survey, wang2019characterizing,zhang2020overcoming}. 
First, the \emph{relatedness of the source task} where a pre-trained model is trained on to the target task largely dictates the performance gain.
Second, using pre-trained~models in different \emph{architectures} also leads to uneven performance gain\yingicml{,} even for the same pair of source and target tasks.
Figure~\ref{fig:intro_model} tells that ResNet-50 pre-trained on ImageNet contribute\yingicml{s} the most to the target task CIFAR-100, compared to the other 
architectures.
Finally, the optimal \emph{layers} to transfer 
vary from pair 
to pair. 
While higher layers encode more semantic patterns that are specific to source tasks, lower-layer features are more generic~\cite{yosinski2014transferable}.
Especially if a pair of tasks are not sufficiently similar, determining the optimal layers is expected to strike the balance between transferring only lower-layer features (as higher layers specific to a source task may hurt the performance of a target task) and transferring more higher-layer features (as training more higher-layers from scratch requires extensive labeled data).
As shown in Figure~\ref{fig:intro_layer}, not transferring the three highest layers is preferred for training with full target data, though transferring all but
the two highest layers achieves the highest test accuracy with scarce target data. This suggests the following

\textbf{Question:} \emph{Which pre-trained model (possibly trained on different source tasks in a supervised or unsupervised manner) and which layers of it should be transferred to benefit the target task the most?}

This research question drives the design of transferability estimation methods, including computation-intensive~\cite{achille2019task2vec,dwivedi2019representation,song2020depara,zamir2018taskonomy} and computation-efficient ones~\cite{bao2019information,cui2018large,nguyen2020leep,tran2019transferability,you2021logme}.
The pioneering works~\cite{achille2019task2vec,zamir2018taskonomy}
directly follow the definition of transfer learning to measure the transferability, and 
thereby require 
fine-tuning on a target task with expensive parameter optimization.
Though their follow-ups~\cite{dwivedi2019representation,song2020depara} alleviate the need of fine-tuning, their prerequisites still include an encoder pre-trained on 
target task\yingicml{s}. 
Keeping in mind that the primary goal of a transferability measure is to select a pre-trained model prior to training on a target task, researchers recently turned towards computation-efficient ways.
The transferability is~estimated as the negative conditional entropy between the labels of the two tasks in~\cite{tran2019transferability,nguyen2020leep}. 
\citet{bao2019information} and~\citet{you2021logme} solved two surrogate optimization problems to estimate the likelihood and marginalized likelihood of labeled target examples, under the assumption that a linear classifier is added on top of the pre-trained model. However, the 
efficiency comes at the price of failing to discriminate 
transferability between layers, for 
\cite{nguyen2020leep,tran2019transferability} that estimate with labels only and~\cite{bao2019information,you2021logme} that consider transferring 
the penultimate layer only.  

\begin{figure}[t]
\centering
    \subfigure[{\small Test accuracy for pre-training ImageNet with different model architectures.} ]{\label{fig:intro_model}\includegraphics[width=0.22\textwidth]{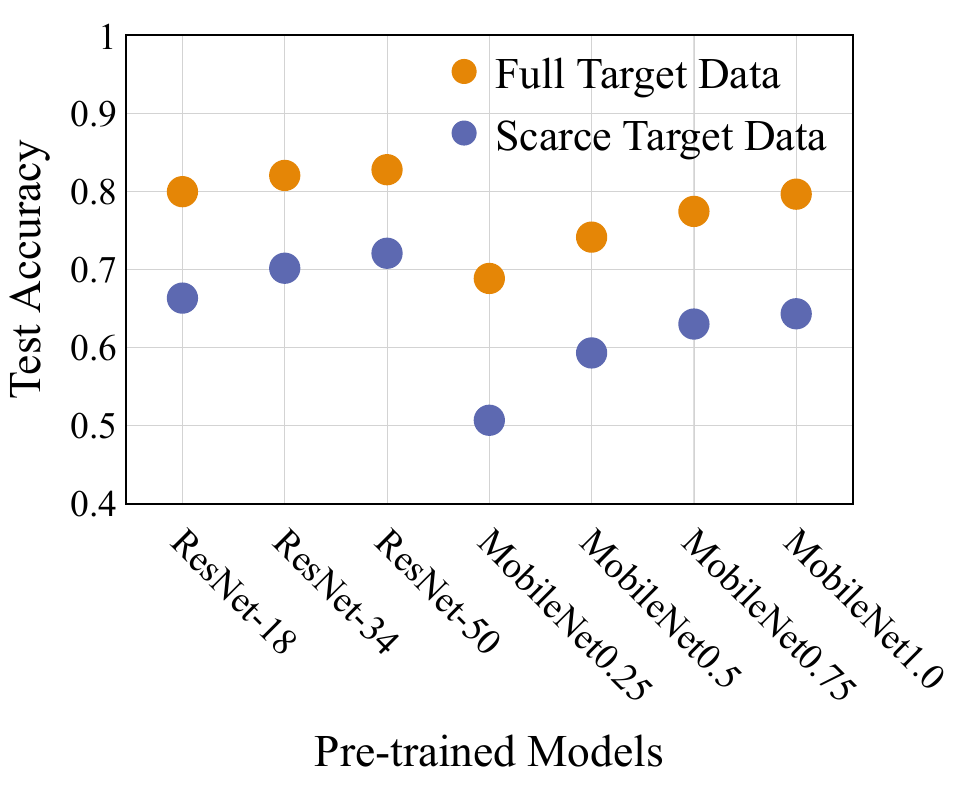} }
    \subfigure[\small Test accuracy for transferring different layers of the pre-trained ResNet34 model.
    ]{\label{fig:intro_layer}\includegraphics[width=0.22\textwidth]{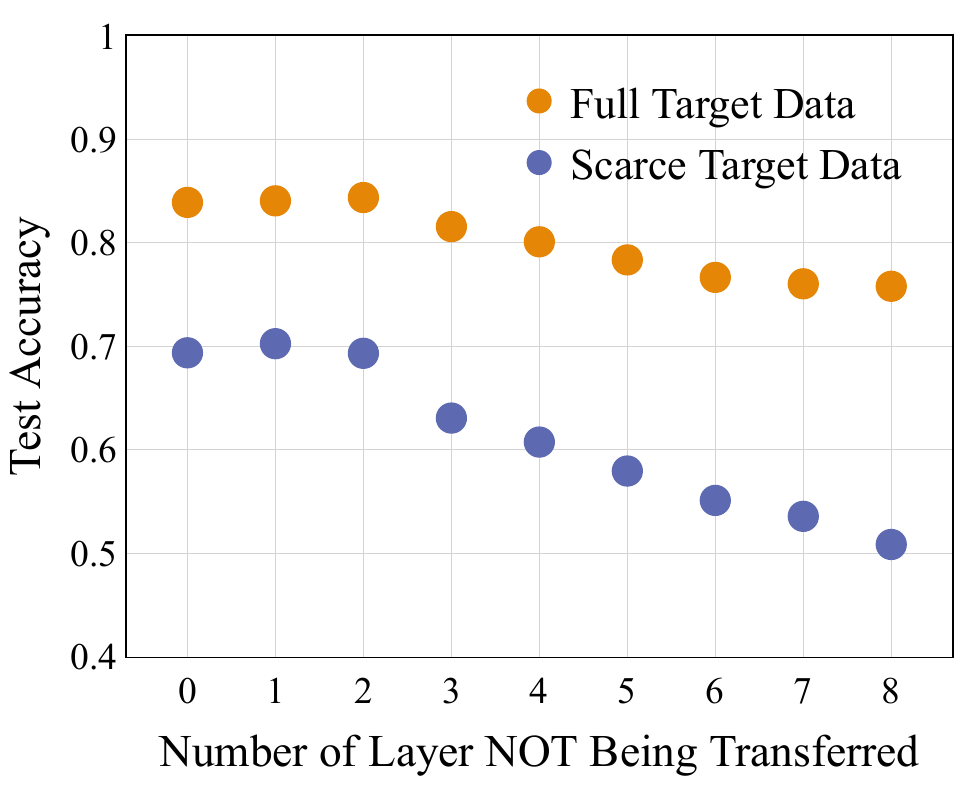} }
	\vspace{-3mm}\caption{\small  Transferring from ImageNet to CIFAR-100. For ``Full Target Data'', all target data 
	are used, while for ``Scarce Target Data'', only 50 target samples per class are used in training.}\vspace{-4mm}
\label{fig:intro}
\end{figure}

We are motivated to pursue a 
\yingicml{computation-efficient}
transferability measure 
\yingicml{without sacrificing the merit of}
computation-intensive 
\yingicml{methods in comprehensive transferability evaluation, especially between layers. 
Mutual information between features and labels has a strong predictive role in the effectiveness of feature representation, dated back to the decision tree algorithm~\cite{quinlan1986induction} and also evidenced in recent studies~\cite{tishby2015deep}.
Markedly, mutual information varies from layer to layer of features given labels, which makes itself an attractive measure for transferability.
In this paper, we propose to estimate the transferability with the mutual information between labels and features of target examples extracted by a pre-trained model at a specific layer.}
Though mutual information itself is notoriously challenging to estimate~\cite{hjelm2018learning}, we overcome this obstacle by resorting to the coding rate proposed in~\cite{ma2007segmentation} inspired from the close connection between rate distortion and entropy in information theory~\cite{cover1999elements}.
\yingicml{The resulting estimation named TransRate}
offers the following advantages: 1) it perfectly matches our need in computation efficiency, free of either prohibitively exhaustive discretization~\cite{tishby2015deep} or neural network training~\cite{belghazi2018mutual}; 2) it is well defined for finite examples from a subspace-like distribution, even if the examples are represented in a high dimensional feature space of a pre-trained model.
In a nutshell, TransRate is simple and efficient, as the only computations it requires are (1) making a single forward pass of the model pre-trained on a source task through the target examples to obtain their features at a set of selected layers and (2) calculating the 
\yingicml{TransRate}.

Despite being ``frustratingly easy'', TransRate enjoys the following benefits that we would highlight.
\begin{itemize}[leftmargin=0.15in,noitemsep]
\vspace{-10pt}
	\item 
	Tra\yingicml{n}sRate allows to select 
	\yingicml{between}
	layers of 
	\yingicml{a pre-trained}
	model 
	\yingicml{for better transfer learning.}
	\item 
    \yingicml{We have theoretically analyzed that TransRate closely aligns} with the performance after transfer learning.
    \item
    \yingicml{A larger value of TransRate is associated with more complete and compact features that are strongly suggestive of better generalization.}
	\item 
	TransRate offers surprisingly good performance on 
	\yingicml{transferability comparison}
	between source tasks, \yingicml{between} 
	architectures, and \yingicml{between} layers. We investigate a total~of 
	\kai{32} pre-trained models (including supervised, unsupervised, \kai{self-supervised,} convolutional and graph neural networks), 16 downstream tasks (including classification and regression), and 
	\yingicml{the}
	insensitivity \yingicml{of TransRate} against the number of 
	\yingicml{labeled}
	examples 
	in a target task.
\end{itemize}

\begin{table*}[t]\caption{Summary of the existing transferability measures and ours.}
\vspace{-0.1in}
\small
\center
\resizebox{0.99\textwidth}{!}{
\begin{tabular}{p{0.24\linewidth} | p{0.11\linewidth}| p{0.11\linewidth}| p{0.11\linewidth}| p{0.18\linewidth} | p{0.13\linewidth} }
\toprule
\makecell[t]{\;\\\;\\Measures} 
& \makecell[t]{Free of\\ Training on \\ Target}
& \makecell[t]{Free of\\ Assessing \\ Source}
& 
\makecell[t]{Free of\\ Optimization}
& 
\makecell[t]{Applicable to \\Unsupervised\\ Pre-trained Models} & 
\makecell[t]{Applicable to\\ Layer Selection}
\\
\midrule
Taskonomy~\cite{zamir2018taskonomy} & \makecell{$\times$} & \makecell{$\times$} & \makecell{$\surd$} & \makecell{$\surd$}  & \makecell{$\surd$}
\\
Task2Vec~\cite{achille2019task2vec} & \makecell{$\times$} & \makecell{$\times$} & \makecell{$\times$} & \makecell{$\surd$} &  \makecell{$\times$}
\\
RSA~\cite{dwivedi2019representation}  & \makecell{$\times$} & \makecell{$\surd$} & \makecell{$\surd$} & \makecell{$\surd$} & \makecell{$\surd$}
\\
DEPARA~\cite{song2020depara} & \makecell{$\times$} & \makecell{$\surd$} & \makecell{$\surd$} & \makecell{$\surd$} & \makecell{$\surd$}
\\
$\mathcal{N}$LEEP~\cite{li2021ranking} & \makecell{$\times$} &  \makecell{$\surd$} & \makecell{$\surd$} & \makecell{$\surd$} & \makecell{$\surd$}
\\
\midrule
DS~\cite{cui2018large} & \makecell{$\surd$} & \makecell{$\times$} & \makecell{$\times$} & \makecell{$\surd$} & \makecell{$\times$}
\\
\cite{zhang2021quantifying} & \makecell{$\surd$} & \makecell{$\times$}  & \makecell{$\times$}  & \makecell{$\times$} & \makecell{$\times$}
\\
\cite{tong2021mathematical} & \makecell{$\surd$} & \makecell{$\times$}  & \makecell{$\times$}  & \makecell{$\times$} & \makecell{$\times$}
\\
NCE~\cite{tran2019transferability} & \makecell{$\surd$} & \makecell{$\times$}  & \makecell{$\surd$}  & \makecell{$\times$} & \makecell{$\times$}
\\
\midrule
H-Score~\cite{bao2019information} & \makecell{$\surd$} & \makecell{$\surd$} & \makecell{$\times$}  & \makecell{$\surd$} & \makecell{$\times$}
\\
LogME~\cite{you2021logme} & \makecell{$\surd$} & \makecell{$\surd$} & \makecell{$\times$} & \makecell{$\surd$}  & \makecell{$\times$}
\\
\midrule
LEEP~\cite{nguyen2020leep} & \makecell{$\surd$} &  \makecell{$\surd$} & \makecell{$\surd$} & \makecell{$\times$} & \makecell{$\times$}
\\
\ours &\makecell{$\surd$}  & \makecell{$\surd$} & \makecell{$\surd$} & \makecell{$\surd$}  & \makecell{$\surd$}
\\
\bottomrule
\end{tabular}
}
\label{tab:0}
\vspace{-0.1in}
\end{table*}

\vspace{-10pt}

\section{Related Works}\label{sec:related}

Re-training or fine-tuning a pre-trained model is a simple yet effective strategy in transfer learning~\cite{pan2009survey}. To improve the performance on a target task and avoid negative transfer, there have been various works on transferability estimation between tasks~\cite{achille2019task2vec,bao2019information,cui2018large,dwivedi2019representation,nguyen2020leep,song2020depara,tran2019transferability,zamir2018taskonomy,li2021ranking}, which we summarize them in Table~\ref{tab:0}.
Taskomomy~\cite{zamir2018taskonomy} and Task2Vec~\cite{achille2019task2vec} evaluate the task relatedness by the loss and the Fisher Information Matrix after fully performing fine-tuning of the pre-trained model on the target task, respectively. In RSA~\cite{dwivedi2019representation} and DEPARA~\cite{song2020depara}, the authors proposed to build a similarity graph between examples for each task based on representations by a pre-trained model on this task, and took the graph 
similarity across tasks as the transferability.
Despite their general applicability in using unsupervised pre-trained models besides supervised ones and selecting the layer to transfer, their computational costs that are as high as fine-tuning with target labeled data exclude their applications to meet the urgent need of transferabilty estimation prior to fine-tuning.

\kai{There also exist transferability measures proposed for domain generalization~\cite{zhang2021quantifying} and multi-source transfer~\cite{tong2021mathematical}, \yingicml{where} a special class of integral probability metric between domains \yingicml{and} the optimal combination coefficients of source models that minimizes the $\chi^2$ between the combined source distribution and the target distribution \yingicml{ were proposed, respectively. Both of them stand in need of source datasets; however, we focus on evaluating the transferability of various pre-trained models, where the source dataset that a pre-trained model is trained on is oftentimes too huge and private to access.}}

This work is more aligned with recent attempts towards computationally efficient transferability measures without training on target data~\cite{bao2019information,cui2018large,nguyen2020leep,tran2019transferability,you2021logme}.
The Earth Mover's Distance between features of the source and the target is used in~\cite{cui2018large}.
\citet{tran2019transferability} proposed the NCE score to estimate the transferability by the negative conditional entropy between labels of a target and a source task.
But alas, 
\yingicml{the reliance on source datasets again disable}
these two methods \yingicml{towards assessing the transferability of a broad range of pre-trained models}. 
To bypass the limitations,
\citet{bao2019information} and~\citet{you2021logme} proposed to directly estimate the likelihood and the marginalized likelihood of labeled target examples, respectively, by assuming that a linear classifier is added on top of the pre-trained model. 
\citet{nguyen2020leep} 
proposed the LEEP score,  where source labels used in NCE~\cite{tran2019transferability} are replaced with soft labels generated by the pre-trained model. \kai{\yingicml{Its extension~}\cite{li2021ranking} 
computes  
\yingicml{more accurate soft source labels via a fitted}
a Gaussian mixture model (GMM), 
\yingicml{at the undesirable cost of training the GMM on the target set similar to computation-intensive methods.}}
Neither of the three, however, is designed for layer selection -- H-Score~\cite{bao2019information} and LogME~\cite{you2021logme} consider the penultimate layer to be transferred only and LEEP estimating the transferability with labels only fails to differentiate by layers. 
Our purpose of the proposed TransRate is a simple but effective transferability measure: 1) it is optimization-free with single forward pass, without solving optimization problems as in~\cite{bao2019information,you2021logme}; 2) besides selecting the source and the architecture of a pre-trained model, it supports layer selection to fill the gap in computationally-efficient measures.

\section{TransRate}~\label{sec:method}
\vspace{-0.3in}

\subsection{Notations and Problem Settings}
We consider the knowledge transfer from a source task  $T_s$ to a target task $T_t$ of $C$-category classification. 
As widely accepted, only the model that is pre-trained on the source task, instead of source data, is accessible. The pre-trained model, denoted by 
$F \!=\! f_{L+1} \circ\! ...\! \circ (f_2 \!\circ\! f_1)$,
consists of an $L$-layer feature extractor and a $1$-layer classifier $f_{L+1}$. Here, $f_l$ is the mapping function at the $l$-th layer. The target task is represented by $n$ labeled data samples $\{(x_i, y_i)\}_{i=1}^n$.
Afterwards, we denote the number of layers to be transferred by $K$ ($K\!\le\! L$). These $K$ layers of the model are named as the pre-trained feature extractor $g\! =\! f_K \circ ... \circ (f_2 \!\circ\! f_1)$.
The feature of $x_i$ extracted by $g$ is denoted as $z_i \!=\! g(x_i)$.
Building on the feature extractor, we construct the target model denoted by \yingicml{$w$} to include 1) the same structure as the $(K+1)$-th to $(L)$-th layers of the source model and 2) a new classifier $f^t_{L+1}$ for the target task. We also refer \kai{to} \yingicml{$w$} as the head of the target model. Following the standard practice of fine-tuning, both the feature extractor $g$ and the head \yingicml{$w$} will be trained on the target task.

We consider the optimal model for the target task as 
\small
\begin{align}
    g^*,&\yingicml{w}^* \!=\!  \vspace{-0.5in}
    &\argmax_{\tilde{g}\in\mathcal{G},\yingicml{w} \in \mathcal{\yingicml{W}}}\mathcal{L}(\tilde{g},\yingicml{w}) \!=\!  \argmax_{\tilde{g}\in\mathcal{G},\yingicml{w} \in \mathcal{\yingicml{W}}}  \frac{1}{n} \sum_{i=1}^n \log p(y_i|z_i; \tilde{g},\yingicml{w}) \nonumber
\end{align}
\normalsize
subject to $\tilde{g}^{(0)}=g$, where $\mathcal{L}$ denotes the log-likelihood, and $\mathcal{G}$ and $\mathcal{W}$ are the spaces of all possible feature extractors and heads, respectively. 
We define the transferability as the expected log-likelihood of the optimal model $\yingicml{w}^*\circ g^*$ on test samples in the target task:

\begin{defn}[Transferability]\label{def:1}
    The transferability of a pre-trained feature extractor $g$ from a source task $T_s$ to a target task $T_t$, denoted by $\mathrm{Trf}_{T_s \rightarrow T_t}(g)$, is measured by the expected log-likelihood of the optimal model $\yingicml{w}^*\circ g^*$ on a random test sample $(x, y)$ of $T_t$: $\mathrm{Trf}_{T_s \rightarrow T_t}(g) := \mathbbm{E}[\log p(y|\yingicml{z^*}; g^*, \yingicml{w}^*)]$ where $\yingicml{z^*}=g^*(x)$.
\end{defn}

This definition of transferability can be used for 1) selection of a pre-trained feature extractor 
among a model zoo $\{g_m\}_{m=1}^{M}$ for a target task, where 
\yingicml{$M$}
pre-trained models could be in different architectures and trained on different source tasks in a supervised or unsupervised manner, and 2) selection of a layer  to transfer among all configurations $\{g_m^l\}_{l=1}^{K}$ given a pre-trained model $g_m$ and a target task.

\subsection{Computation-Efficient Transferability Estimation}

Computing the transferability defined in Definition~\ref{def:1} is as prohibitively expensive as fine-tuning all $M$ pre-trained models or $K$ layer configurations of a pre-trained model \yingicml{on} 
the target task, while the transferability offers benefits only when it can be calculated a priori.
To address this shortfall, we propose TransRate, a frustratingly easy measure, to estimate the defined transferability.
The transferability characterizes how well the optimal model, composed of the feature extractor $g^*$ initialized from $g$ and the head $\yingicml{w}^*$, performs on the target task, where the performance is evaluated by the log-likelihood.
However, the optimal model $\yingicml{w}^*\!\circ\! g^*$ is inaccessible without optimizing $\mathcal{L}(\tilde{g},\yingicml{w})$.
For tractability, we follow prior computation-efficient 
transferability measures~\cite{nguyen2020leep,you2021logme} to estimate the performance of $\yingicml{w}^*\circ g$ instead.
By reasonably assuming that \yingicml{$w^*$} can extract all the information related to the target task in the pre-trained feature extractor $g$, we argue that the mutual information between the pre-trained feature extractor $g$ and the target task serves as a strong indicator for the performance of the model
$\yingicml{w}^*\circ g$.
Therefore, the proposed TransRate measures this mutual information as,
\begin{equation}
    \small
    \mathrm{TrR}_{T_s \rightarrow T_t}(g) = \yingicml{h}(Z) - \yingicml{h}(Z|Y) \approx H(Z^\Delta) - H(Z^\Delta|Y),
\label{eqn:transrate}
\end{equation}
where $Y$ are labels of target examples, $Z\!=\!g(X)$ \yingicml{and $Z^\Delta$} are features \yingicml{and quantized features} of them extracted by the pre-trained feature extractor $g$ . 
\yingicml{Eqn.~(\ref{eqn:transrate}) follows $h(Z)\!\approx\! H(Z^\Delta)\!+\!log\Delta$($\Delta\!\to\! 0$)~\cite{cover1999elements}, where}
\kai{$H(\cdot)$ \yingicml{denotes}
the Shannon entropy of a discrete random variable \yingicml{(e.g., $Z^\Delta$ with the quantization error $\Delta$)}, and 
\yingicml{$h(\cdot)$} is
the differential entropy of a continuous random variable \yingicml{(e.g., $Z$)}.}

\begin{figure}[b]
    \centering
    \begin{tikzpicture}{-{Stealth}}
    \node[rect, label={[label distance=-18mm]90:Shannon entropy}] (n1) at (0,0) {$\lim\limits_{\Delta\to 0} H(Z^{\Delta})$};
    \node[rect, right= 2cm of n1, label={[label distance=-17.5mm]90:Rate distortion}] (n2) {$R(Z,\epsilon)$};
    \node[rect, right=of n2, label={[label distance=-18.4mm]90:Coding rate}] (n3) {$R(\hat{Z},\epsilon)$};
    \draw [arrow] (n1) -- (n2) node[midway, above] {$\Delta=\sqrt{2\pi e\epsilon}$ };
    \draw [arrow] (n1) -- (n2) node[midway, below] {(a)};
    \draw [arrow] (n2) -- (n3) node[midway, above] {$\hat{Z}$};
    \draw [arrow] (n2) -- (n3) node[midway, below] {(b)};
\end{tikzpicture}
    \caption{\yingicml{Illustration of the relationship between the three information measures: 
    (a) the rate distortion of a continual random variable amounts to $ H(Z^{\sqrt{2\pi e\epsilon}}) + o(1)$ when $\epsilon\!\to\! 0$~\cite{binia1974epsilon}, where a larger $\epsilon$ introduces an approximation error; (b) the coding rate provides an empirical estimate of the rate distortion, where the approximation error is dictated by the degree to which finite samples $\hat{Z}$ represent the true random variable $Z$. 
    }}
    \label{fig:measure_relationship}
\end{figure}

Based on the theory in \cite{qin2019rethinking}, \kai{we show in Proposition~\ref{prop:1} that \ours provides an upper bound \kai{and lower bound} to the log-likelihood of the model $\yingicml{w}^*\circ g$. 
}

\begin{prop}\label{prop:1}
Assume the target task has a uniform label distribution, i.e.\yingicml{,} $p(Y=y^c) = \frac{1}{C}$ holds for all $c=1,2, ...,C$. We then have:
\small 
\begin{equation*}
\begin{split}
    & \mathcal{L}(g, h^*) \lessapprox \mathrm{TrR}_{T_s \rightarrow T_t}(g) 
    - H(Y), \\
    & \mathcal{L}(g, h^*) \gtrapprox \mathrm{TrR}_{T_s \rightarrow T_t}(g )
    - H(Y) -\yingicml{H(Z^\Delta)}. 
\end{split}
\end{equation*}
\normalsize
\end{prop}

\kai{Note that NCE, LEEP and \ours all provide a 
lower bound for the maximal log-likelihood,
\yingicml{whereas}
only \ours has been shown to be a tight upper bound of the maximal log-likelihood. Since the
maximal log-likelihood is closely related to the 
\yingicml{transfer learning performance, this}
proposition implies that \ours 
highly align\yingicml{s} with the transfer performance. 
A detailed proof and more analysis on the relationship between \ours and transfer 
\yingicml{performance}
can be found in Appendix~\ref{appsec:d1} 
and Appendix~\ref{appsec:c}.} 

\begin{figure*}[t!]
\centering
    \subfigure[][\tiny \makecell{ $R(\hat{Z},\yingicml{0.01})\!\approx\!6.01$, \\ $-R(\hat{Z},\yingicml{0.01}|Y)\!\approx\!-5.35$, \\ $\mathrm{TrR}(\yingicml{g},\yingicml{0.01})\!\approx\!0.66$}]{\label{fig:classa}\includegraphics[width=0.235\textwidth]{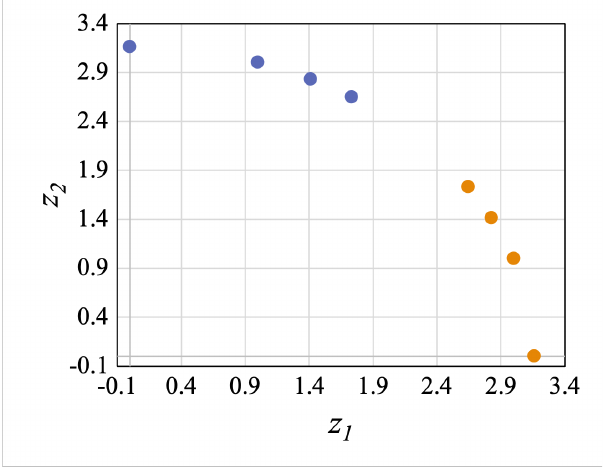} }
    \subfigure[][\tiny \makecell{ $R(\hat{Z},\yingicml{0.01})\!\approx\!6.01$, \\ $-R(\hat{Z},\yingicml{0.01}|Y)\!\approx\!-5.88$, \\ $\mathrm{TrR}(\yingicml{g},\yingicml{0.01})\!\approx\!0.13$}]{\label{fig:classb}\includegraphics[width=0.235\textwidth]{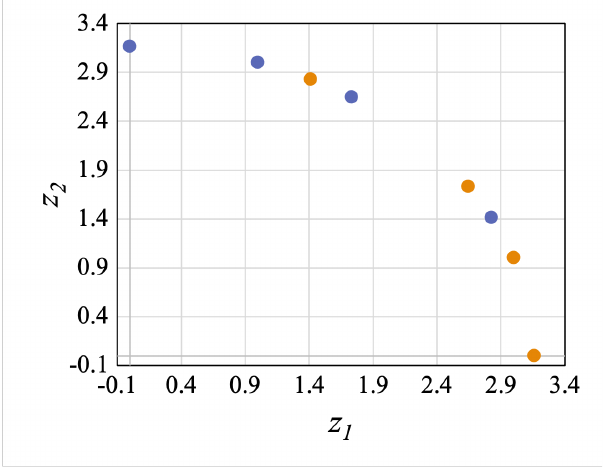} }
    \subfigure[][\tiny \makecell{ $R(\hat{Z},\yingicml{0.01})\!\approx\!5.90$, \\ $-R(\hat{Z},\yingicml{0.01}|Y)\!\approx\!-5.35$, \\ $\mathrm{TrR}(\yingicml{g},\yingicml{0.01})\!\approx\!0.54$}]{\label{fig:classc}\includegraphics[width=0.235\textwidth]{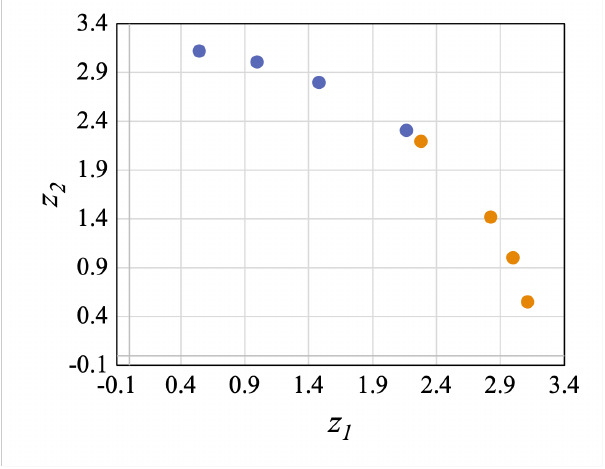} }
    \subfigure[][\tiny \makecell{ $R(\hat{Z},\yingicml{0.01})\!\approx\!5.90$, \\ $-R(\hat{Z},\yingicml{0.01}|Y)\!\approx\!-5.80$, \\ $\mathrm{TrR}(\yingicml{g},\yingicml{0.01})\!\approx\!0.10$}]{\label{fig:classd}\includegraphics[width=0.235\textwidth]{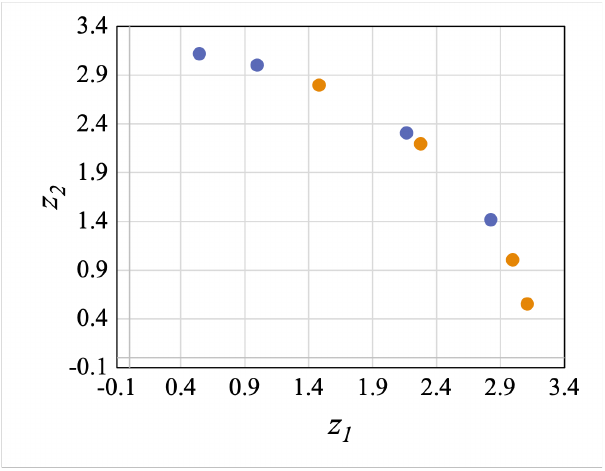} }
    \vspace{-3mm}\caption{\small Toy examples illustrating the effectiveness of the TransRate. The horizontal and vertical axes represent the two dimensions of features $\hat{Z}$. There are two classes in $Y$, pictorially illustrated with two colors. 
    }
\label{fig:class}
\end{figure*}

Computing the TransRate in \yingicml{Eqn.}~(\ref{eqn:transrate}), however, remains a daunting challenge, as the mutual information is notoriously difficult to compute especially for continuous variables in high-dimensional settings~\cite{hjelm2018learning}.
A popular solution for mutual information estimation is to 
\yingicml{have the quantization $Z^\Delta$ via}
the histogram method~\cite{tishby2015deep},
though it requires an extremely large memory capacity. Even if we  divide each dimension of $Z$ into only $10$ bins,  there will be $10^d$ bins where $d$ is the dimension of $Z$ that is usually greater than 128. Other estimators include kernel density estimator~
\cite{moon1995estimation} and k-NN estimator~\cite{beirlant1997nonparametric,kraskov2004estimating}. 
KDE suffers from singular solutions when the number of examples is smaller than their dimension; the 
$k$-NN estimator requiring exhaustive computation of nearest neighbors of all examples may be too computationally expensive if more examples are available. 
Recent trends in deep neural networks have 
\kai{led} to a proliferation studies of approximating the mutual information or entropy by a neural network~\cite{belghazi2018mutual,hjelm2018learning,shalev2020neural} and obtain\kai{ing} 
\kai{a} high-accuracy estimation by optimizing the neural network.
Unfortunately, training neural networks is contrary to our premise of an optimization-free transferability measure.


Fortunately, \yingicml{as shown in Figure~\ref{fig:measure_relationship}}, the rate distortion $R(Z,\epsilon)$ \yingicml{defining}
the minimal number of binary bits to encode $Z$ with an expected decoding error less than $\epsilon$ 
\yingicml{has been proved to be closely related to the Shannon entropy, i.e., $R(Z,\epsilon)\!=\!H(Z^{\Delta}) \!+\! o(1)$ with $\Delta=\sqrt{2\pi e\epsilon}$ when $\epsilon\!\to\! 0$~\cite{binia1974epsilon,cover1999elements}.}
Most crucially, the work of~\cite{ma2007segmentation} offers \yingicml{the coding rate $R(\hat{Z},\epsilon)$ as} an efficient and accurate \yingicml{empirical} estimation 
\yingicml{to $R(Z,\epsilon)$, provided with}
$n$ finite samples $\hat{Z} = [z_1, z_2, ..., z_n]\in \mathbbm{R}^{d \times n}$ from a subspace-like distribution \yingicml{where $d$ is the dimension of $z_i$}. 
Concretely, 
\begin{equation}
\small
    R(\hat{Z}, \epsilon) = \frac{1}{2} \logdet (I_d + \frac{1}{n \yingicml{\epsilon} 
    } \hat{Z} \hat{Z}^\top),
\label{eqn:R_z_1}
\end{equation}
where 
$\epsilon$ is the distortion rate.
\yingicml{Coding rate has been verified to be qualified even for samples that} are in high-dimensional feature representations  or from a non-Gaussian distribution~\cite{ma2007segmentation}, \yingicml{which is often the case for features by deep neural networks. 
Therefore, we resort to 
$R(\hat{Z},\epsilon)$ as an approximation to $H(Z^\Delta)$ ($\Delta=\sqrt{2\pi e\epsilon}$) with a small value of $\epsilon$.
}
More properties of 
\yingicml{the coding rate}
will be discussed in \kai{Appendix~\ref{appsec:c1} 
and Appendix~\ref{appsec:d2}.} 

Next we investigate the rate distortion estimate $R(\hat{Z},\epsilon|Y)$ as an 
\yingicml{approximation}
to the second component of TransRate, i.e., 
\yingicml{$H(Z^{\Delta}|Y)$}.
Define ${Z}^c = \{z | Y =y^c\}$ as the random variable for features of the target samples in the $c$-th class, whose labels are all $y^c$. 
When $\epsilon\to 0$, we then have

\vspace{-0.2in}
\small
\begin{align}
\label{eqn:h_z_y}
    &\yingicml{H(Z^{\Delta}|Y) \approx h(Z|Y)-\log{\Delta}} \nonumber \\
 = & -\! \int\displaylimits_{z\in Z} \sum_{y \in Y}  p(z, y) \log p(z | y) \mathrm{d} z  \yingicml{-\log{\Delta}}\nonumber \\
     =&-\!\! \int\displaylimits_{z\in Z} \!\sum_{c=1}^C  p(Y \!=\! y^c) p(z| Y \!=\! y^c) \log p(z | Y\!=\!y^c) \mathrm{d} z \yingicml{-\log{\Delta}} \nonumber \\
     = & \!\sum_{c=1}^C \frac{n_c}{n} \int\displaylimits_{z\in Z^c} - p(z) \log p(z) \mathrm{d} z \yingicml{-\log{\Delta}} \nonumber \\
     = & \!\sum_{c=1}^C \frac{n_c}{n} \yingicml{[h}(Z^c)\yingicml{-\log{\Delta}] \!=\! \sum_{c=1}^C \frac{n_c}{n} H((Z^c)^{\Delta})}, 
\end{align}
\normalsize 
where $n_c$ is the number of training samples in the $c$-th class. 
According to~(\ref{eqn:h_z_y}), it is direct to derive
\begin{equation}
\small
\begin{split}
    R(\hat{Z},\epsilon|Y) 
    & =  \sum_{c=1}^C \frac{n_c}{n} R(\hat{Z}^c, \epsilon) \\
    & = \sum_{c=1}^C \frac{n_c}{2n} \logdet (I_d + \frac{1}{n_c \yingicml{\epsilon}}
    \hat{Z}^c {\hat{Z}^c}{}^\top),
\label{eqn:R_z_2}
\end{split}
\end{equation}
\normalsize 
where $\hat{Z}^c \!=\! [z_1^c, z_2^c, \!...,\!z_{n_c}^c]\in \mathbbm{R}^{d \times n_c}$ denotes $n_c$ samples in the \yingicml{$c$}-th class. Combining~(\ref{eqn:R_z_1}) and~(\ref{eqn:R_z_2}), we conclude with the TransRate 
we use in practice for transferability~estimation:
\begin{equation}
\small
    \mathrm{TrR}_{T_s \rightarrow T_t}(g,\epsilon) = R(\hat{Z},\epsilon) - R(\hat{Z},\epsilon|Y). \nonumber
\label{eqn:transrate_2}
\end{equation}
\yingicml{Note that} we 
\yingicml{use}
$\mathrm{TrR}_{T_s \rightarrow T_t}(g)$ and $\mathrm{TrR}_{T_s \rightarrow T_t}(g,\epsilon)$ to denote the ideal and the working TransRate, respectively.

\textbf{Completeness and Compactness}  
We argue that 
those pre-trained models that 
produce both 
complete and compact~features \yingicml{tend to have high TransRate scores}. (1) \emph{Completeness:} \begin{small}$R(\hat{Z},\epsilon)$\end{small} as the first term of \begin{small}$\mathrm{TrR}_{T_s \rightarrow T_t}(g,\epsilon)$\end{small} evaluates whether the features $\hat{Z}$ by the pre-trained feature extractor $g$ include sufficient information for solving the target task -- features between different classes of examples 
\yingicml{should}
be as diverse as possible.
$\hat{Z}$ in Figure~\ref{fig:classa} is more diverse than that in Figure~\ref{fig:classc}, evidenced by a larger value of 
\begin{small}$R(\hat{Z},\yingicml{0.01})$\end{small}. 
(2) \emph{Compactness}: The second term, i.e., \begin{small}$-R(\hat{Z},\epsilon| Y)$\end{small}, assesses whether the features $\hat{Z}^c$ for each $c$-th class are compact enough for good generalization. 
Each of the two classes spans a  wider range in Figure~\ref{fig:classb} than that in  Figure~\ref{fig:classa}, so that the value of \begin{small}$-R(\hat{Z},\yingicml{0.01}| Y)$\end{small} 
is smaller.

Furthermore, there is theoretical evidence to strengthen the argument above. 
Consider a binary classification problem
with \begin{small}$\hat{Z}\! =\! [\hat{Z}^1, \hat{Z}^2]\!\in\!\mathbb{R}^{d\times n}$\end{small}, where both  $\hat{Z}^1$ and $\hat{Z}^2$ have $n$ $d$-dim examples. By defining 
\begin{small}$\alpha \!=\!  1/n\epsilon$\end{small},~we have 
\begin{small}
$\mathrm{TrR}_{T_s \rightarrow T_t}(g,\epsilon)=\frac{1}{2}\log\det\{(I_{n/2} +\alpha (\hat{Z}^1)^\top \hat{Z}^1 +\alpha (\hat{Z}^2)^\top \hat{Z}^2) +\alpha^2 [(\hat{Z}^1)^\top \hat{Z}^1 (\hat{Z}^2)^\top \hat{Z}^2 - (\hat{Z}^1)^\top \hat{Z}^2 (\hat{Z}^2)^\top \hat{Z}^1] \} - B$
\end{small}
where \begin{small}$B = \frac{1}{2} (R(\hat{Z}_1, \epsilon) + R(\hat{Z}_2, \epsilon)) $\end{small}. 
We assume \begin{small}$(\hat{Z}^1)^\top \hat{Z}^1$\end{small} and \begin{small}$(\hat{Z}^2)^\top \hat{Z}^2$\end{small} to be fixed,  so that 
 TransRate maintaining the \emph{compactness} within each class
 \yingicml{(i.e., $B$ is a constant)}
 and hinging on the \emph{completeness} only maximizes at \begin{small}$(\hat{Z}^1)^\top \hat{Z}^2 \!=\! 0$\end{small} and minimizes at \begin{small}$\hat{Z}_1 \!=\! \hat{Z}_2$\end{small}.
That is, TransRate favors the diversity between different classes, while penalizes if the overlap between classes is high.
\kai{Detailed proof and more theoretical analysis about the completeness and compactness can be found in Appendix~\ref{appsec:d3}.} 


\begin{figure*}[t!]
\centering
    \subfigure[][\tiny \makecell{ NCE on CIFAR\\ $R_p\!=\!0.3803$, \\ $\tau_K\!=\!0.3091$, \\ $\tau_\omega\!=\!0.5680$}]{\label{fig:s11}\includegraphics[width=0.15\textwidth]{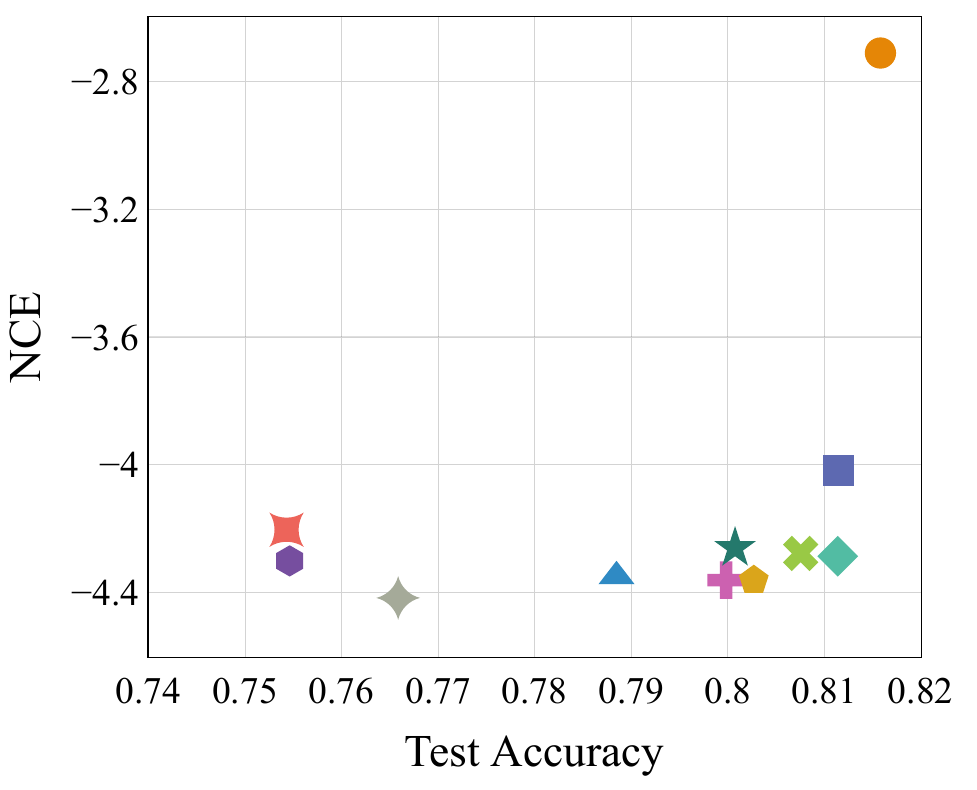} }
    \subfigure[][\tiny \makecell{ LEEP on CIFAR \\ $R_p\!=\!0.2883$, \\ $\tau_K\!=\!0.0909$, \\ $\tau_\omega\!=\!0.3692$}]{\label{fig:s12}\includegraphics[width=0.15\textwidth]{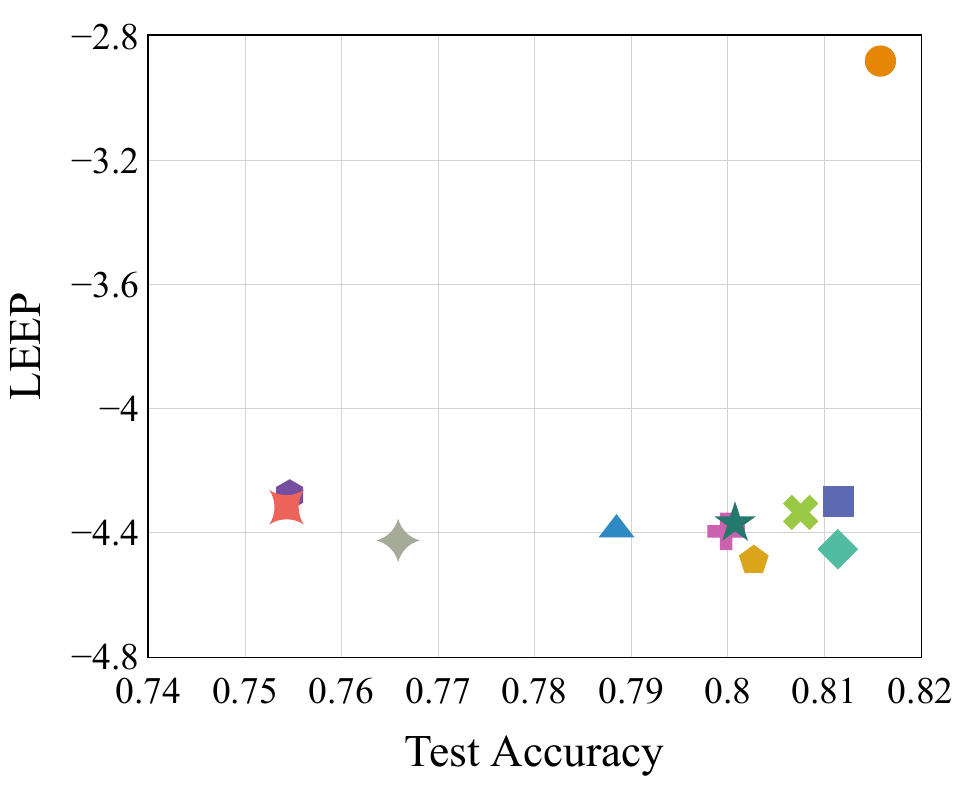} }
    \subfigure[][\tiny \makecell{ LFC on CIFAR\\ $R_p\!=\!0.5330$, \\ $\tau_K\!=\!0.6364$, \\ $\tau_\omega\!=\!0.8141$}]{\label{fig:s13}\includegraphics[width=0.15\textwidth]{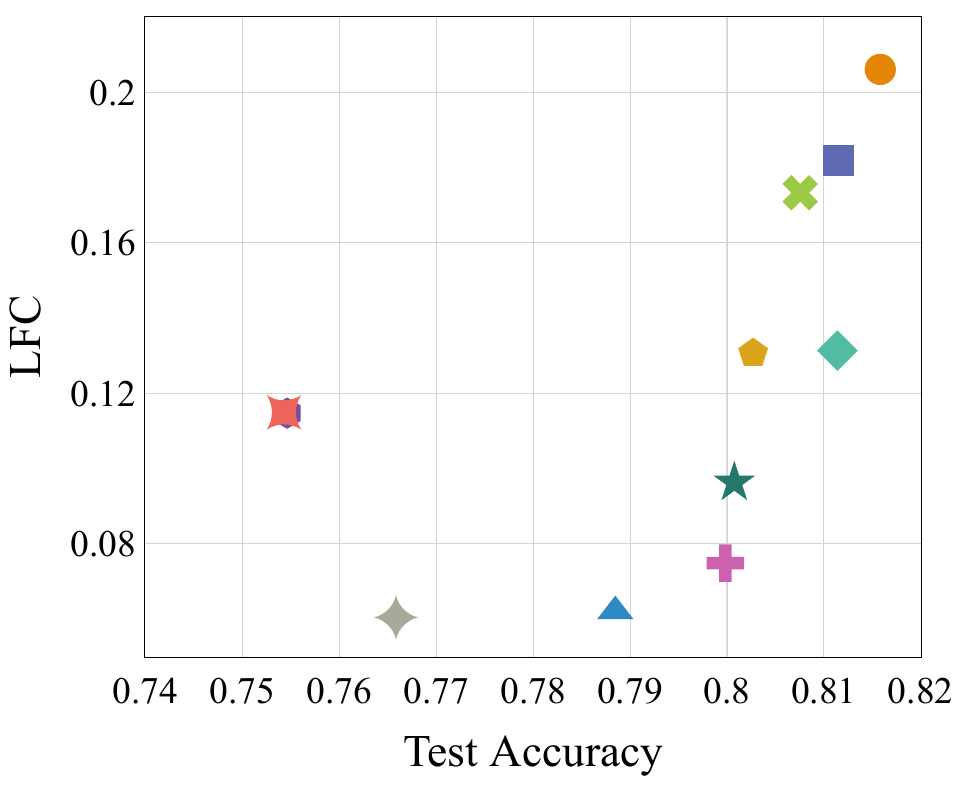} }
    \subfigure[][\tiny \makecell{ H-Score on CIFAR\\ $R_p\!=\!0.5078$, \\ $\tau_K\!=\!0.7091$, \\ $\tau_\omega\!=\!0.8134$}]{\label{fig:s14}\includegraphics[width=0.15\textwidth]{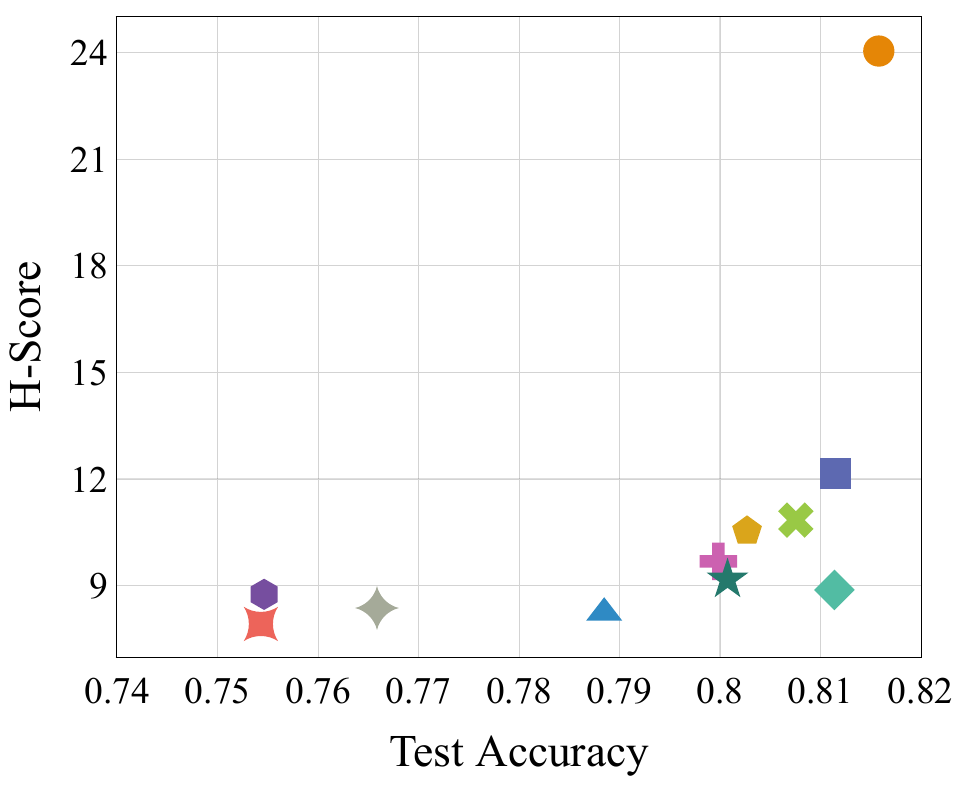} }
    \subfigure[][\tiny \makecell{ LogME on CIFAR \\ $R_p\!=\!0.4947$, \\ $\tau_K\!=\!0.7091$, \\ $\tau_\omega\!=\!0.8134$}]{\label{fig:s15}\includegraphics[width=0.15\textwidth]{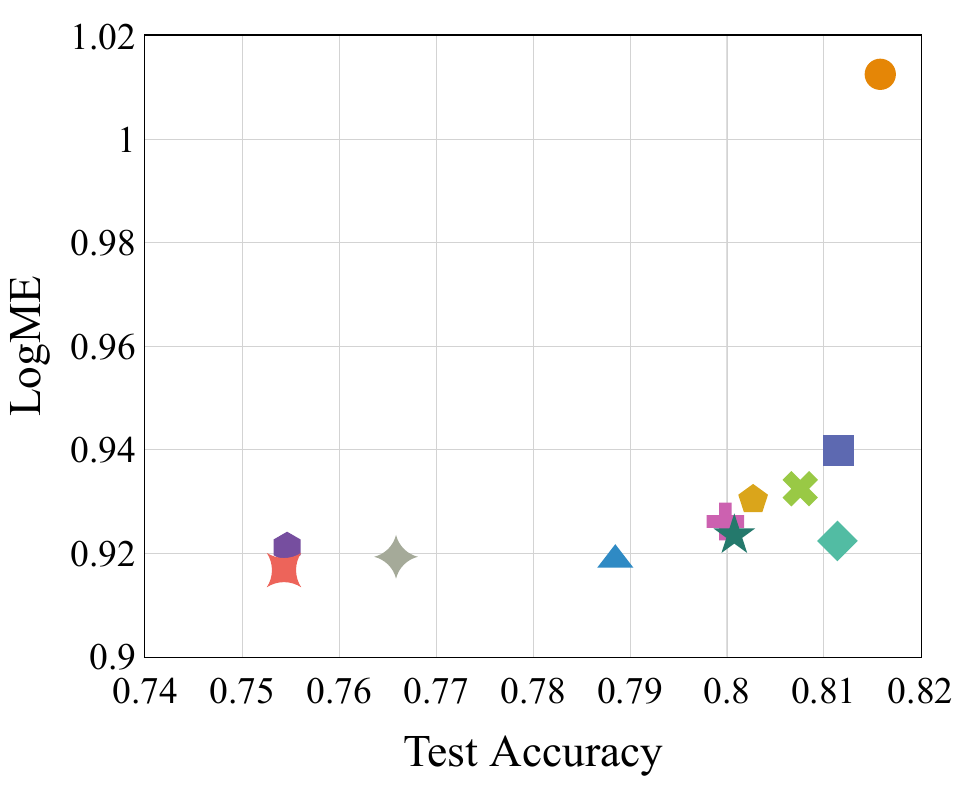} }
    \subfigure[][\tiny \makecell{ TrR on CIFAR\\ $R_p\!=\!{\bf 0.7262}$, \\ $\tau_K\!=\!{\bf0.8182}$, \\ $\tau_\omega\!=\!{\bf0.9055}$}]{\label{fig:s16}\includegraphics[width=0.15\textwidth]{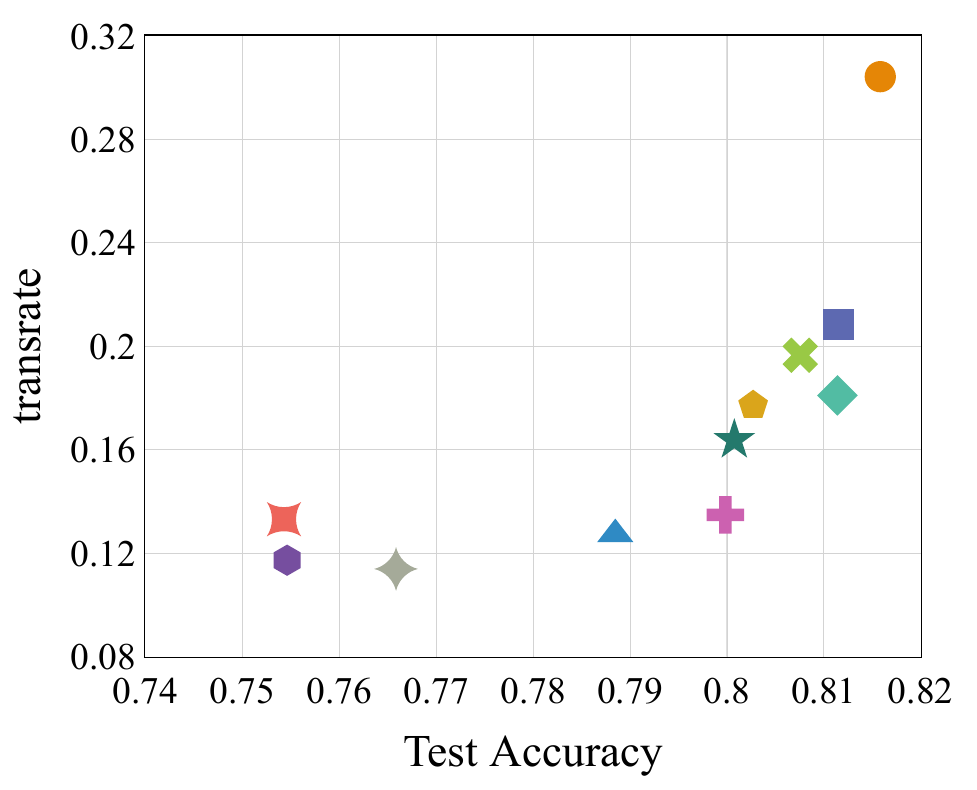} }
    \\
    \vspace{-2mm}
    \subfigure[][\tiny \makecell{ NCE on FMNIST\\ $R_p\!=\!0.6995$, \\ $\tau_K\!=\!0.4909$, \\ $\tau_\omega\!=\!0.6114$}]{\label{fig:s21}\includegraphics[width=0.15\textwidth]{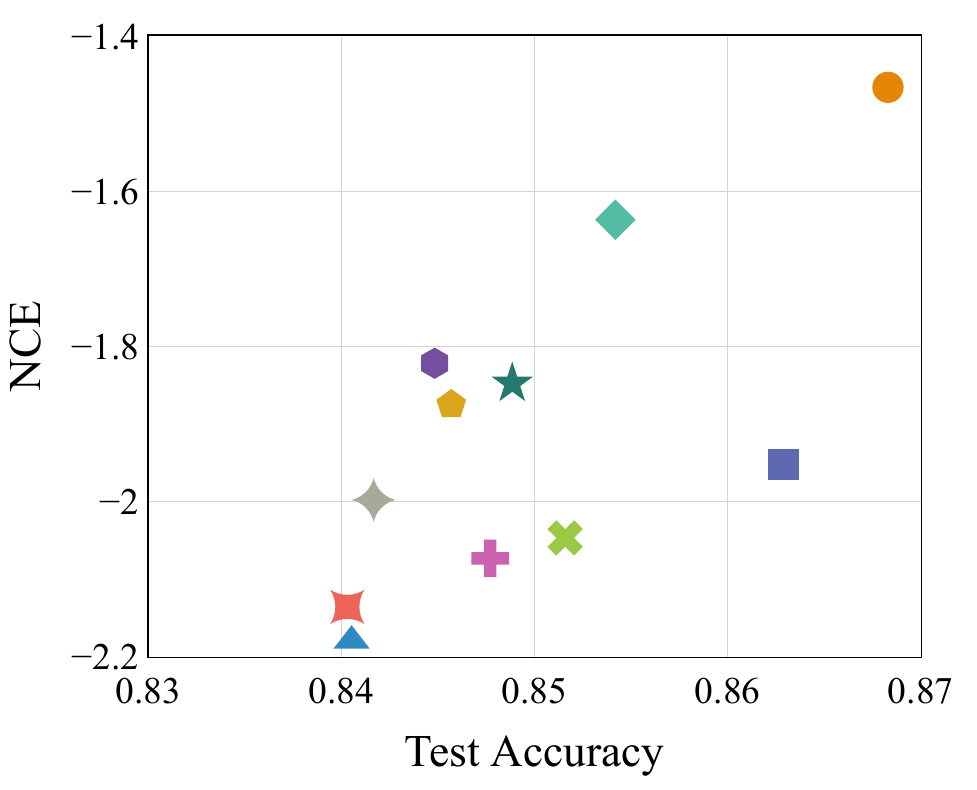} }
    \subfigure[][\tiny \makecell{ LEEP on FMNIST \\ $R_p\!=\!0.5200$, \\ $\tau_K\!=\!0.1273$, \\ $\tau_\omega\!=\!0.3383$}]{\label{fig:s22}\includegraphics[width=0.15\textwidth]{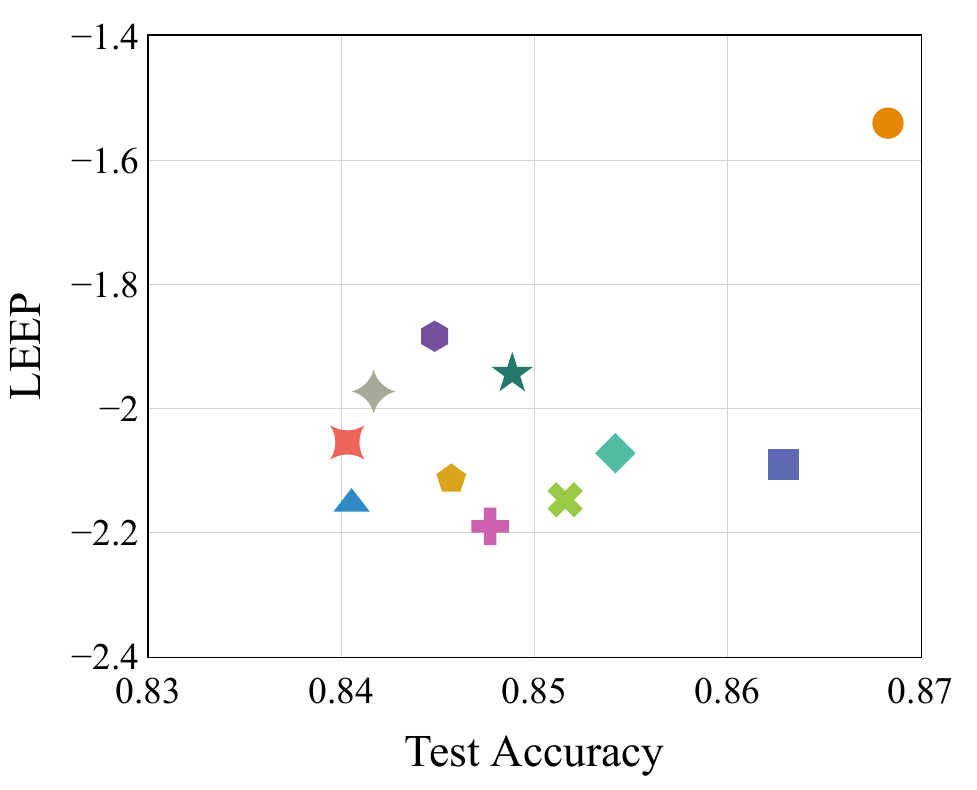} }
    \subfigure[][\tiny \makecell{ LFC on FMNIST\\ $R_p\!=\!0.7248$, \\ $\tau_K\!=\!0.4545$, \\ $\tau_\omega\!=\!0.6001$}]{\label{fig:s23}\includegraphics[width=0.15\textwidth]{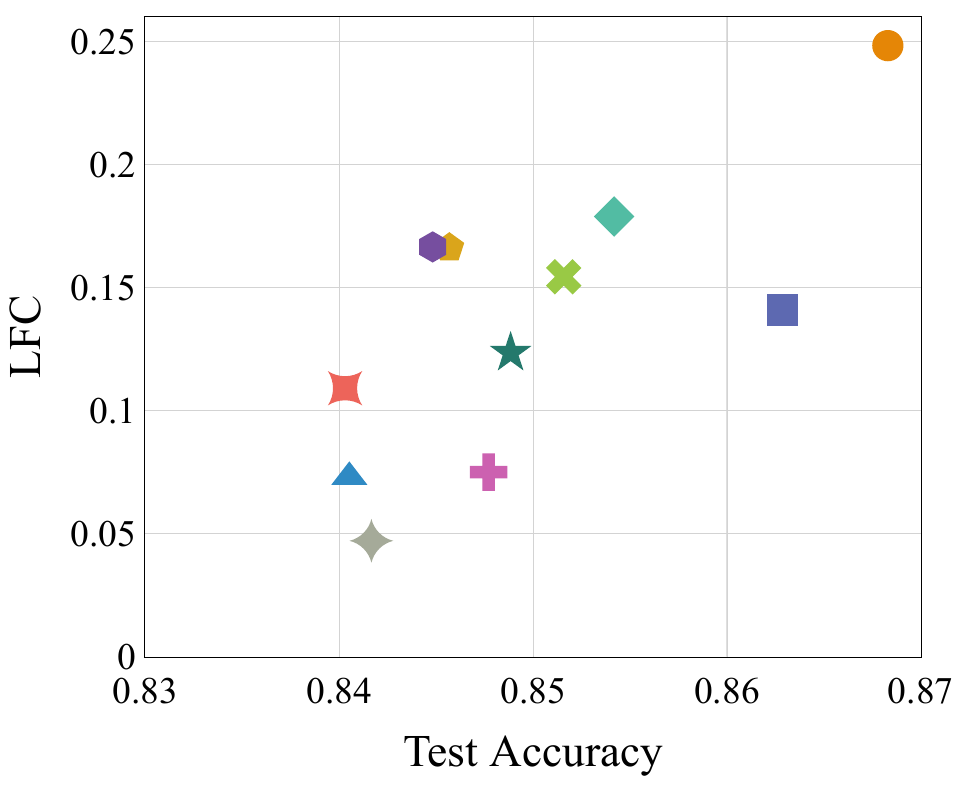} }
    \subfigure[][\tiny \makecell{ H-Score on FMNIST\\ $R_p\!=\!0.5945$, \\ $\tau_K\!=\!0.1273$, \\ $\tau_\omega\!=\!0.3468$}]{\label{fig:s24}\includegraphics[width=0.15\textwidth]{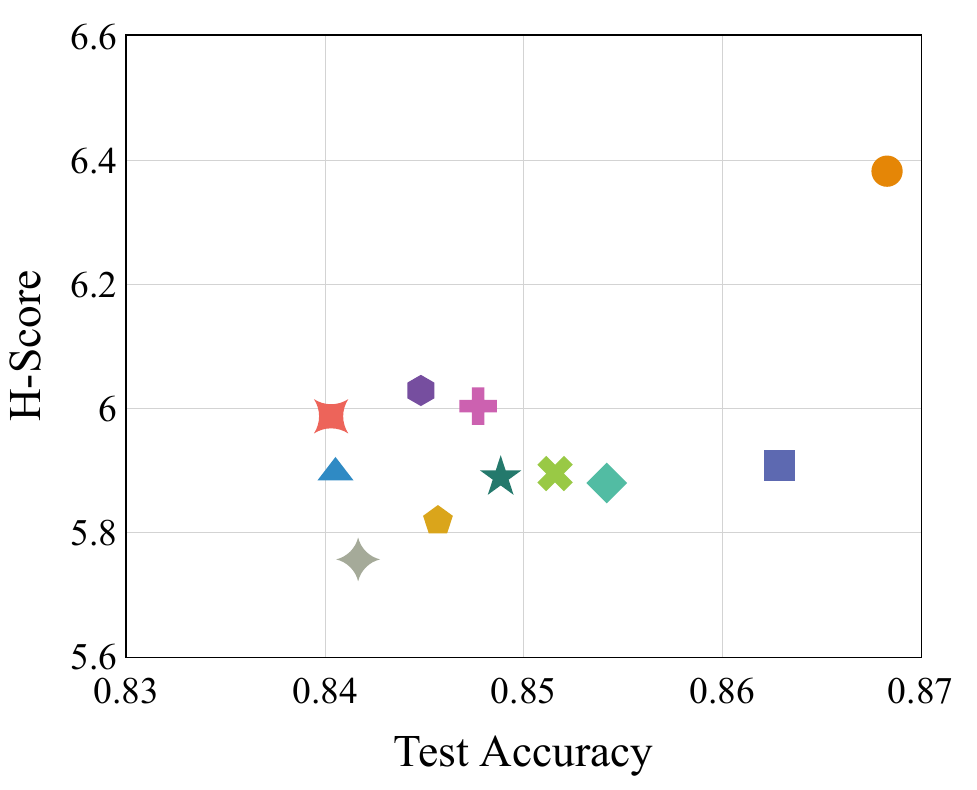} }
    \subfigure[][\tiny \makecell{ LogME on FMNIST \\ $R_p\!=\!0.5595$, \\ $\tau_K\!=\!0.0545$, \\ $\tau_\omega\!=\! 0.2781$}]{\label{fig:s25}\includegraphics[width=0.155\textwidth]{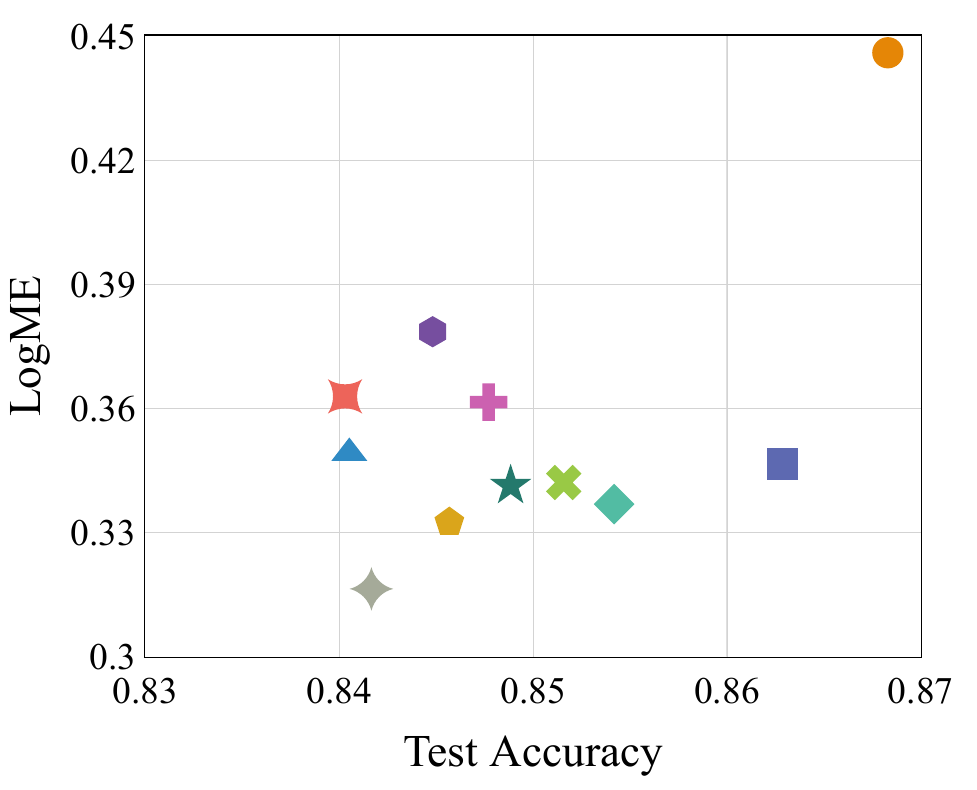} }
    \subfigure[][\tiny \makecell{ TrR on FMNIST\\ $R_p\!=\!{\bf0.8614}$, \\ $\tau_K\!=\!{\bf0.6727}$, \\ $\tau_\omega\!=\!{\bf0.8031}$}]{\label{fig:s26}\includegraphics[width=0.15\textwidth]{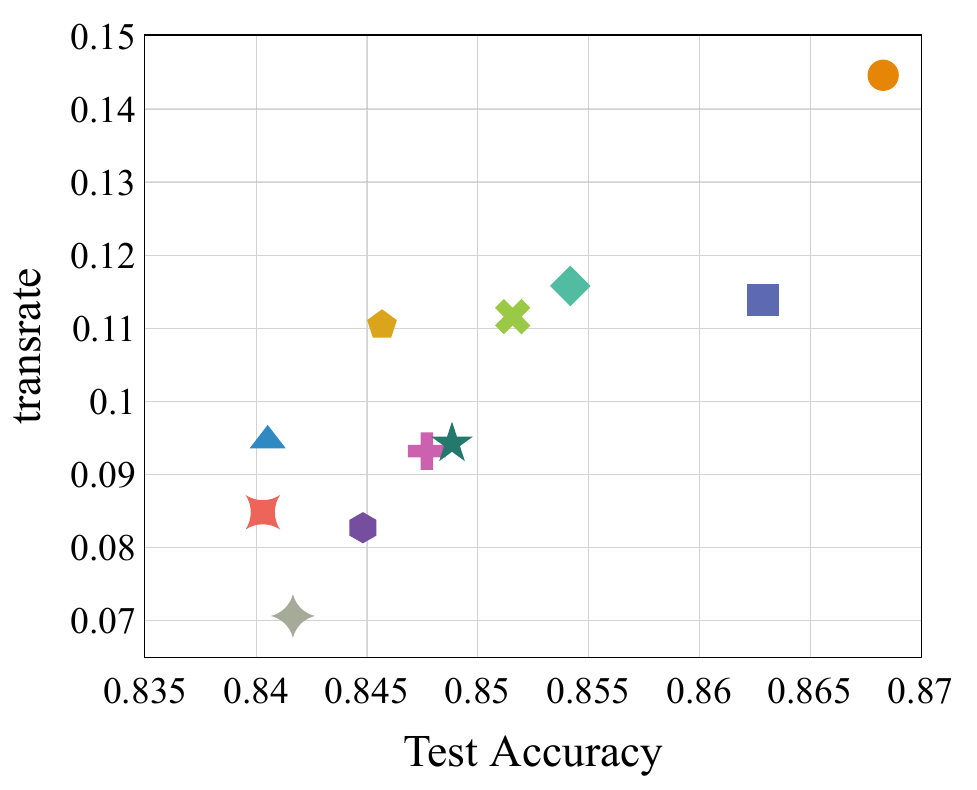} }
    \\
    \subfigure{\includegraphics[width=0.6\textwidth]{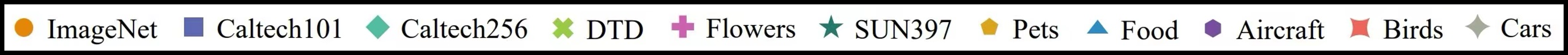}}
    \vspace{-3mm}\caption{\small Transferability estimation on transferring ResNet-18 pre-trained on $11$ different source datasets to CIFAR-100 and 
    \yingicml{FMNIST}.}\vspace{-4mm}
\label{fig:source_selection}
\end{figure*}

\section{Experiments}\label{sec:exp}

In this section, we evaluate the correlation between predicted transferability by \ours and the transfer learning performance in various settings and for different tasks. Due to page limit, experiments covering more settings and the wall-clock time comparison are available in Appendix~\ref{appsec:b}. 

\subsection{Implementation Details}

We consider fine-tuning a pre-trained model from a source dataset to the target task without access to any source data. For fine-tuning of the target task, the feature extractor is initialized by the pre-trained model. Then the feature extractor together with a randomly initialized head is optimized by running SGD on a cross-entropy loss for $100$ epoches. The \kai{batch size (16, 32, 64, 128)}, learning rate (from 0.0001 to 0.1) and weight decay 
(from 1E-6 to 1E-4) are 
\yingicml{determined}
via grid search of the best average transfer performance over 10 runs 
on a validation set. \kai{The reported 
transfer performance is 
\yingicml{an}
average of 
top 5 accuracies over 20 runs of experiments \yingicml{under the best hyperparameters above}.}

Before performing fine-tuning on the target task, we calculate \ours and the other \yingicml{baseline} transferability measures 
on training examples of 
\yingicml{a}
target task.
To compute the proposed \ours score, we first run a single forward pass of the pre-trained model through all target examples to extract their features $\hat{Z}$, \kai{\yingicml{and} then centralize $\hat{Z}$ to have zero mean}. Second, we compute the TransRate score as $R(\hat{Z}, \epsilon) - R(\hat{Z}, \epsilon | Y)$.
In 
the experiments, we set \kai{$\yingicml{\epsilon}
=$1E-4} by default.  
\kai{Since the scales of features extracted by different feature extractors may vary a lot, we scale the feature\yingicml{s} by {\small $1/\sqrt{\text{tr}(\hat{Z} \hat{Z}^\top)}$}, such that \yingicml{the} trace of the variance matrix of the normalized {\small $\hat{Z}$} is consistently equal to $1$ for all models. In the experiments of source selection and model selection, the features extracted by the pre-trained model trained on different source datasets or with different network architectures have signficantly different patterns, making it difficult to directly compare their TransRate. To tackle this problem and improve the performance of TransRate, we project the variance matrix {\small $\hat{Z} \hat{Z}^\top$} and {\small $\hat{Z}^c {\hat{Z}^c}{}^\top$} by a low-rank matrix {\small $(\hat{Z} \hat{Z}^\top)^{-1} \hat{U} \hat{U}^\top$ }, where {\small $\hat{U}$} is a matrix whose $c$-th row is the centroid feature of $c$-th class. 
}

We adopt LEEP~\cite{nguyen2020leep}, NCE~\cite{tran2019transferability}, Label-Feature Correlation (LFC)~\cite{deshpande2021linearized}, H-Score~\cite{bao2019information} and LogME~\cite{you2021logme} 
as the baseline methods. For a fair comparison, we assume no data from source tasks is available. 
In this scenario, the NCE score, defined by $-H(Y|Y_S)$ where $Y_S$ is the labels from the source task, cannot be computed following the procedure described in its original paper. 
Instead, we follow \cite{nguyen2020leep} to replace $Y_S$ with the softmax label generated by the classifer of 
\yingicml{a}
pre-trained model.
Another setting for a fair comparison is that only one single forward pass through 
target examples is allowed for computational efficiency. In this case,
\yingicml{we calculate H-score by pre-trained features and skip the computation of H-score based on the optimal target features as suggested in}
\cite{bao2019information}.

To measure the performance of \ours and five baseline methods in estimating the transfer learning performance, we follow \cite{nguyen2020leep,tran2019transferability} to compute the Pearson correlation coefficient between the score and the average accuracy of the fine-tuned model on testing samples of the target set. The Kendall's $\tau$~\cite{kendall1938new} and its variant, weighted $\tau$, are also adopted as performance metrics. For brief, we denote the Pearson correlation coefficient, Kendall's $\tau$ and weighted $\tau$ by $R_p$, $\tau_K$, and $\tau_\omega$, respectively. 

\begin{figure*}[t!]
\centering
    \subfigure[][\tiny \makecell{ LFC from SVHN\\ $R_p\!=\!-0.1895$, \\ $\tau_K\!=\!-0.4667$, \\ $\tau_\omega\!=\!-0.5497$}]{\label{fig:l13}\includegraphics[width=0.18\textwidth]{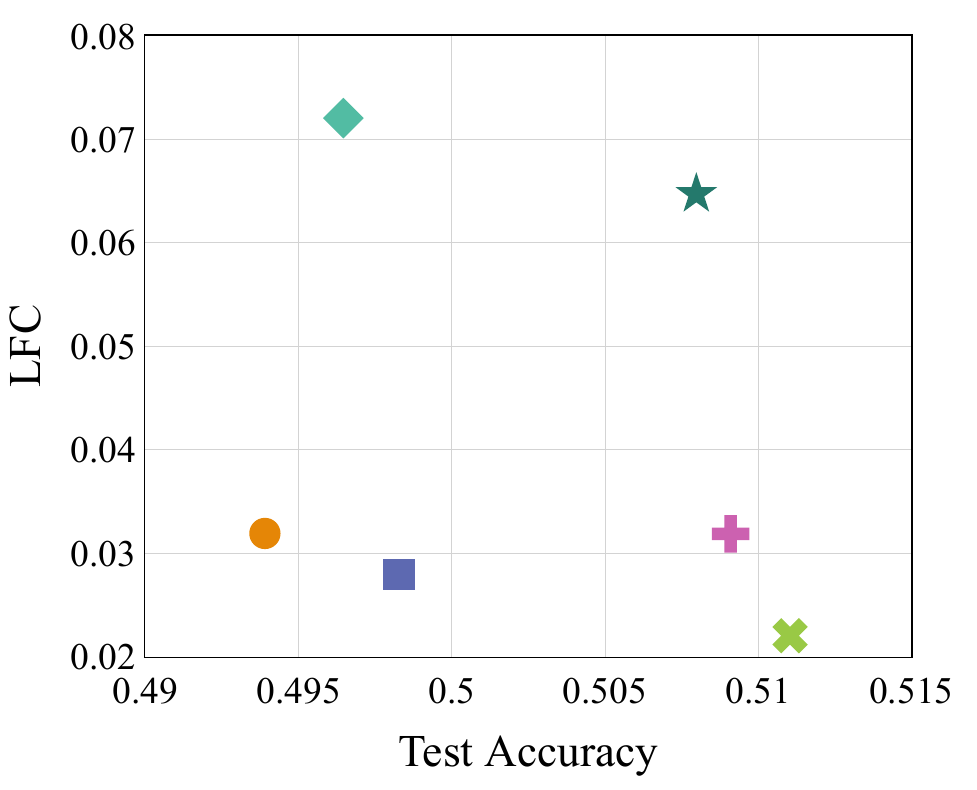} }
    \subfigure[][\tiny \makecell{ H-Score from SVHN\\ $R_p\!=\!-0.5320$, \\ $\tau_K\!=\!{-0.2000}$, \\ $\tau_\omega\!=\!-0.2993$}]{\label{fig:l14}\includegraphics[width=0.18\textwidth]{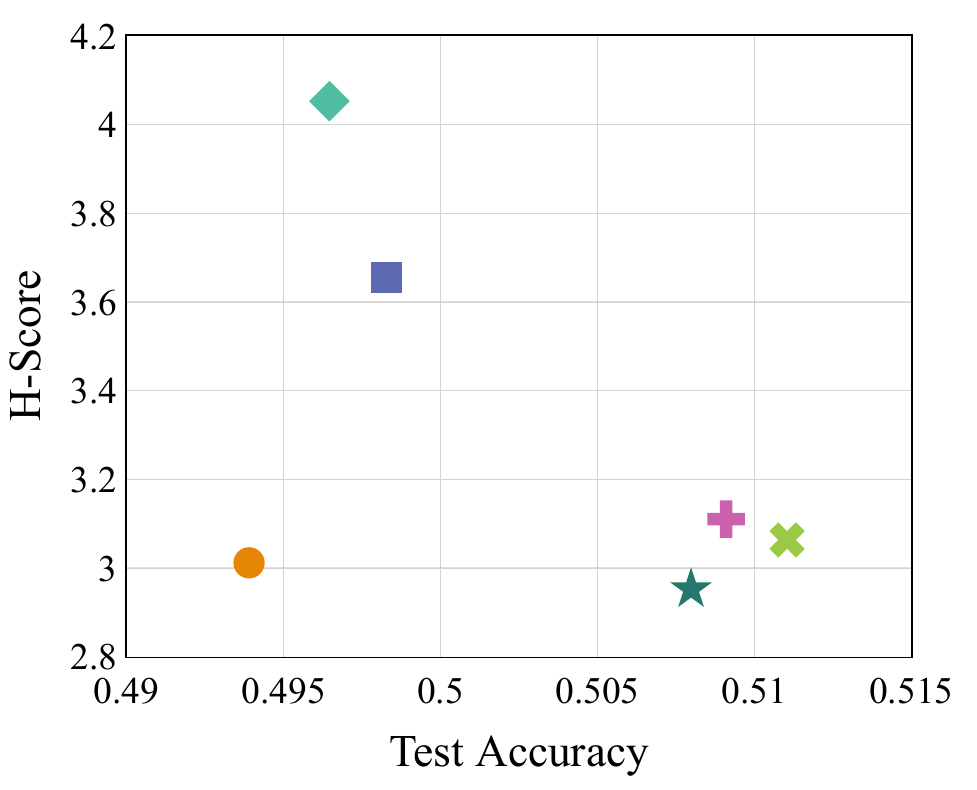} }
    \subfigure[][\tiny \makecell{ LogME from SVHN \\ $R_p\!=\!-0.3352$, \\ $\tau_K\!=\!{-0.0667}$, \\ $\tau_\omega\!=\! -0.2340$}]{\label{fig:l15}\includegraphics[width=0.18\textwidth]{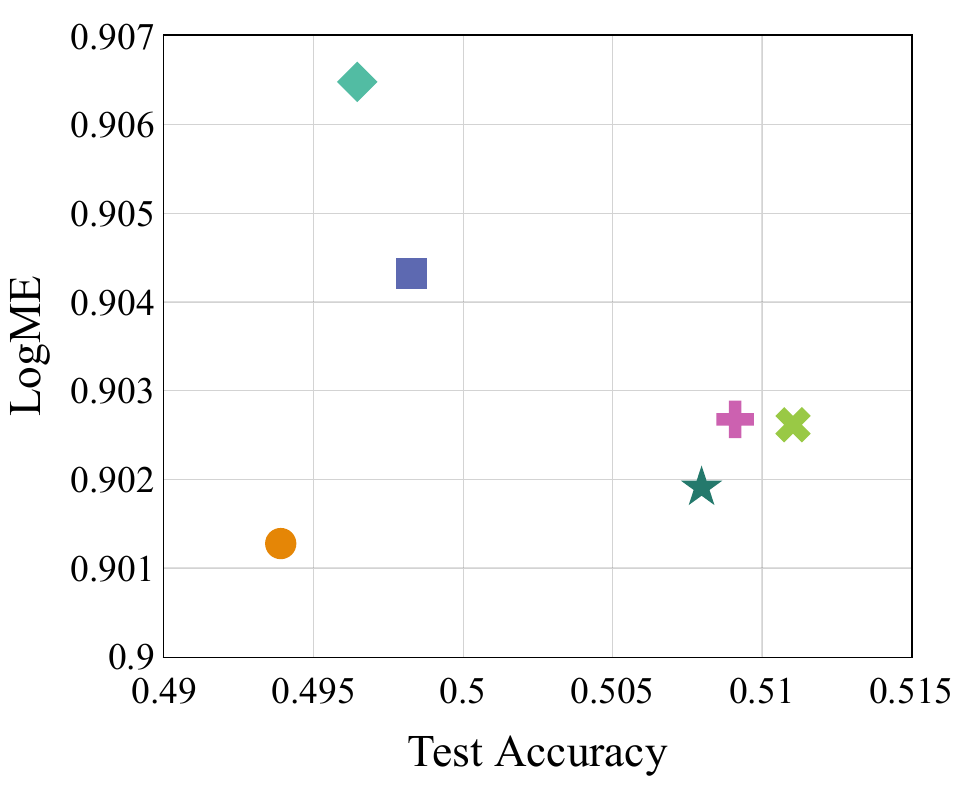}}
    \subfigure[][\tiny \makecell{ TrR from SVHN\\ $R_p\!=\!{\bf0.9769}$, \\ $\tau_K\!=\!{\bf0.8667}$, \\ $\tau_\omega\!=\!{\bf0.9265}$}]{\label{fig:l16}\includegraphics[width=0.18\textwidth]{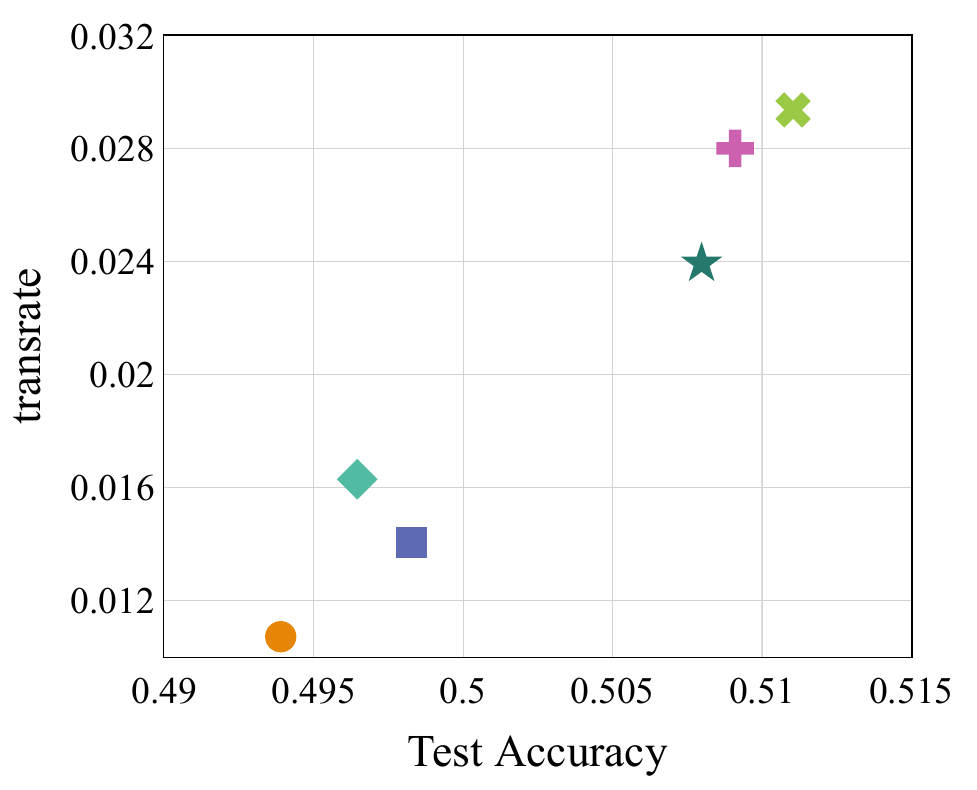}}
    \\
    \vspace{-3mm}
    \subfigure[][\tiny \makecell{ LFC from Birdsnap\\ $R_p\!=\!0.7003$, \\ $\tau_K\!=\!0.6667$, \\ $\tau_\omega\!=\!0.5200$}]{\label{fig:l23}\includegraphics[width=0.18\textwidth]{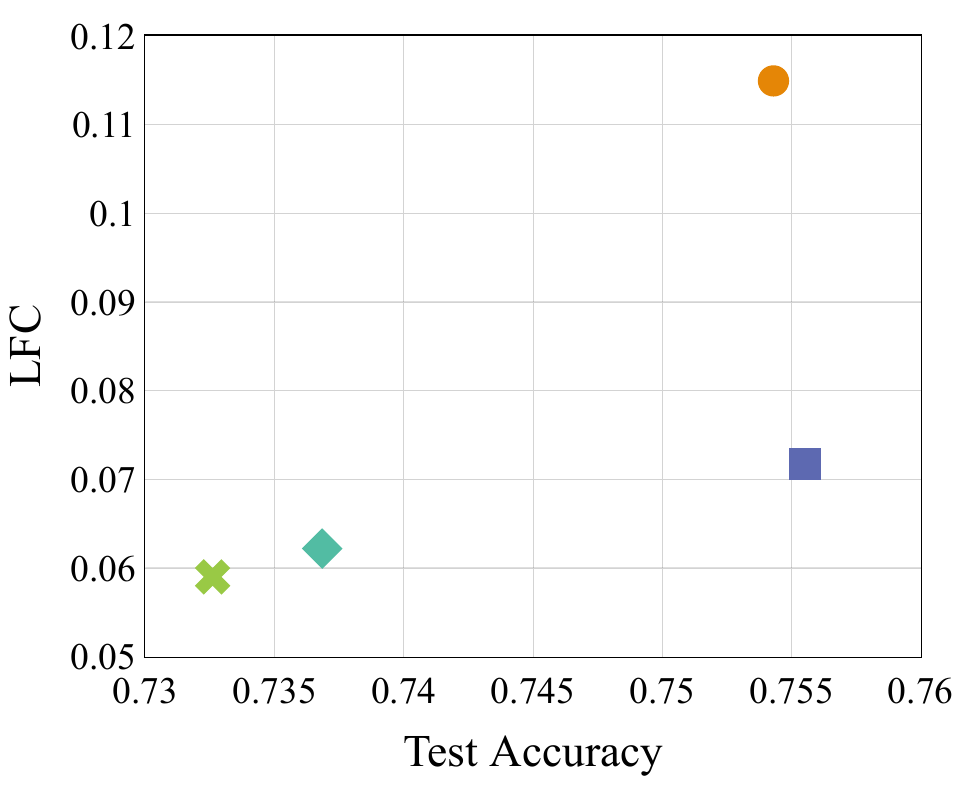} }
    \subfigure[][\tiny \makecell{ H-Score from Birdsnap\\ $R_p\!=\!0.3166$, \\ $\tau_K\!=\!{0.0000}$, \\ $\tau_\omega\!=\!0.3067$}]{\label{fig:l24}\includegraphics[width=0.18\textwidth]{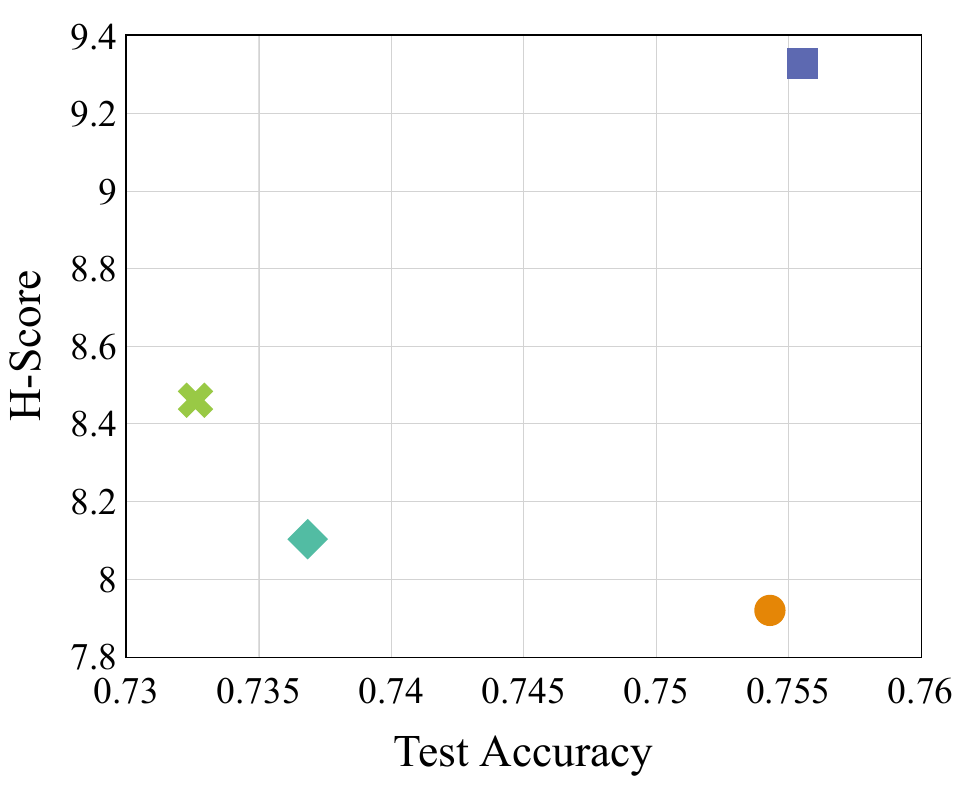} }
    \subfigure[][\tiny \makecell{ LogME from Birdsnap \\ $R_p\!=\!-0.5207$, \\ $\tau_K\!=\!-0.3333$, \\ $\tau_\omega\!=\!-0.2933$}]{\label{fig:l25}\includegraphics[width=0.18\textwidth]{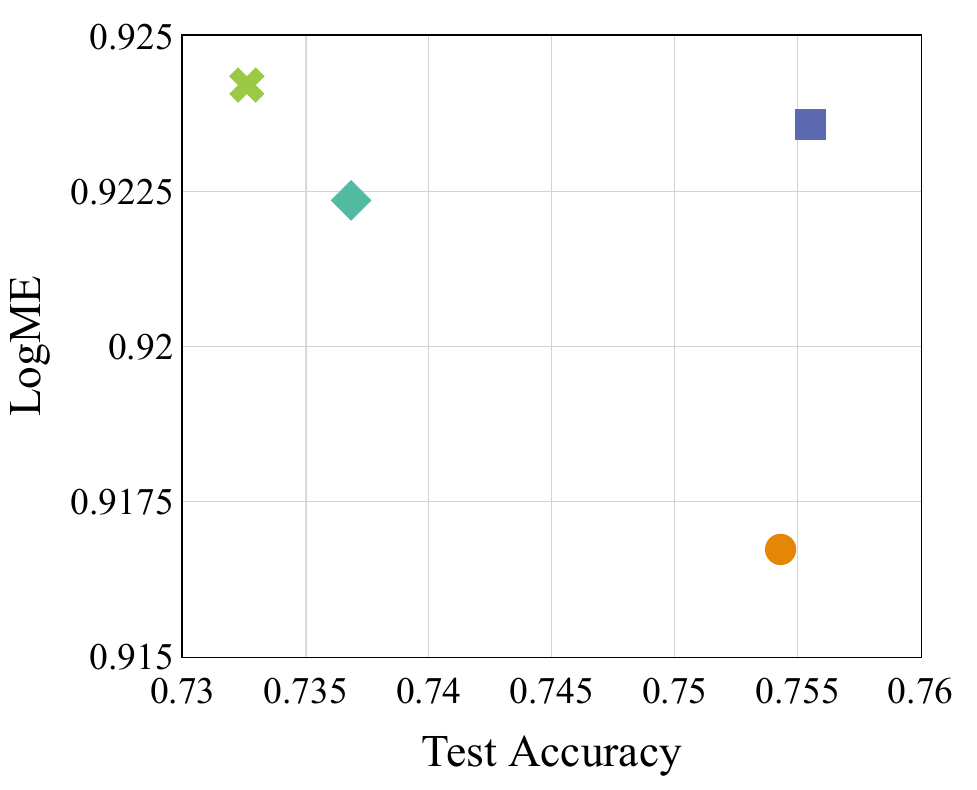} }
    \subfigure[][\tiny \makecell{ TrR from Birdsnap\\ $R_p\!=\!{\bf 0.9871}$, \\ $\tau_K\!=\!{\bf 0.6667}$, \\ $\tau_\omega\!=\!{\bf 0.8133}$}]{\label{fig:l26}\includegraphics[width=0.18\textwidth]{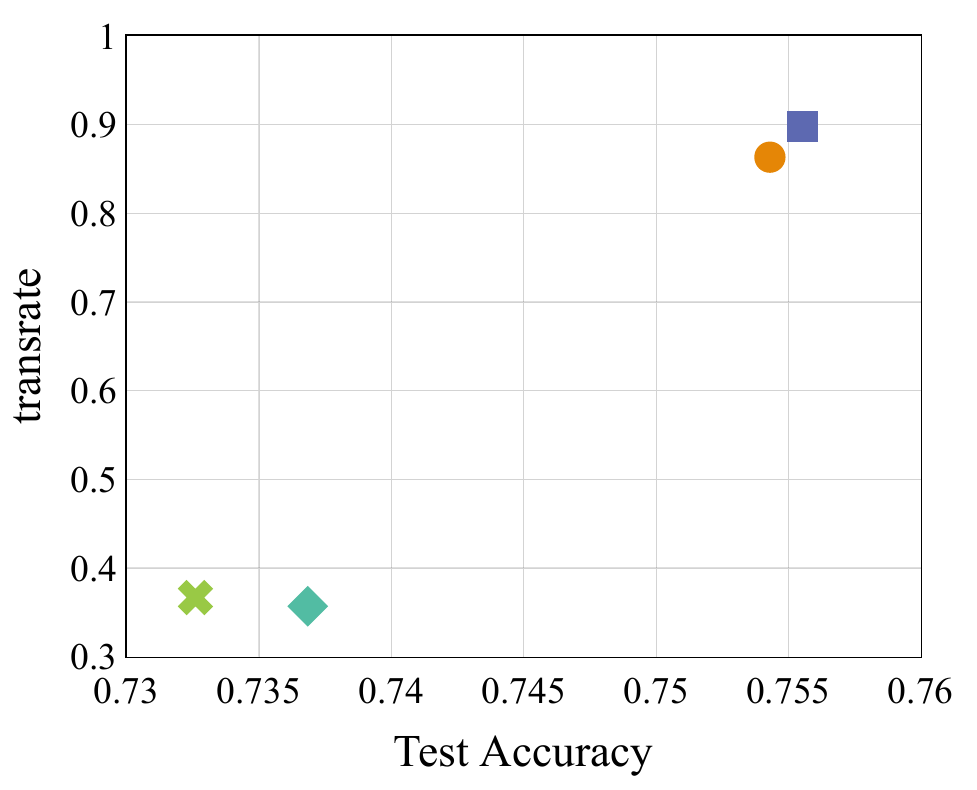} }
    \\
    \subfigure{\includegraphics[width=0.5\textwidth]{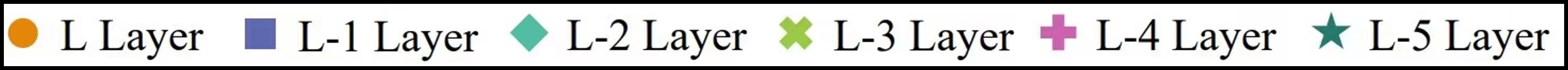}}
    \vspace{-0mm}\caption{\small Transferability estimation on transferring different layers of ResNet-20 pre-trained on SVHN and of ResNet-18 pre-trained on Birdsnap to CIFAR-100.}\vspace{-4mm}
\label{fig:layer_selection}
\end{figure*}
\subsection{Results} \label{sec:exp_transfer_results}
\textbf{TransRate as a Criterion for Source Selection.}
One of the most important applications of \ours is source model selection for a target task. 
\yingicml{Here}
we evaluate the performance of \ours and other baseline measures for selection of a pre-trained model from $11$ source datasets to a specific target task. The source datasets are ImageNet~\cite{russakovsky2015imagenet} and $10$ image datasets from \cite{salman2020adversarially}, including Caltech-101, Caltech-256, DTD, Flowers, SUN397, Pets, Food, Aircraft, Birds and Cars. For each source dataset, we pre-train a ResNet-18~\cite{he2016deep}, freeze it and discard the source data during fine-tuning. 
CIFAR-100~\cite{krizhevsky2009learning} and 
\yingicml{FMNIST}~\cite{xiao2017fashion} are adopted as the target tasks. 
\kai{For all target datasets, we use the whole training set for fine-tuning and for transferability estimation.}
The details of these datasets and their pre-trained models are available in Appendix~\ref{appsec:a}; experiments on more target tasks are available in Appendix~\ref{appsec:b1}.

Figure~\ref{fig:source_selection} \yingicml{show that}
\kai{LFC, H-Score, LogME and \ours all correctly predict the ranking of top-5 source models, except the one pre-trained on Caltech-256, for CIFAR-100. \ours achieves the best $R_p$, $\tau_K$ and $\tau_\omega$, which means that the ranking predicted by \ours is more accurate than \yingicml{the others}.}
As for FMNIST, \ours correctly predicts the top-4 source models, \kai{though slightly underestimates the transferability of the Caltech-101 model}, while all the other baselines fail to accurately predict the rank of the Caltech-101 model 
which comes the second among all. 
\ours outperforms the baselines by a large margin in all correlation coefficients. 
These results demonstrate that TransRate can serve as a practical criterion for source selection in transfer learning.

\begin{figure*}[t!]
\centering
    \subfigure[][\tiny \makecell{ NCE on CIFAR\\ $R_p\!=\!0.9654$, \\ $\tau_K\!=\!0.8095$, \\ $\tau_\omega\!=\! 0.7322$}]{\label{fig:m11}\includegraphics[width=0.15\textwidth]{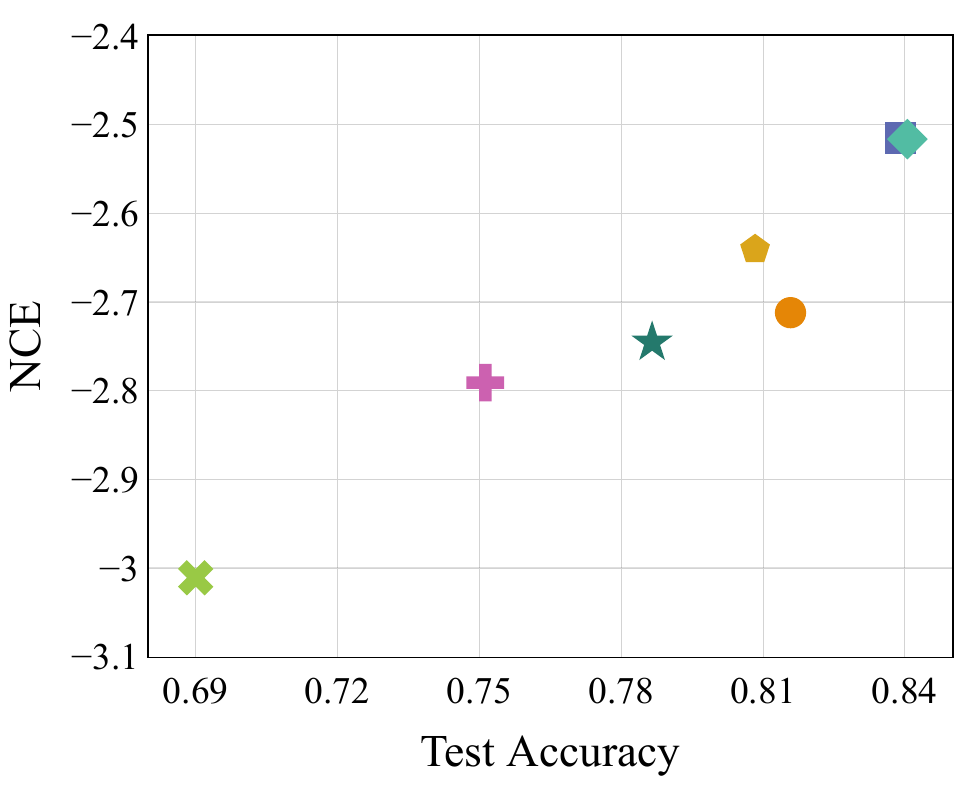} }
    \subfigure[][\tiny \makecell{ LEEP on CIFAR \\ $R_p\!=\!{\bf 0.9696}$, \\ $\tau_K\!=\!0.8095$, \\ $\tau_\omega\!=\!0.8650$}]{\label{fig:m12}\includegraphics[width=0.15\textwidth]{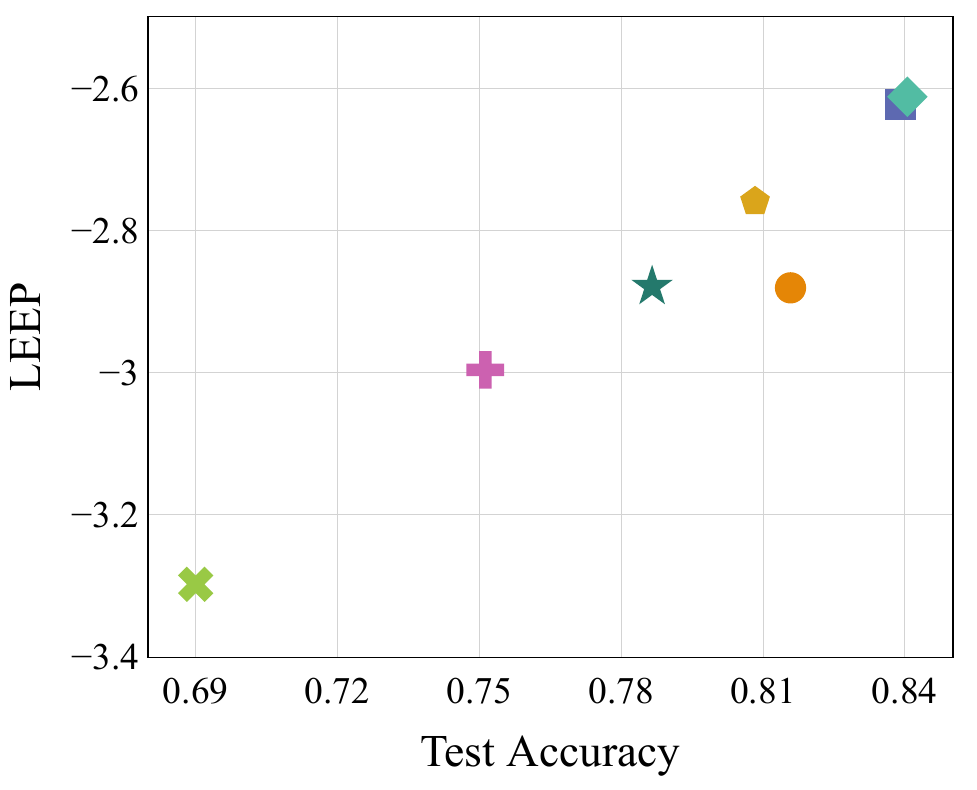} }
    \subfigure[][\tiny \makecell{ LFC on CIFAR\\ $R_p\!=\!0.0664$, \\ $\tau_K\!=\!-0.0476$, \\ $\tau_\omega\!=\!-0.0680$}]{\label{fig:m13}\includegraphics[width=0.15\textwidth]{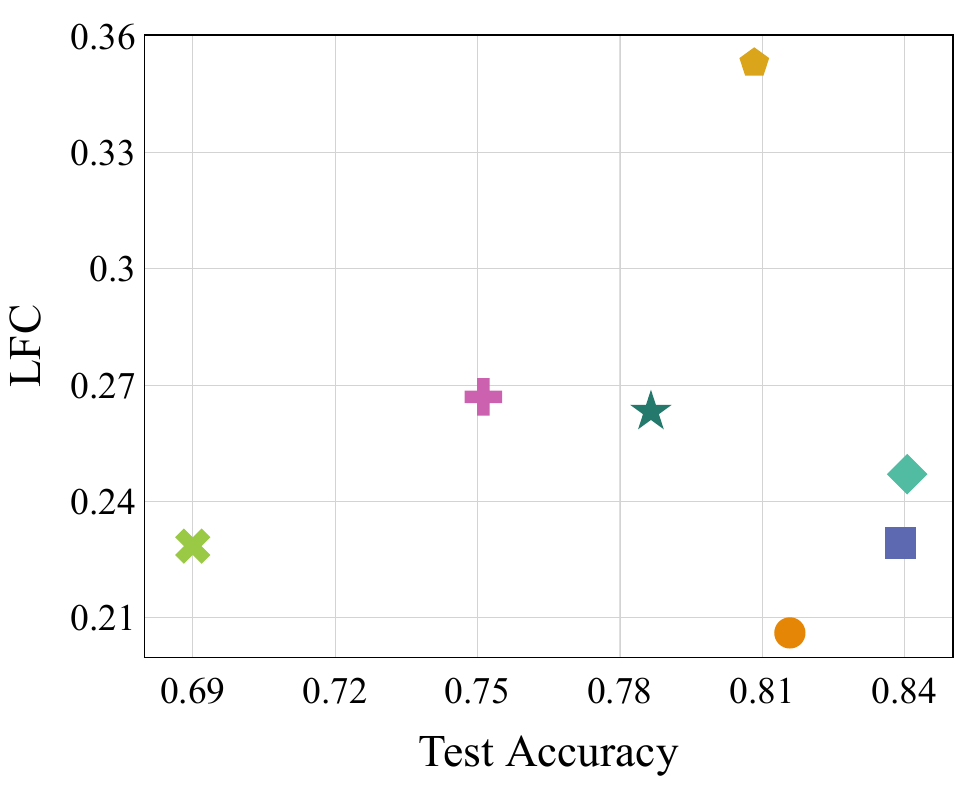} }
    \subfigure[][\tiny \makecell{ H-Score on CIFAR\\ $R_p\!=\!0.3802$, \\ $\tau_K\!=\!{ 0.3333}$, \\ $\tau_\omega\!=\!0.5041$}]{\label{fig:m14}\includegraphics[width=0.15\textwidth]{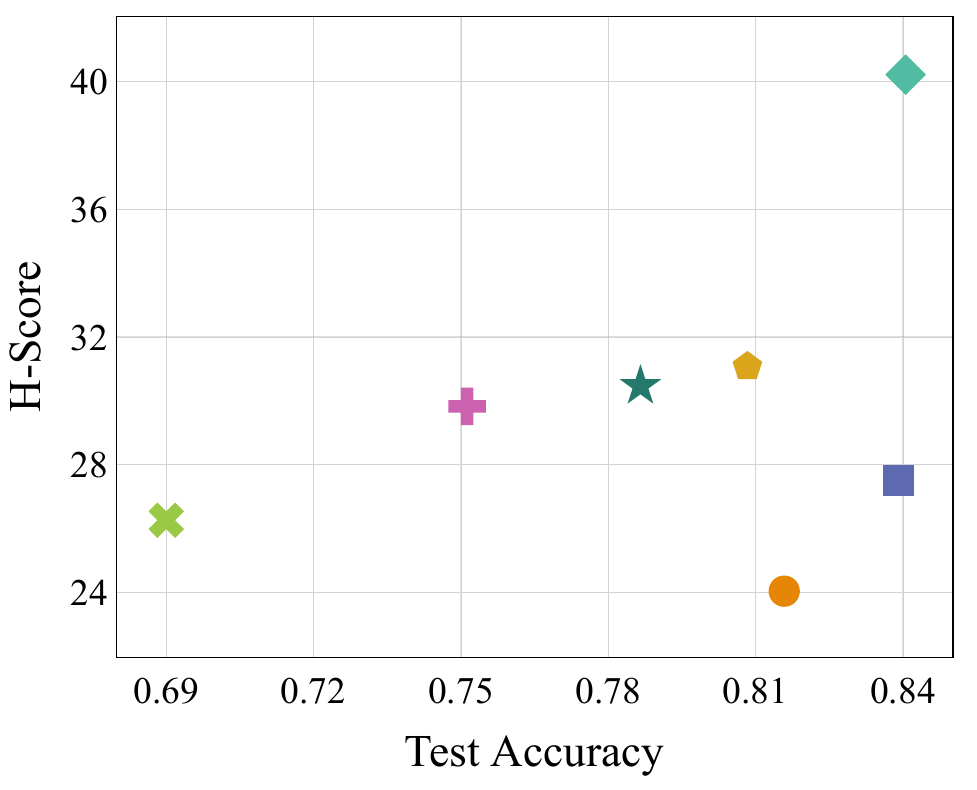} }
    \subfigure[][\tiny \makecell{ LogME on CIFAR \\ $R_p\!=\!0.5672$, \\ $\tau_K\!=\!{ 0.5238}$, \\ $\tau_\omega\!=\! 0.6186$}]{\label{fig:m15}\includegraphics[width=0.15\textwidth]{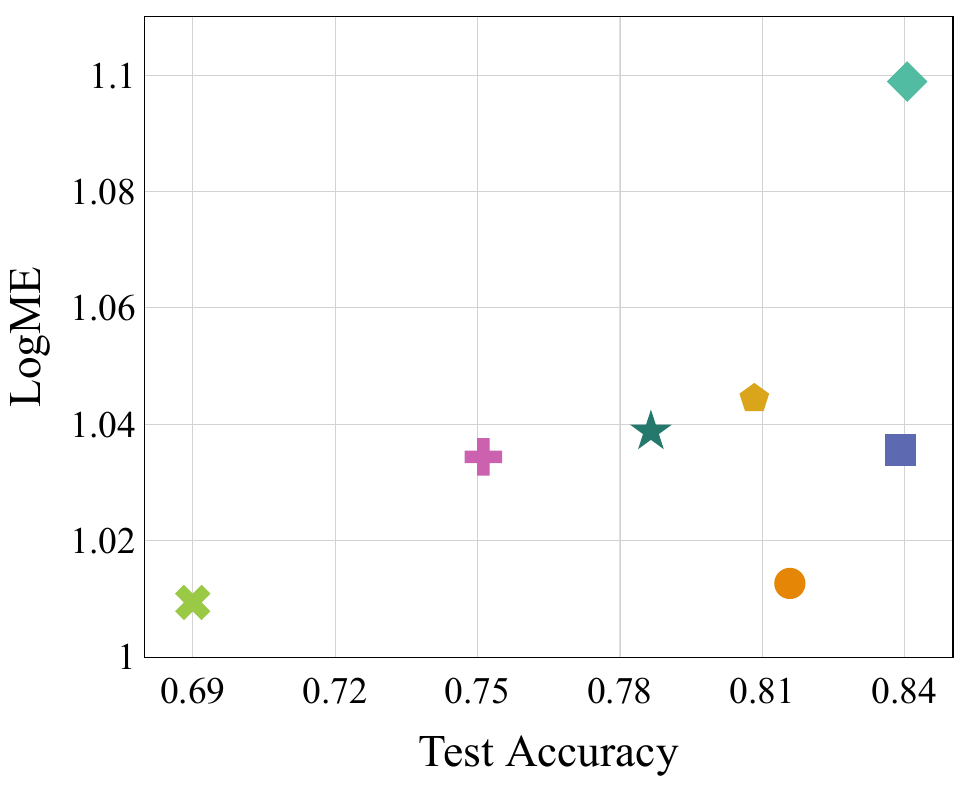} }
    \subfigure[][\tiny \makecell{ TrR on CIFAR\\ $R_p\!=\!{0.8055}$, \\ $\tau_K\!=\!{\bf 0.9048}$, \\ $\tau_\omega\!=\!{\bf0.9421}$}]{\label{fig:m16}\includegraphics[width=0.15\textwidth]{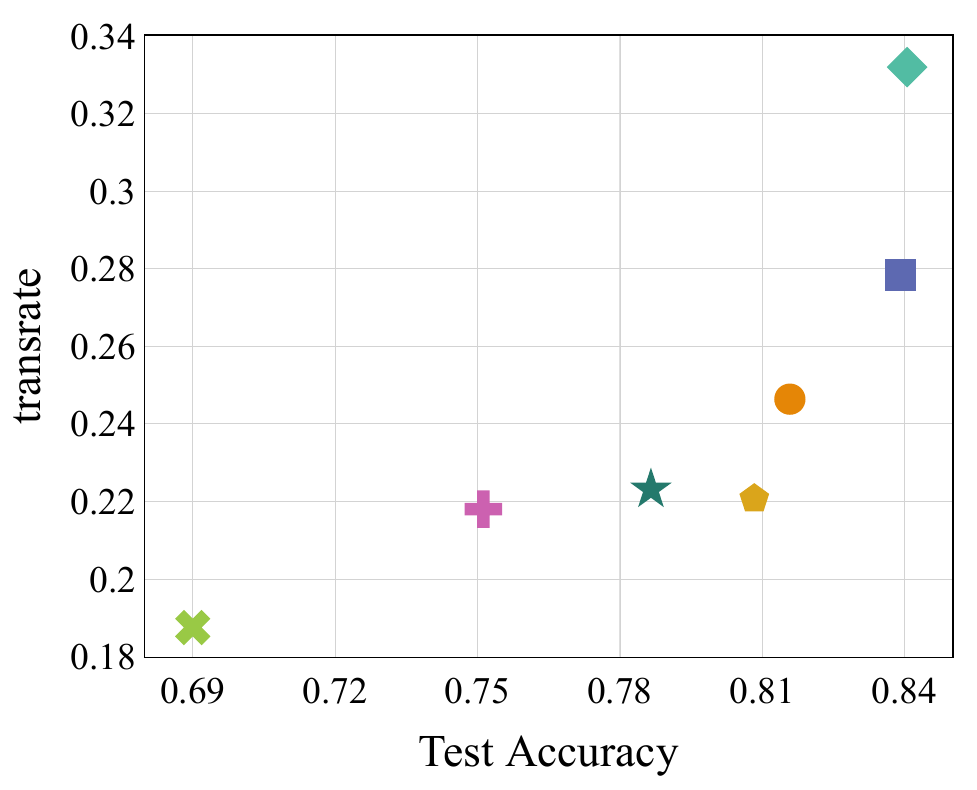} }
    \\
    \subfigure{\includegraphics[width=0.8\textwidth]{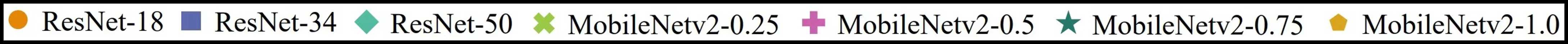}}
    \vspace{-1mm}
    \caption{\small Result on transferring models with different architectures from ImageNet to CIFAR-100.}\vspace{-4mm}
\label{fig:model_selection}
\end{figure*}

\label{appsec:b7}
\begin{figure*}[t!]
\centering
    \subfigure[][\tiny \makecell{ H-Score on BBBP\\ $R_p\!=\!-0.1572$, \\ $\tau_K\!=\!-0.3333$, \\ $\tau_\omega\!=\!-0.2933$}]{\label{fig:d21}\includegraphics[width=0.15\textwidth]{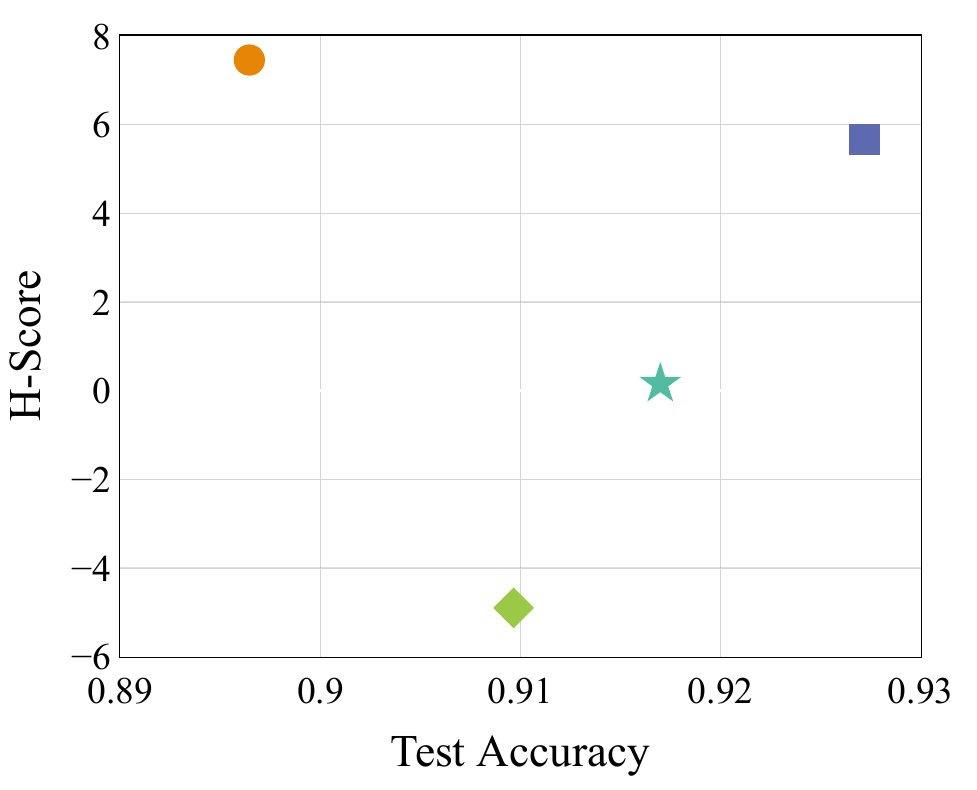} }
    \subfigure[][\tiny \makecell{ LFC on BBBP\\ $R_p\!=\!-0.1034$, \\ $\tau_K\!=\!-0.0$, \\ $\tau_\omega\!=\!-0.0667$}]{\label{fig:d22}\includegraphics[width=0.15\textwidth]{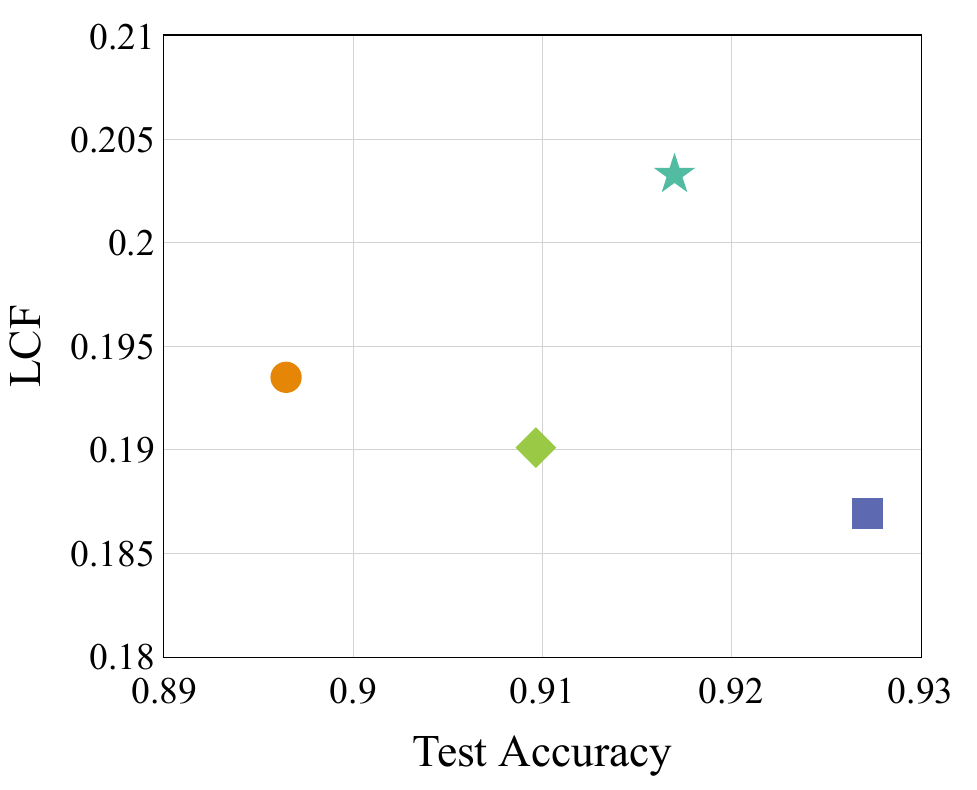} }
    \subfigure[][\tiny \makecell{ LogME on BBBP\\ $R_p\!=\!-0.1838$, \\ $\tau_K\!=\!0.0$, \\ $\tau_\omega\!=\!-0.0667$}]{\label{fig:d23}\includegraphics[width=0.15\textwidth]{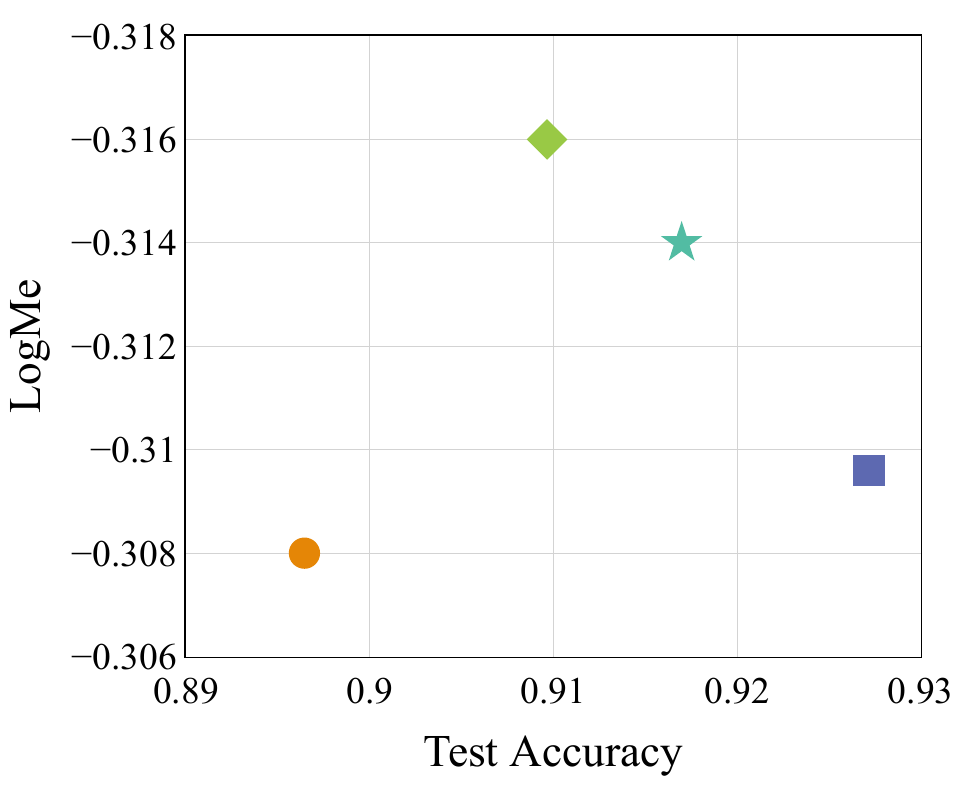} }
    \subfigure[][\tiny \makecell{ TrR on BBBP\\ $R_p\!=\!{\bf0.6129}$, \\ $\tau_K\!=\!{\bf 0.6667}$, \\ $\tau_\omega\!=\!{\bf0.7333}$}]{\label{fig:d24}\includegraphics[width=0.15\textwidth]{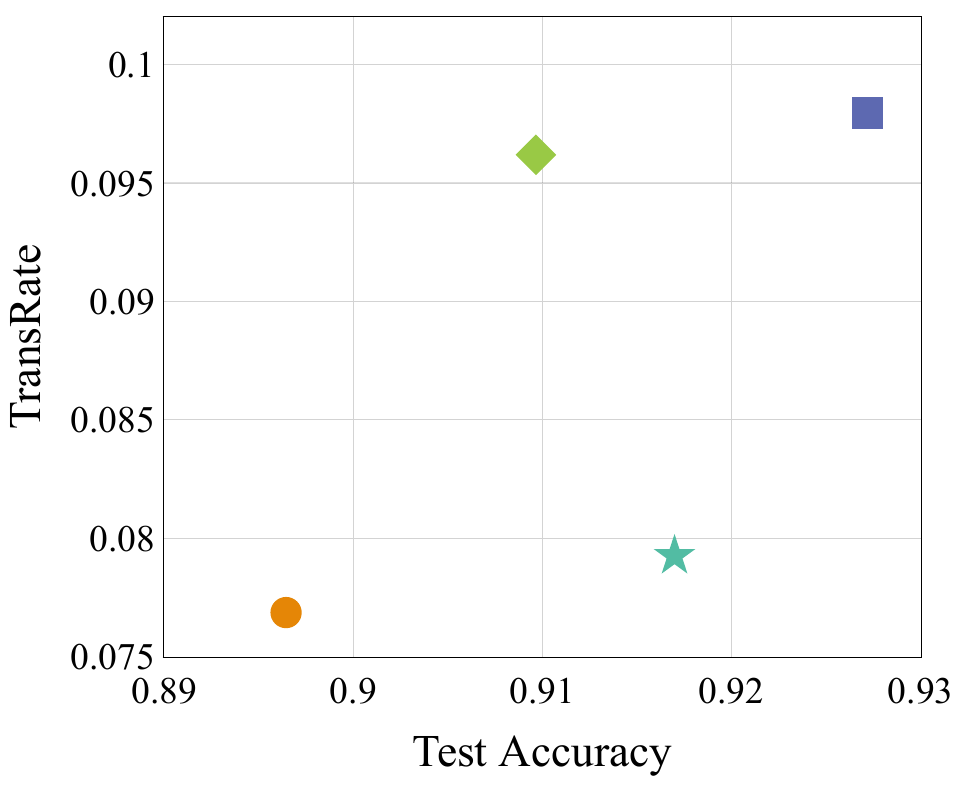} }
    \subfigure[][\tiny \makecell{ LogME on FreeSolv \\ $R_p\!=\!-0.5952$, \\ $\tau_K\!=\!-0.3333$, \\ $\tau_\omega\!=\!-0.3333$}]{\label{fig:d41}\includegraphics[width=0.15\textwidth]{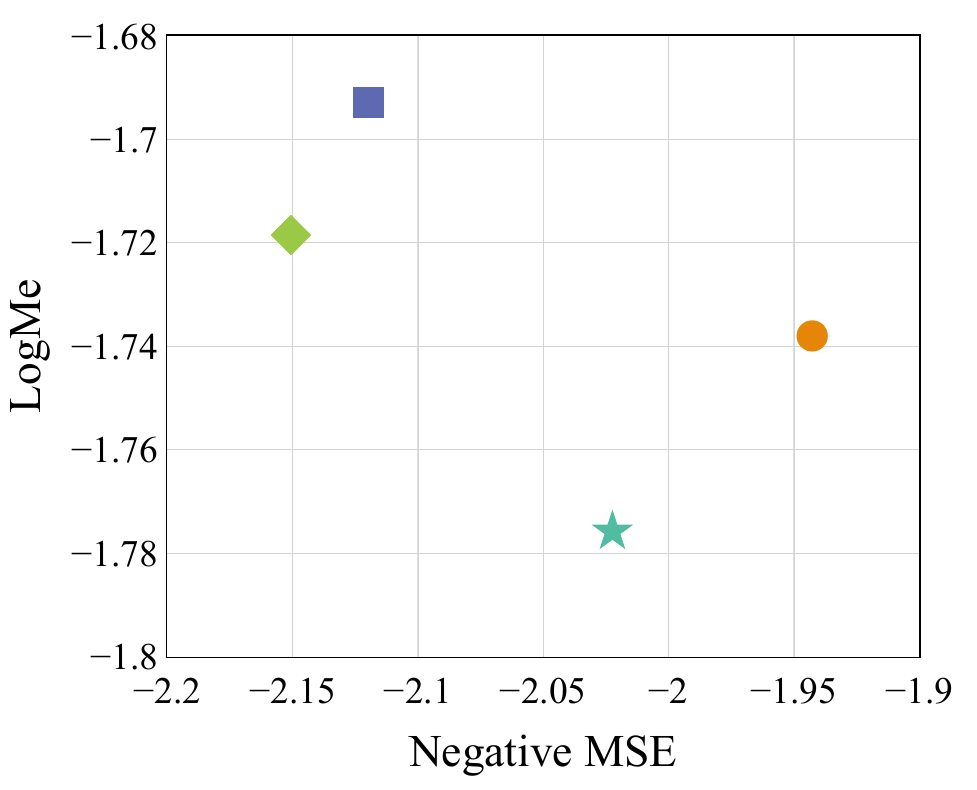} }
    \subfigure[][\tiny \makecell{ TrR on FreeSolv\\ $R_p\!=\!{\bf0.9582}$, \\ $\tau_K\!=\!{\bf1.0}$, \\ $\tau_\omega\!=\!{\bf1.0}$}]{\label{fig:d42}\includegraphics[width=0.15\textwidth]{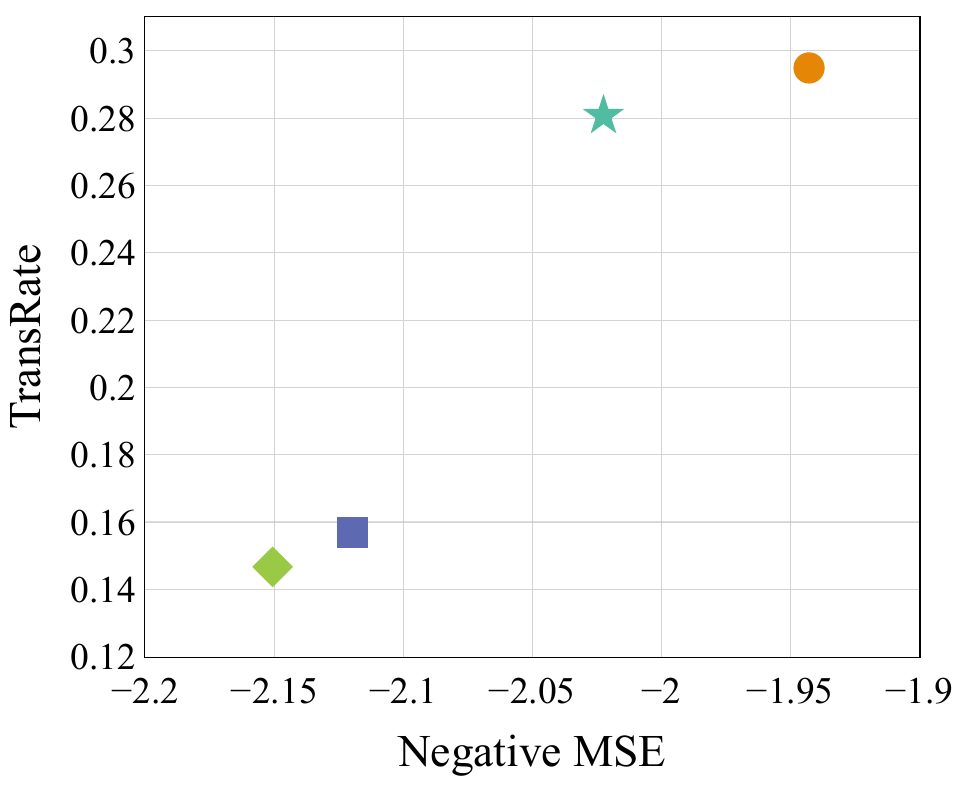} }
    \\
    \vspace{-3mm}
    \subfigure[][\tiny \makecell{ H-Score on BACE\\ $R_p\!=\!-0.7514$, \\ $\tau_K\!=\!-0.5477$, \\ $\tau_\omega\!=\!-0.5095$}]{\label{fig:d11}\includegraphics[width=0.15\textwidth]{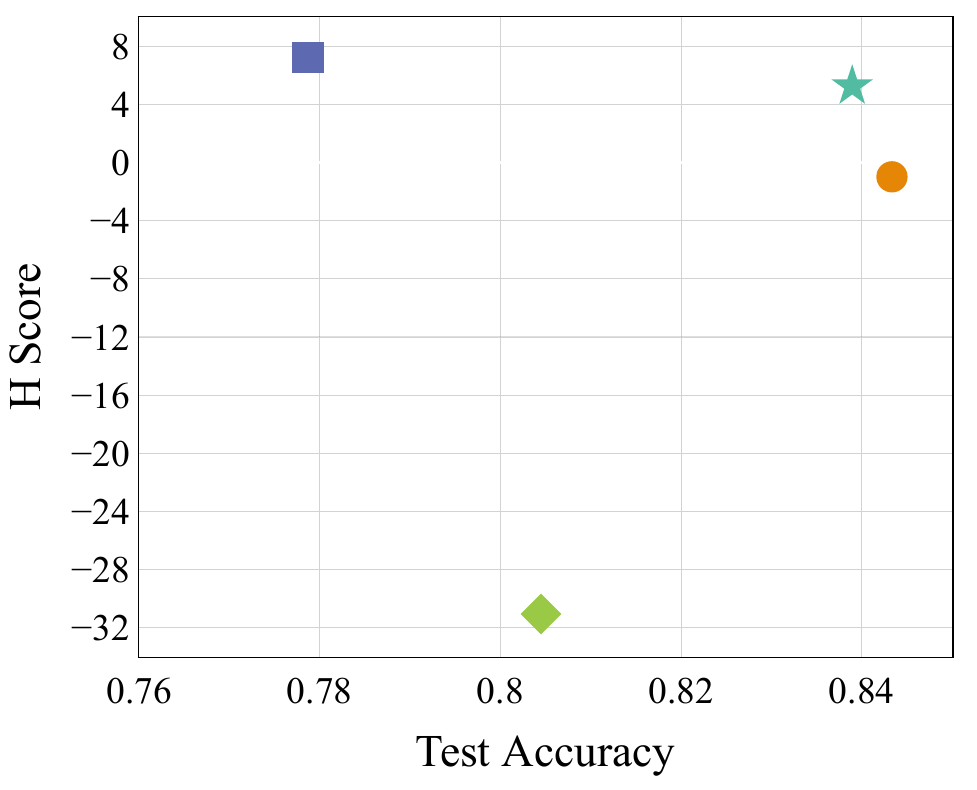} }
    \subfigure[][\tiny \makecell{ LFC on BACE\\ $R_p\!=\!0.1160$, \\ $\tau_K\!=\!-0.3333$, \\ $\tau_\omega\!=\!-0.4400$}]{\label{fig:d12}\includegraphics[width=0.15\textwidth]{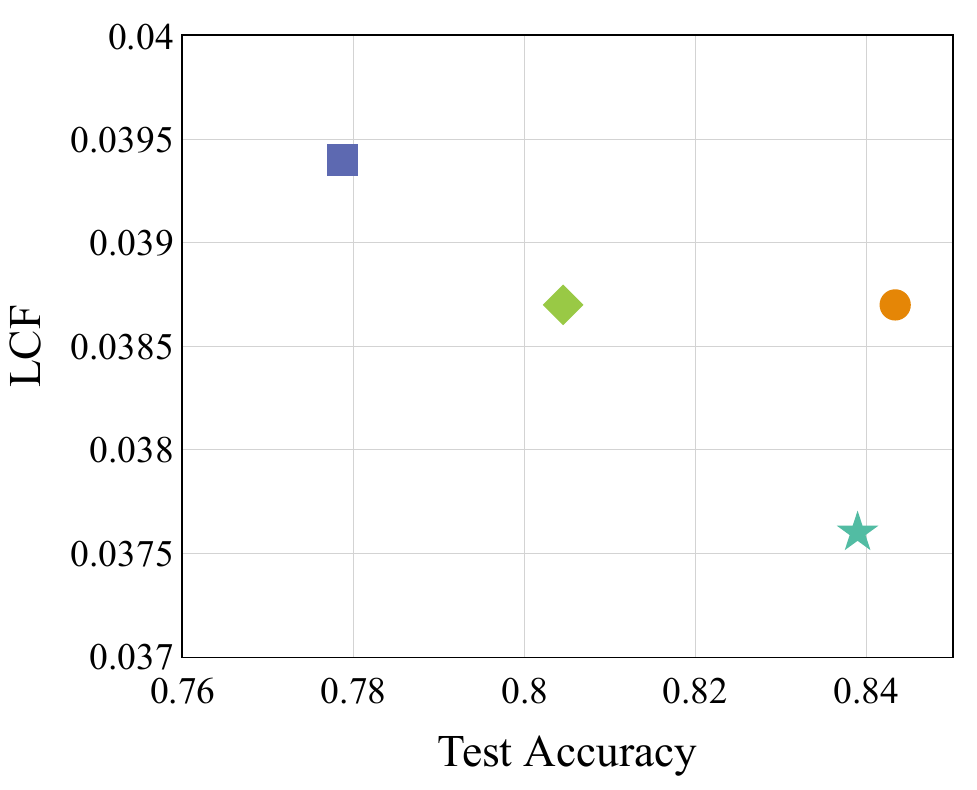} }
    \subfigure[][\tiny \makecell{ LogME on BACE\\ $R_p\!=\!-0.8625$, \\ $\tau_K\!=\!-1.0$, \\ $\tau_\omega\!=\!-1.0$}]{\label{fig:d13}\includegraphics[width=0.15\textwidth]{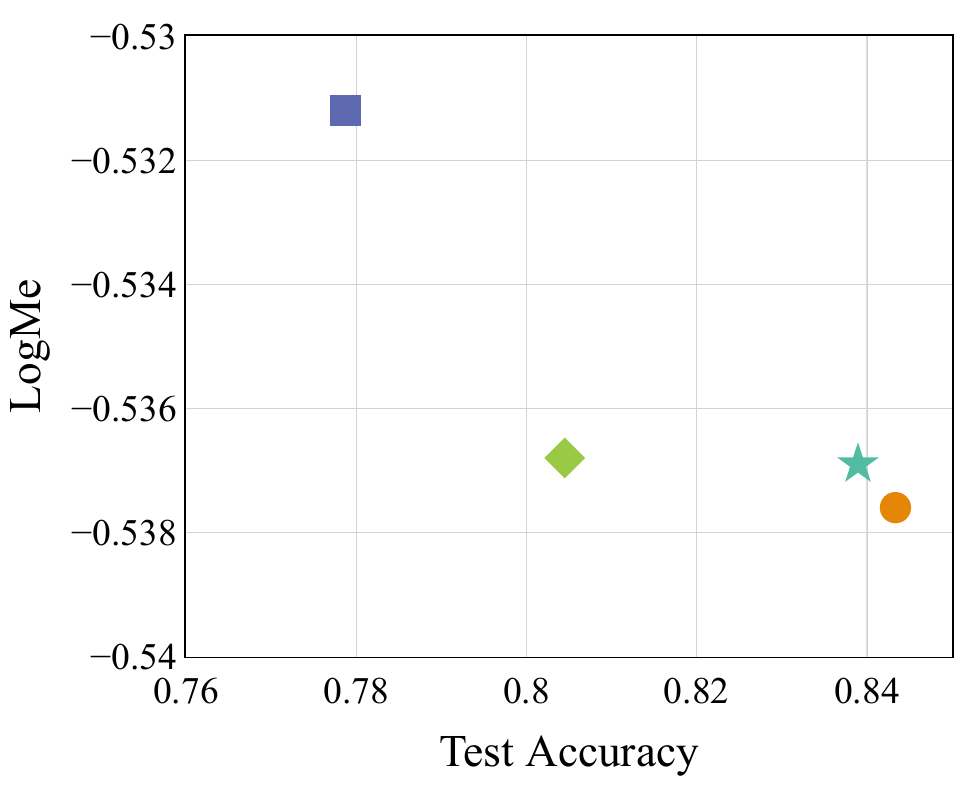} }
    \subfigure[][\tiny \makecell{ TrR on BACE\\ $R_p\!=\!{\bf0.9424}$, \\ $\tau_K\!=\!{\bf1.0}$, \\ $\tau_\omega\!=\!{\bf1.0}$}]{\label{fig:d14}\includegraphics[width=0.15\textwidth]{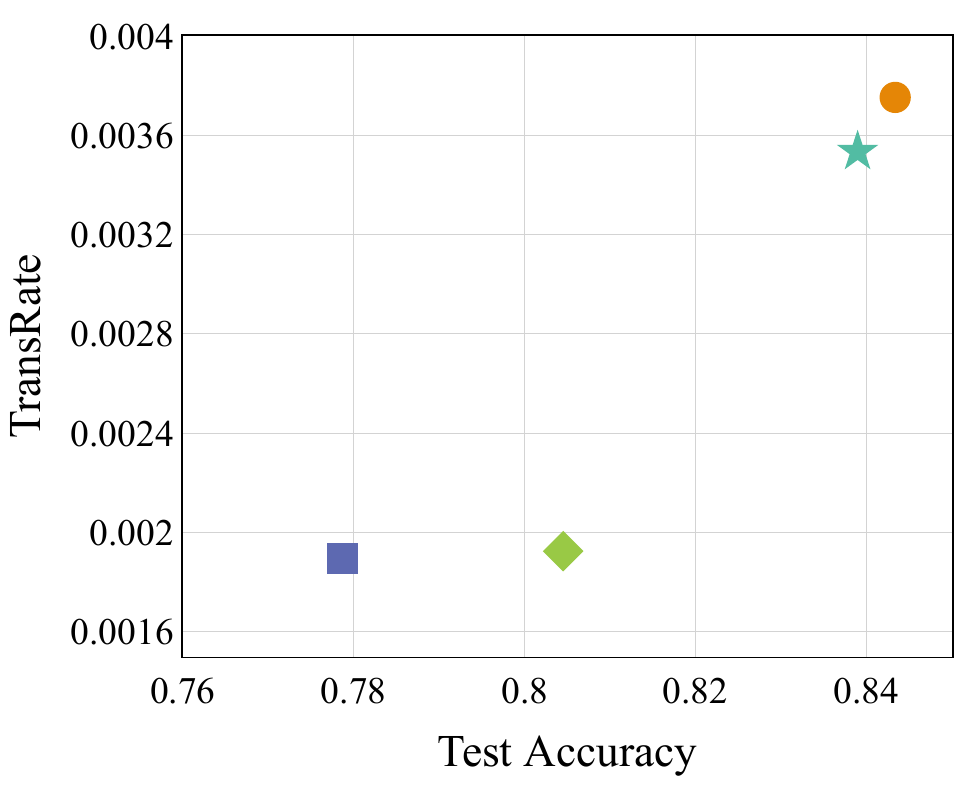} }
    \subfigure[][\tiny \makecell{ H-LogME on ESOL\\ $R_p\!=\!-0.3825$, \\ $\tau_K\!=\!-0.3333$, \\ $\tau_\omega\!=\!-0.4400$}]{\label{fig:d31}\includegraphics[width=0.15\textwidth]{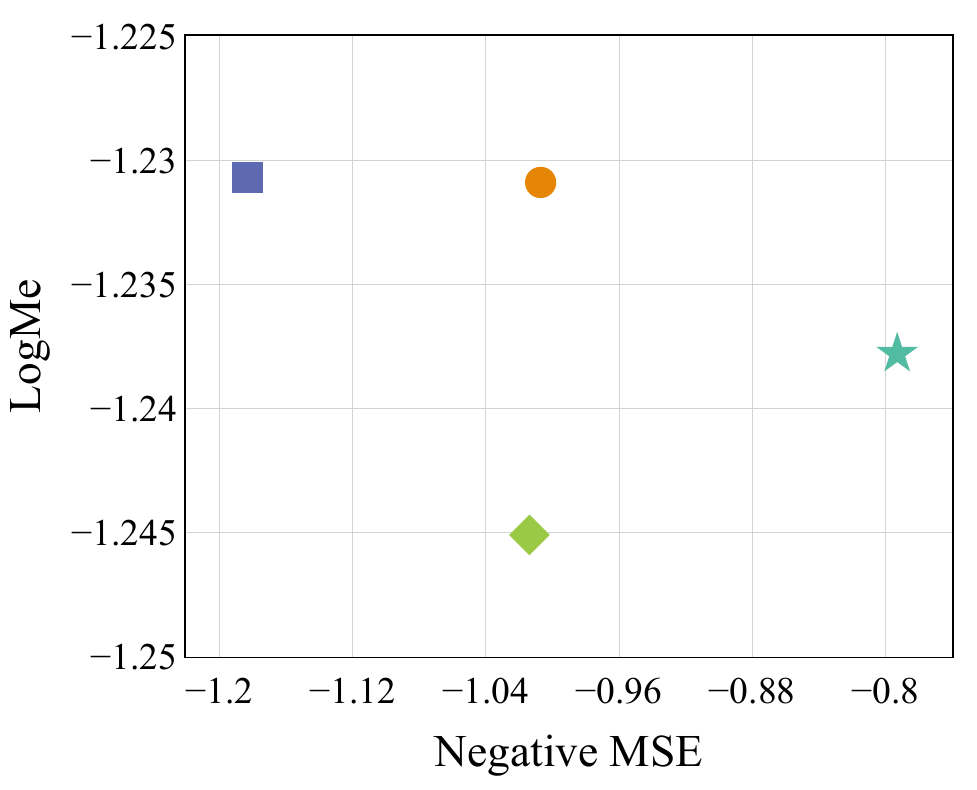} }
    \subfigure[][\tiny \makecell{ TrR on ESOL\\ $R_p\!=\!{\bf0.7422}$, \\ $\tau_K\!=\!{\bf1.0}$, \\ $\tau_\omega\!=\!{\bf1.0}$}]{\label{fig:d32}\includegraphics[width=0.15\textwidth]{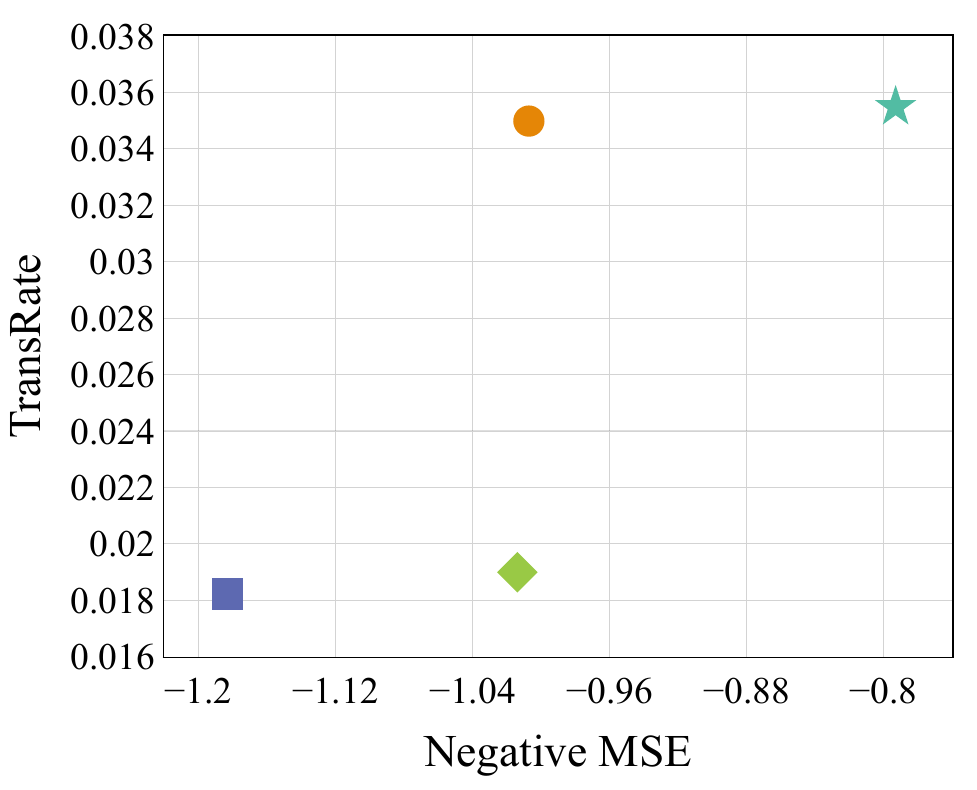} }
    \\
    \subfigure{\includegraphics[width=0.35\textwidth]{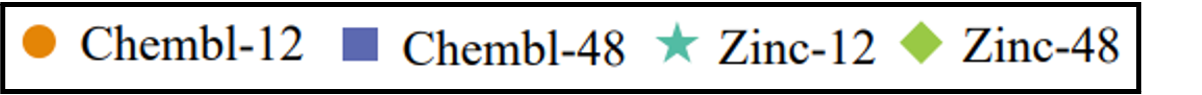}}
    \vspace{-2mm}
    \caption{\small Result on transfering GNNs pre-trained on molecules sampled from different datasets to molecule property prediction tasks.}\vspace{-4mm}
\label{fig:drug}
\end{figure*}

\textbf{TransRate as a Criterion for Layer Selection.}
As introduced in Section \ref{sec:intro}, transferring different layers of the pre-trained model results in different accuracies; that is, the optimal layers to transfer are task-specific. 
To study the correlation between all transferability measures and the performance of transferring different layers, we conduct experiments on transferring the first layer to the $K$-th layer only. 
\kai{In this experiment, we consider transferring a ResNet-20 model pre-trained on SVHN or a ResNet-18 model pre-trained on Birdsnap to CIFAR-100. The candidate values of $K$ for ResNet-20 and ResNet-18 are $\{9, 11, 13, 15, 17, 19\}$ and $\{11, 13, 15, 17\}$.}
\yingicml{The selected} 
layers to transfer 
are initialized by the pre-trained model and the remaining 
\yingicml{ones}
are trained from scratch. 
NCE and LEEP are excluded in this experiment as they are not applicable to layer selection. For \ours and other baselines, the transferability is estimated 
\yingicml{using} the features extracted by
the first $K$-th layer. Note that when $K$ is not the last layer, we will apply the average pooling function on the features, which is used by the original ResNet in the last layer. 
More details about the experimental settings are available in Appendix~\ref{appsec:a}. 

From 
Figure \ref{fig:layer_selection} 
we observe that \ours is the only method that correctly predicts the layer with the highest performance in both experiments. In the experiment transferring different layers of the pre-trained model from SVHN, \ours achieves the highest correlation coefficients. The baselines even have negative coefficients, which means that their predictions are inverse to the correct ranking. 
\kai{In the experiment transferring from Birdsnap, \ours correctly predicts the rank of the top 2 \yingicml{layers with the highest} transfer performance 
and also achieves the highest correlation coefficients.}
Both experiments demonstrate the superiority of \ours in selecting the best layer for transfer. More experiments of layer selection with different source datasets, models, and target datasets are available in Appendix~\ref{appsec:b2}. 

\begin{figure*}[t!]
\centering
    \subfigure[][\tiny \makecell{ LFC on CIFAR\\ $R_p\!=\!0.4261$, \\ $\tau_K\!=\!0.6667$, \\ $\tau_\omega\!=\!0.5200$}]{\label{fig:ss13}\includegraphics[width=0.18\textwidth]{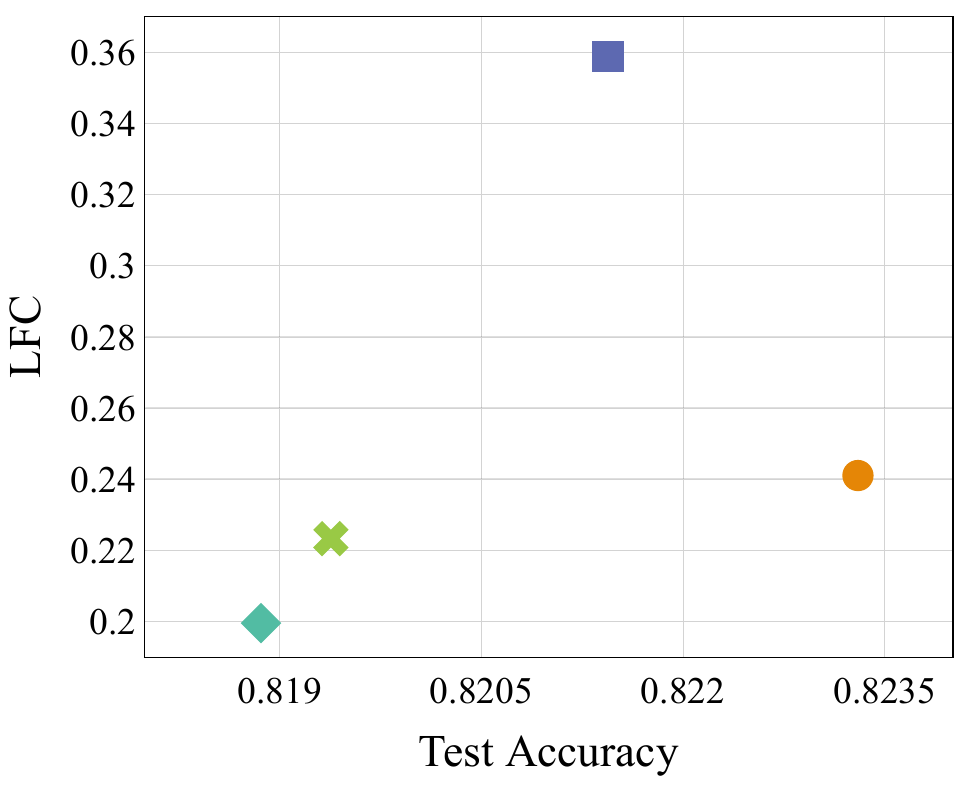} }
    \subfigure[][\tiny \makecell{ H-Score on CIFAR\\ $R_p\!=\!-0.9006$, \\ $\tau_K\!=\!{ -0.6667}$, \\ $\tau_\omega\!=\!-0.6667$}]{\label{fig:ss14}\includegraphics[width=0.18\textwidth]{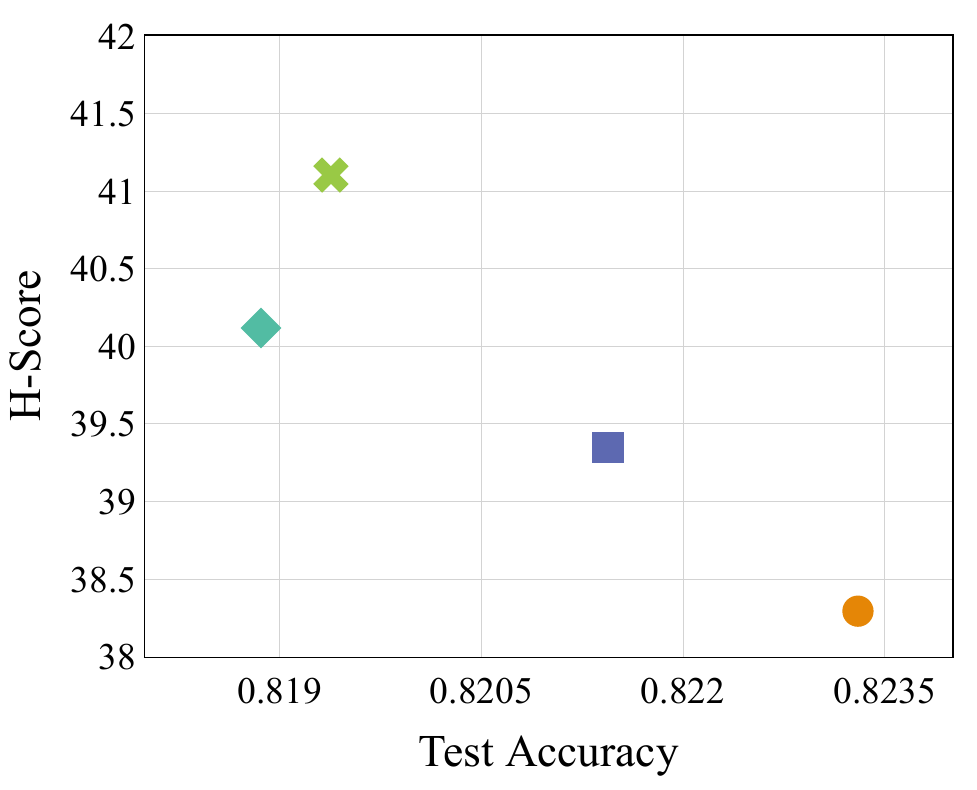} }
    \subfigure[][\tiny \makecell{ LogME on CIFAR \\ $R_p\!=\!-0.8595$, \\ $\tau_K\!=\!{ -0.6667}$, \\ $\tau_\omega\!=\! -.6667$}]{\label{fig:ss15}\includegraphics[width=0.18\textwidth]{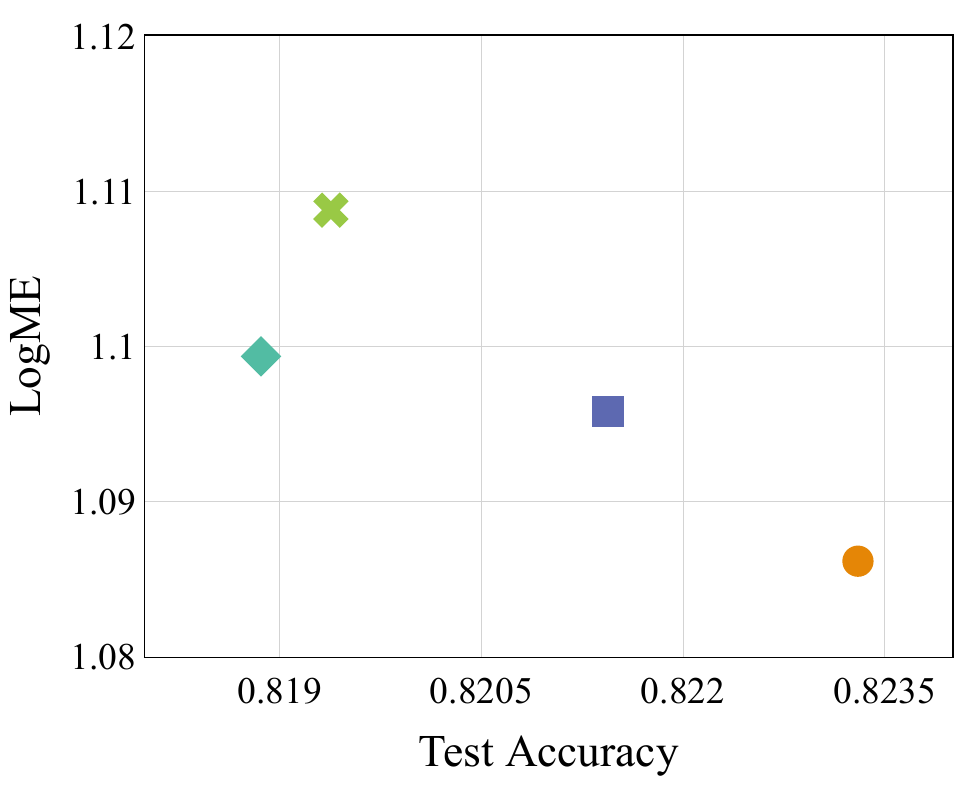} }
    \subfigure[][\tiny \makecell{ TrR on CIFAR\\ $R_p\!=\!{\bf 0.8550}$, \\ $\tau_K\!=\!{\bf 0.6667}$, \\ $\tau_\omega\!=\!{\bf0.8133}$}]{\label{fig:ss16}\includegraphics[width=0.18\textwidth]{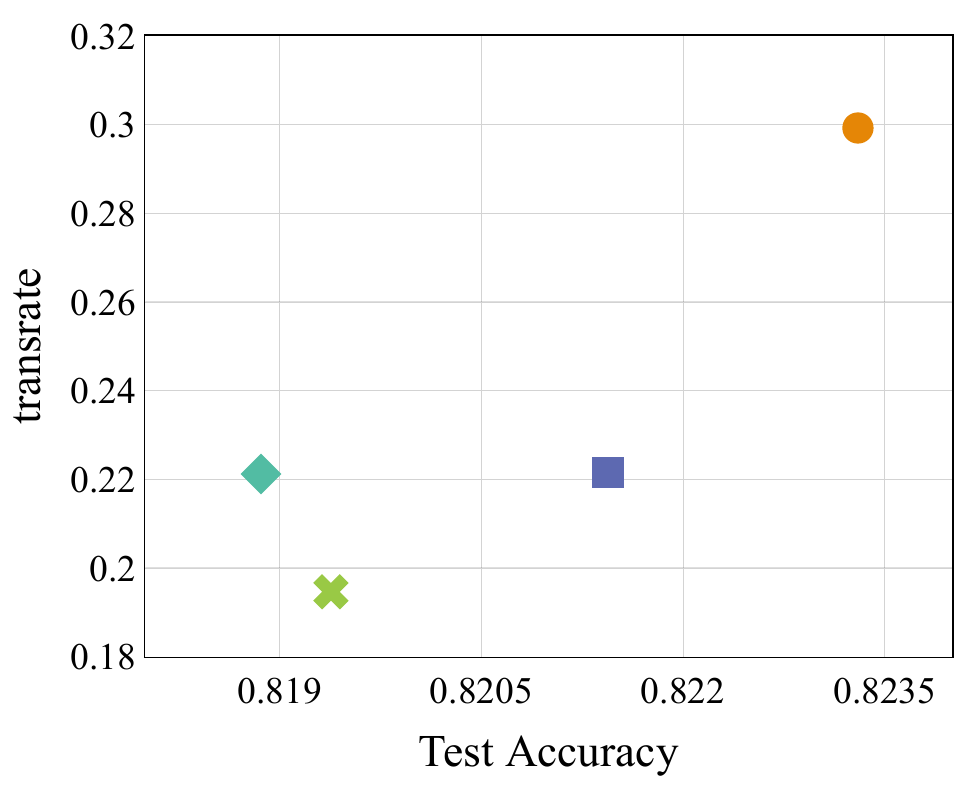} }
    \\
    \subfigure{\includegraphics[width=0.24\textwidth]{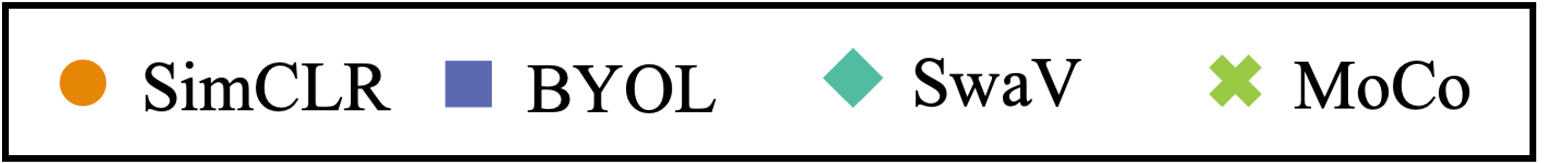}}
    \vspace{-1mm}
    \caption{\small Result on transferring ResNet50 pre-trained with different self-supervised algorithm\yingicml{s} from ImageNet to CIFAR-100.}\vspace{-4mm}
\label{fig:sslselection}
\end{figure*}

\textbf{TransRate as a Criterion for Pre-trained Model Selection.}
Another important application of transferability measures is the selection of pre-trained models in different architectures. 
In practice, various pre-trained models on public large datasets are available. For example, PyTorch provides more than 20 pre-trained neural networks for ImageNet.
To maximize the transfer performance from such a dataset, it is necessary to estimate the transferability of various model
candidates and select the one with the maximal score to transfer. 
In this experiment, we 
\yingicml{consider}
seven kinds of pre-trained models on ImageNet to CIFAR-100. The seven types 
include ResNet18~\cite{he2016deep}, ResNet34, ResNet50,  MobileNet0.25,  MobileNet0.5,  MobileNet0.75, MobileNet1.0~\cite{sandler2018mobilenetv2}.
\kai{
Figure \ref{fig:model_selection} 
\yingicml{tells}
that \ours in general has a significant linear correlation 
\yingicml{with}
the transfer accuracy, though it slightly underestimates MobileNet1.0.}
Though the predictions of LEEP and NCE achieve the best $R_p$, they
\yingicml{rank ResNet-18 incorrectly with underestimation.}
This also explains why they obtain lower $\tau_K$ and $\tau_\omega$ than TransRate. 
\yingicml{The performances of}
LFC, H-Score and LogME 
are not as competitive as NCE, LEEP and TransRate\yingicml{, though.} 
\kai{More experiments of model selection with more networks and target datasets are available in Appendix~\ref{appsec:model_selection}.}

\textbf{Estimation of Unsupervised
Pre-trained Models to Classification and Regression Tasks.}
We \yingicml{also} evaluate the effectiveness of TransRate and the baselines on estimating tranferability from different 
\yingicml{unsupervised}
pre-trained models. 
\yingicml{The first type of self-supervised models we consider is}
GROVER~\cite{rong2020self} \yingicml{for graph neural networks (GNN). We evaluate the transferability of four candidate models by varying two types of architectures and two types of pre-trained datasets, which are denoted by ChemBL-12, ChemBL-48, Zinc-12, Zinc-48.}
We consider four target tasks that predict the molecular ADMET properties, including BBBP~\cite{martins2012bayesian}, BACE~\cite{subramanian2016computational}, Esol~\cite{delaney2004esol}, FreeSolv~\cite{mobley2014freesolv}. 
The BBBP and BACE are classification tasks, while Esol and FreeSolv are regression tasks. More details about the settings of the pre-trained GNN models and the datasets are available in Appendix~\ref{appsec:a}. 
Figure \ref{fig:drug} 
\yingicml{show}
that in all four experiments, TransRate achieves the best performance regarding all $3$ coefficients. To be specific, TransRate correctly predicts 
\kai{\yingicml{the} ranking of all models in all experiment\yingicml{s} expcept \yingicml{the} Zinc-48 model on BBBP. 
\yingicml{while the baselines}}
all fail to predict the best model.

\kai{We also evaluate the performance in selecting \yingicml{4} model\yingicml{s} pre-trained on ImageNet by 4 self-supervised algorithms, including SimCLR~\cite{chen2020simple}, BYOL~\cite{grill2020bootstrap}, SwaV~\cite{caron2020unsupervised}, MoCo~\cite{he2020momentum}. 
Figure \ref{fig:sslselection} show that TransRate is the only method that correctly predicts the 
\yingicml{best-performing} 
model. \yingicml{Though} 
it overestimates the performance of SwaV, it still achieves the best correlation coefficients $R_p$, $\tau_K$ and $\tau_\omega$, outperforming the baseline methods by a large margin. \kai{Results on more target datasets are available in Appendix \ref{appsec:ssl_model_selection}.}
These results demonstrate the wide applicability as well as the effectiveness of TransRate 
\yingicml{in predicting the best unsupervised pre-trained model 
\kai{for} regression or classification target tasks.}
}

\begin{figure}[b]
\centering
    \vspace{-3mm}
    \subfigure{\includegraphics[width=0.15\textwidth]{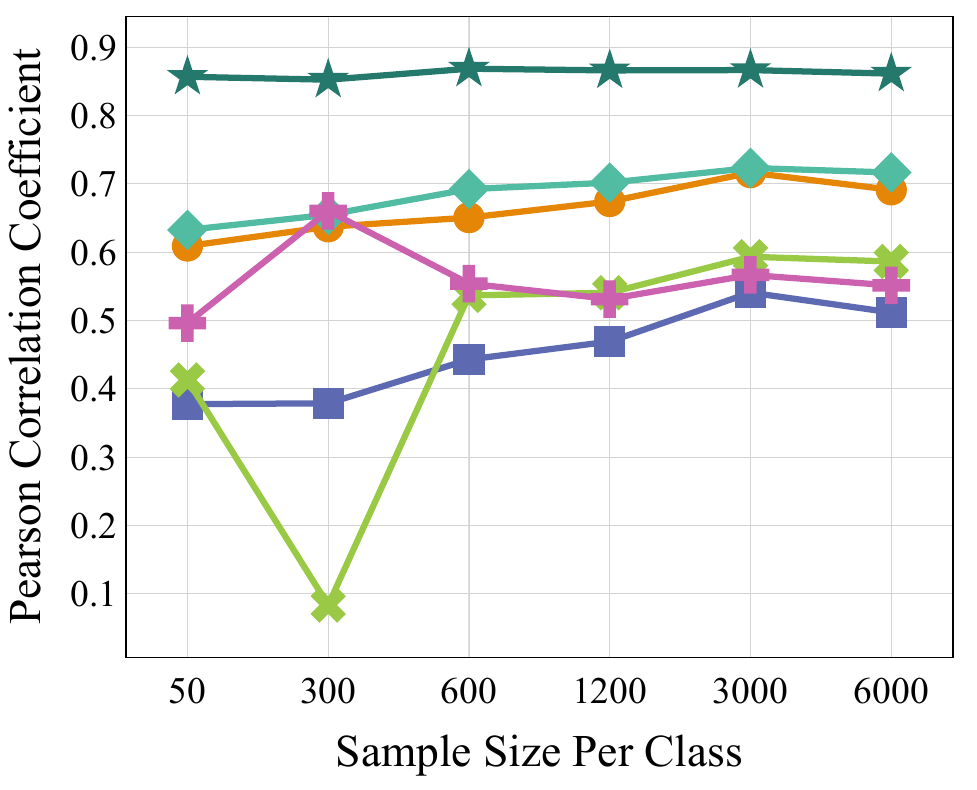} }
    \subfigure{\includegraphics[width=0.15\textwidth]{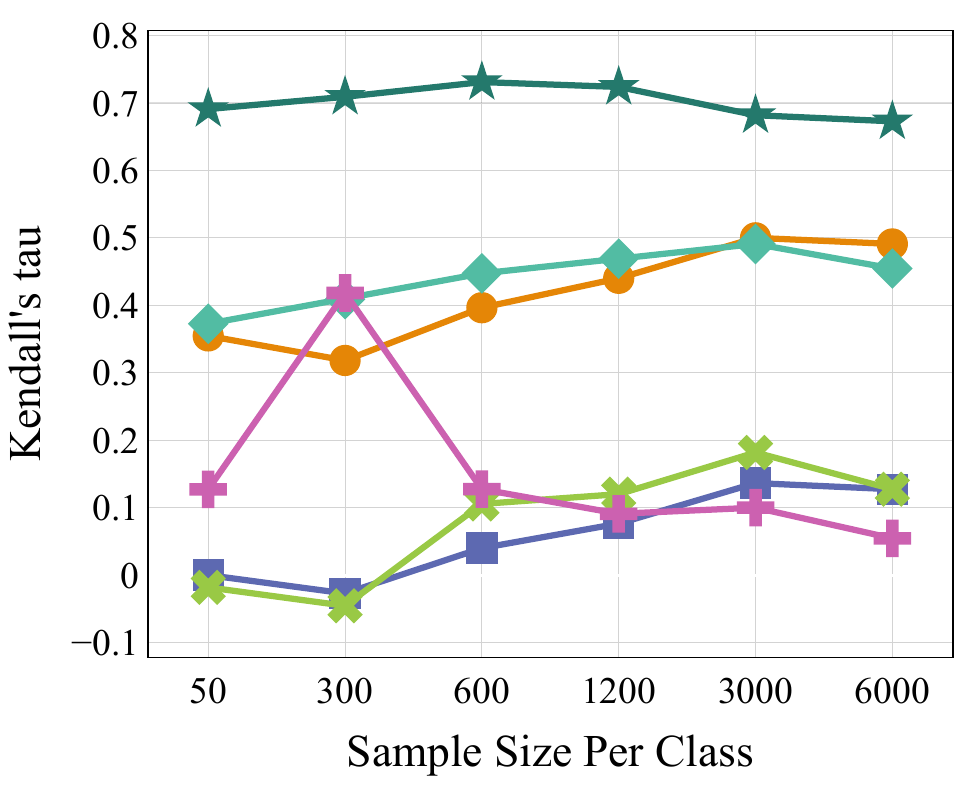}}
    \subfigure{\includegraphics[width=0.15\textwidth]{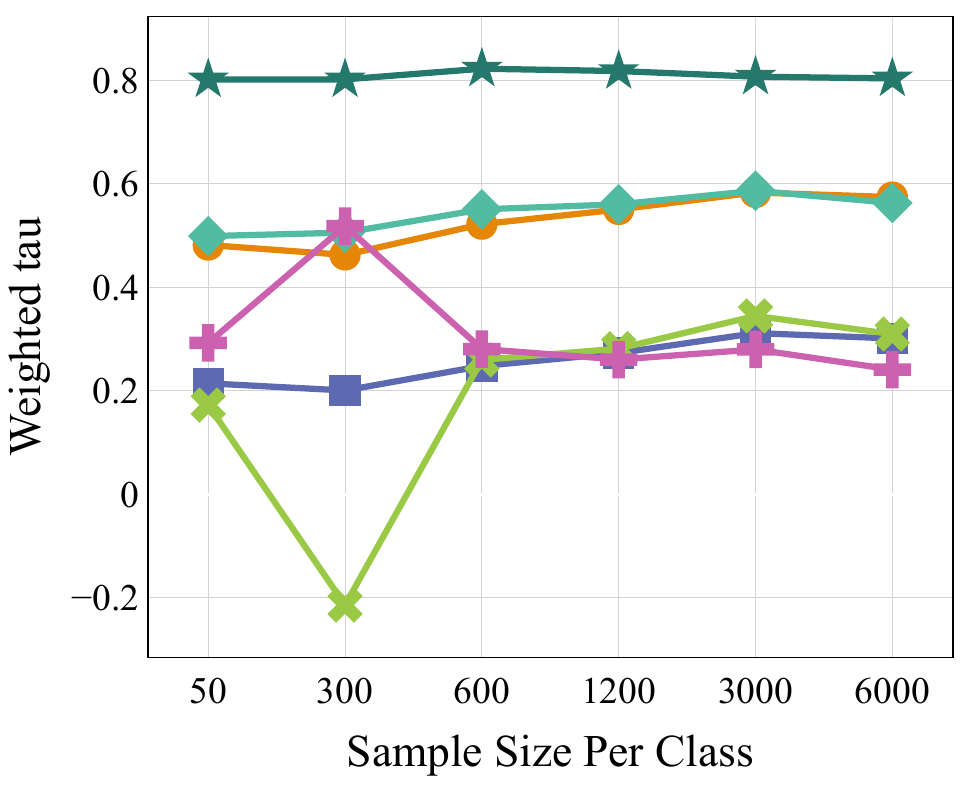} }
    \\
    \subfigure{\includegraphics[width=0.3\textwidth]{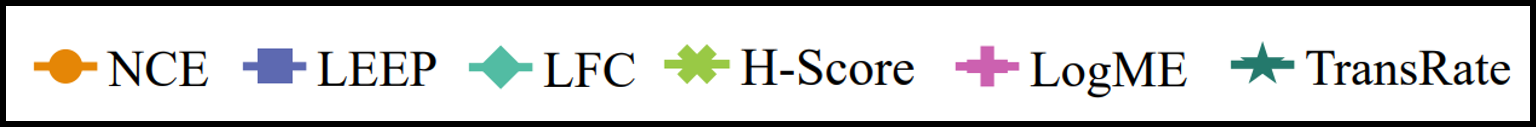}}
    \caption{ \small 
    Influence of the sample size of target datasets to the performance of a transferability measure,
    when fine-tuning the pre-trained ResNet-18 from $11$ different source datasets to FMNIST.}
\label{fig:sampel_si}
\end{figure}

\subsection{\yingicml{Discussion on Sensitivity to $\epsilon$ and Sample Size}}
\yingicml{As discussed in Figure~\ref{fig:measure_relationship}, the approximation error in \ours depends on 1) $\epsilon$ and 2) the sample size. By default, we have $\epsilon=$1E-4, while we would investigate}
the sensitivity of \ours to the value of \yingicml{$\epsilon$}. 
Appendix~\ref{appsec:b5} and Appendix~\ref{appsec:d4} \yingicml{ demonstrate that as long as $\epsilon$ is less than a threshold, the performance of \ours for estimating transferability and even the values of \ours barely change.}

\yingicml{It is inevitable for both \ours and the baselines to have the estimation error caused by a very limited number of samples that are insufficient to represent the true distribution. We would study}
\kai{
the sensitivity of \ours and the baseline methods with regard to the number of training samples in 
\yingicml{a}
target task. We adopt the same experiment settings as in \yingicml{source selection}, 
except that the number of samples available for each class varies from $50$ to $6000$. 
The trend\yingicml{s}
of the three types of correlation coefficients in Figure \ref{fig:sampel_si} 
\yingicml{speak}
that the performances of all algorithms generally drop 
when the sample size per class decreases. 
Unlike the baselines, the Kendall's $\tau$ and weighted $\tau$ of TransRate drop only by a minor percentage.
This shows its superiority in predicting the correct ranking of the models, even when only a small number of samples are available for estimation.
}

\section{Conclusion 
}\label{sec:conclusion}
In this paper, we propose a
\yingicml{frustratingly easy}
transferability measure named TransRate that flexibly supports estimation for transferring both holistic and partial layers of a pretrained model. 
TransRate estimates the mutual information between 
features extracted by a pre-trained model and labels with 
the coding rate. 
Both theoretic and empirical studies \kai{demostrate that} 
TransRate 
\yingicml{strongly correlates with}
the transfer learning performance, making it a qualified transferability measure for source dataset, model, and layer selection. 

\section*{Acknowledgements}
We would like to thank Yaodong Yu for the helpful discussion regarding the properties of coding rate. Ying Wei acknowledge the support of Project 9229073 by RMGS of Research Grants Council (RGC), Hong Kong.

\bibliography{transrate_ref}
\bibliographystyle{icml2022}

\appendix

\onecolumn


\section{Omitted Experiment Details in  Section 4}
\label{appsec:a}
\subsection{Image Datasets Description}
\label{appsec:a1}
\noindent  \textbf{Aircraft}~\cite{maji2013aircraft} \quad  The dataset consists of 10,000 aircraft images in 100 classes, with 66/34 or 67/33 training/testing images per class.

\noindent \textbf{Birdsnap}~\cite{berg2014birdsnap} \quad The dataset has 49,829 images of 500 species of North American Birds. It is divided into a training set with 32,677 images and a testing set with 8,171 images.

\noindent \textbf{Caltech-101}~\cite{fei2004caltech101} \quad The dataset contains 9,146 images from 101 object categories. The number of images in each category is between 40 and 800. Following \cite{salman2020adversarially}, we sample $30$ images per category as a training set and use the rest of images as a testing test.

\noindent  \textbf{Caltech-256}~\cite{griffin2007caltech256} \quad The dataset is an extension of Caltech 101. It contains 30,607 images spanning 257 categories. Each category has 80 to 800 images. We sample $60$ images per category for training and use the rest of images for testing. 

\noindent  \textbf{Cars}~\cite{krause2013cars} \quad The dataset consists of 16,185 images of 196 classes of cars. It is divided into a training set with 8,144 images and a testing set with 8,041 images. 

\noindent \textbf{CIFAR-10}~\cite{krizhevsky2009learning} \quad The dataset contains 60,000 color images in 10 classes, with each image in the size of 32$\times$32. Each class has 5,000 training samples and 1,000 testing samples.

\noindent \textbf{CIFAR-100}~\cite{krizhevsky2009learning} \quad The dataset is the same as CIFAR-10 except that it has 100 classes each of which contains 500 training images and 100 testing images.

\noindent \textbf{DTD}~\cite{cimpoi2014dtd} \quad The dataset consists of 5,640 textural images with sizes ranging between 300$\times$300 and 600$\times$600. There are a total of 47 categories, with 80 training and 40 testing images in each category.

\noindent \textbf{Fashion MNIST}~\cite{xiao2017fashion} \quad The dataset involves 70,000  grayscale images from 10 classes with the size of each image as 28 $\times$ 28. Each class has 6,000 training samples and 1,000 testing samples. \kai{Note that we limit the number of examples per class to be $300$ when using Fashion MNIST as a target dataset.}

\noindent \textbf{Flowers}~\cite{nilsback2008flowers} \quad The dataset consists of 102 categories of flowers that are common in the United Kindom. Each category contains between 40 and 258 images. We sample 20 images per category to construct the training set and use the rest of 6,149 images as the testing set.

\noindent \textbf{Food}~\cite{bossard2014food} \quad The dataset contains 101,000 images organized by 101 types of food. Each type of food contains 750 training images and 250 testing images. 

\noindent \textbf{Pets}~\cite{parkhi2012cats} \quad The dataset contains 7,049 images of pets belonging to 47 species. The training set contains 3,680 images and the testing set has 3,669 images.

\noindent \textbf{SUN397}~\cite{xiao2010sun} \quad  This dataset has 397 classes, each having 100 scenery pictures. For each class, there are 50 training samples and 50 testing samples.

\subsection{Molecule Datasets Description}
\label{appsec:a2}

\noindent \textbf{BBBP}~\cite{martins2012bayesian} \quad The dataset contains 2,039 compounds with binary labels of blood-brain barriers penetration. In the dataset, 1560 compounds are positive and 479 compounds are negative.

\noindent \textbf{BACE}~\cite{subramanian2016computational} \quad The dataset involves 1,513 recording compounds which could act as the inhibitors of human $\beta$-secretase 1 (BACE-1). 691 samples are positive and the rest 824 samples are negative.

\noindent \textbf{Esol}~\cite{delaney2004esol} \quad The dataset documents the water solubility (log solubility in mols per litre) for 1,128 small molecules.

\noindent \textbf{Freesolv}~\cite{mobley2014freesolv} \quad The dataset contains 642 records of the hydration free energy of small molecules in water from both experiments and alchemical free energy calculations.

Note that on both BBBP and BACE we perform binary classification by training with a binary cross-entropy loss.
Yet the regression tasks on Esol and Freesolv datasets are trained by an MSE loss. For all four datasets, we random split 80\% samples for training and 20\% samples for testing.

\subsection{Pre-trained Models}
\label{appsec:a3}

In the experiments, all models except GROVER used in Sec. 4.5 follow the standard architectures in PyTorch's torchvision. 
For those models pre-trained on ImageNet, we directly download them from PyTorch's torchvision. 
For those models pre-trained on other source datasets, we pre-train them with the hyperparameters obtained via grid search to guarantee the best performance. For GROVER~\cite{rong2020self}, we run the released codes provided by the authors on two large-scale unlabelled molecule datasets.
\yingicml{In total we have 4 pre-trained GROVER models, by considering two model architectures and pre-training on two datasets.
The first 
model contains about 12 million parameters while the 
second
has about 48 million parameters. Each of the two models is pre-trained on two unlabeled molecules datasets. The first one has 2 million molecules collected from ChemBL~\cite{gaulton2012chembl} and MoleculeNet~\cite{wu2018moleculenet}; the second one contains 11 million molecules sampled from ZINC15~\cite{sterling2015zinc} and ChemBL. In total, we have 4 different pre-trained models, denoted by ChemBL-12, ChemBL-48, Zinc-12, Zinc-48.}

\kai{
\subsection{Performance measure}
\label{appsec:a4}

We measure the performance of TransRaet and the baseline methods by Pearson correlation coefficient $R_p$, Kendall's $\tau$ and weighted $\tau$. The Pearson correlation coefficient measures the linear correlation between the predicted scores and  transfer accuracies of the transfer tasks. Kendall's $\tau$, also known as Kendall rank correlation coefficient, measures the rank correlation, i.e. the similarity between two rankings ordered by transfer accuracy and transferability scores. Weighted $\tau$ is a variant of Kendall's $\tau$ which focuses more on the top rankings. Since transferability scoring usually aims at selecting the best pre-trained model, we would highlight that weighted $\tau$ is the best one among these three measures for the transferability scoring method.
}

\subsection{Details about the Layer Selection Experiments in Section 4.3}
\label{appsec:a5}

In the layer selection experiments, we transfer the first $K$ layers from the pre-trained model and train the remaining layers from scratch. Notice that an average pooling function is applied to aggregate the outputs from different channels in the last layer of ResNet, which reduces the dimension of features. 
Inspired by this, we also apply an average pooling on the output of the $K$-th layer to reduce the feature dimension when we compute 
\kai{\ours and other baseline transferability measures,}
even if $K$ is not the last layer in layer selection. Besides the experiments on ResNet-20 \kai{and ResNet-18}, we also conduct layer selection on 
ResNet-34 and present the results in Appendix \ref{appsec:b2}. \kai{We consider only the last layer of a block in ResNet as a condidate layer}. The details of candidate layers and the feature dimensions are summarized in Table \ref{tab:layer}.


\begin{table*}[h]\caption{The configurations of layer selection for three model architectures.}
\small
\centering
\begin{tabular}{c  c c c c c c  c c c c c c}
\toprule
  \multirow{2}{*}{ResNet-20} & Candidate Layers: & 19 & 17 & 15 & 13 & 11 & 9 & & 
\\
    & Feature Dimension: & 64 & 64 & 64 & 32 & 32 & 32 & &
\\
\midrule
  \multirow{2}{*}{ResNet-18} & Candidate Layers: & 17 & 15 & 13 & 11 & & & &
\\
    & Feature Dimension: & 512 & 512 & 256 & 256 & & & &
\\
\midrule
  \multirow{2}{*}{ResNet-34} & Candidate Layers: & 33 & 30 & 27 & 25 & 23 & 21 & 19 & 17
\\
    & Feature Dimension: & 512 & 512 & 512 & 256 & 256 & 256 & 256 & 256 
\\
\bottomrule
\end{tabular}
\label{tab:layer}
\vspace{-0.1in}
\end{table*}

\subsection{Details about Applying TransRate on Regression Tasks}
\label{appsec:a6}

When estimating TransRate on regression target tasks, the label (or target value)
$y_i$ is not discrete so that we cannot directly compute $R(\hat{Z}, \epsilon | Y)$ in TransRate. 
The key insight behind TransRate for classification the overlap between the features of samples from different classes. Similarly, the transferability in a regression task can be estimated by the extent of the overlap between the
features of samples with different target values. 
To realize this, we rank all $n$ target values $\{y_i\}_{i=1}^n$ and divide them evenly into $C=10$ ranges. 
For samples in each range, we compute $R(\hat{Z}^c, \epsilon)$ where $\hat{Z}^c$ is the feature matrix of samples
in the $c$-th range. 
Finally, we calculate $R(\hat{Z}, \epsilon | Y) = \sum_{c=1}^C R(\hat{Z}^c, \epsilon)$ and subtract $R(\hat{Z}, \epsilon | Y)$ from $R(\hat{Z}, \epsilon)$ to obtain the resulting TransRate.

\subsection{Source Codes of TransRate}
\label{appsec:a7}

We implement the TransRate by Python. \kai{The codes are as follows:}

\begin{lstlisting}[language=Python]

import numpy as np

def coding_rate(Z, eps=1E-4):
    n, d = Z.shape
    (_, rate) = np.linalg.slogdet((np.eye(d) + 1 / (n * eps) * Z.transpose() @ Z))
    return 0.5 * rate
    
def transrate(Z, y, eps=1E-4):
    Z = Z - np.mean(Z, axis=0, keepdims=True)
    RZ = coding_rate(Z, eps)
    RZY = 0.
    K = int(y.max() + 1)
    for i in range(K):
        RZY += coding_rate(Z[(y == i).flatten()], eps)
    return RZ -  RZY / K

\end{lstlisting}

Here, we observe that TransRate can be implemented by \kai{10} lines of codes, which demonstrates its simplicity. It takes the features $\hat{Z}$ (``Z'' in the codes) and the labels $Y$ (``y'' in the codes), and \kai{distortion rate $\epsilon$ (``eps'' in the codes)} as input, calls the function ``coding\_rate'' to calculate the coding rate of $R(\hat{Z}, \epsilon)$ and $R(\hat{Z}, \epsilon | Y)$, and finally returns the TransRate.


\section{Extra Experiments}
\label{appsec:b}


\begin{table*}[t!]\caption{\small Transferability estimation on transferring ResNet-18 pre-trained on $11$ different source datasets to different target datasets. The best performance among all transferability measures is highlighted in bold
.}
\small
\center
\resizebox{0.75\textwidth}{!}{
\begin{tabular}{c  c c c c c c  c c c c c c}
\toprule
  Target Datasets
  &
  Measures & NCE & LEEP & LFC & H-Score & LogME & TransRate 
\\
\midrule
\multirow{3}{*}{CIFAR-100}
    & $R_p$ & 0.3803 & 0.2883 & 0.5330 & 0.5078 & 0.4947 & {\bf 0.7262} 
\\
    & $\tau_K$ & 0.3091 & 0.0909 & 0.6364 & 0.7091 & 0.7091 & {\bf 0.8182} 
\\
    & $\tau_\omega$ & 0.5680 & 0.3692 & 0.8141 & 0.8134 & 0.8134 & {\bf 0.9055}
\\
\midrule
\multirow{3}{*}{FMNIST}
    & $R_p$ & 0.6995 & 0.5200 & 0.7248 & 0.5945 & 0.5595 & {\bf 0.8614}
\\
    & $\tau_K$& 0.4909 & 0.1273 & 0.4545 & 0.1273 & 0.0545 & {\bf 0.6727}
\\
    & $\tau_\omega$ & 0.6114 & 0.3383 & 0.6001 & 0.3468 & 0.2781 & {\bf 0.8031}
\\
\midrule
\multirow{3}{*}{Aircraft}
    & $R_p$ & -0.8247 & -0.5721 & -0.6217 & {\bf 0.6758} & 0.6657 &  0.4722
\\
    & $\tau_K$& -0.5111 & -0.3333 & -0.6000 & {\bf 0.5556} & {\bf 0.5556} & 0.5111
\\
    & $\tau_\omega$ & -0.6180 & -0.4054 & -0.6854 & 0.5493 & 0.5493 & {\bf 0.5950}
\\
\midrule
\multirow{3}{*}{Birdsnap}
    & $R_p$ & 0.3595 & 0.5460 & 0.7361 & {\bf 0.8682} & 0.8502 &  0.8677
\\
    & $\tau_K$& 0.2000 & 0.4667 &  0.2889 & 0.6444 & 0.6444 & {\bf0.7333}
\\
    & $\tau_\omega$ & 0.3995 & 0.5579 & 0.1115 & 0.5530 & 0.5530 & {\bf 0.6354}
\\
\midrule
\multirow{3}{*}{Caltech-101}
    & $R_p$ & 0.2671 & 0.5255 & 0.5827 & {\bf 0.9058} & 0.8524 & 0.8962
\\
    & $\tau_K$& 0.1111 & 0.2444 & 0.4667 & {\bf 0.8667} & {\bf 0.8667} & {\bf 0.8667}
\\
    & $\tau_\omega$ & 0.3578 & 0.4457 &  0.4801 & {\bf 0.8381} & {\bf 0.8381} & {\bf 0.8381}
\\
\midrule
\multirow{3}{*}{Caltech-256}
    & $R_p$ & 0.3734 & 0.5655 & 0.5549 & {\bf 0.9080} & 0.8812 & 0.8600
\\
    & $\tau_K$&  0.2000 & 0.3333 & 0.3333 & {\bf 0.8667} & {\bf 0.8667} & 0.8222
\\
    & $\tau_\omega$  & 0.4953 & 0.5956 &  0.3415 & 0.8424 & 0.8424 & {\bf 0.9174}
\\
\midrule
\multirow{3}{*}{Cars}
    & $R_p$ & -0.6298 & -0.1296 & -0.0897 &  {\bf 0.8385} & 0.8302  & 0.7356
\\
    & $\tau_K$& -0.2444 & -0.2889 & -0.0667 & {\bf 0.7333} & {\bf 0.7333} &  0.6444
\\
    & $\tau_\omega$ & -0.1193 & 0.0408 &  0.0375 & {\bf 0.8019} & 0.7273 & 0.7597
\\
\midrule
\multirow{3}{*}{DTD}
    & $R_p$ & 0.0218 & 0.1662 & 0.5243 & 0.9208 & {\bf 0.9293} & 0.9131
\\
    & $\tau_K$& 0.1556  & 0.2444 & 0.4222 & 0.6000 & 0.7333 & {\bf 0.7778}
\\
    & $\tau_\omega$ & 0.1366 & 0.3519 & 0.3699 & 0.4409 & 0.7079 & {\bf 0.7755}
\\
\midrule
\multirow{3}{*}{Flowers}
    & $R_p$ & -0.3360 & -0.2790 & 0.2631 & 0.8385 & 0.7967 & {\bf0.9509}
\\
    & $\tau_K$& -0.2889 & -0.2444 & -0.0222 & 0.6000 & 0.6444 & {\bf 0.7778}
\\
    & $\tau_\omega$ & -0.1258 & 0.0865 & -0.1650 & 0.5176 & {0.6035} &  {\bf 0.7868}
\\
\midrule
\multirow{3}{*}{Food}
    & $R_p$ & 0.2485 & 0.4300  & 0.5656 &  {\bf 0.9243} & 0.9169  &  0.9065
\\
    & $\tau_K$& 0.1556 & 0.3333 & 0.2444 & {\bf 0.6000} & {\bf 0.6000} & {\bf 0.6000}
\\
    & $\tau_\omega$  & 0.3214 &  0.4878 & 0.0860 & 0.4927 & 0.4927 & {\bf 0.5335}
\\
\midrule
\multirow{3}{*}{Pets}
    & $R_p$ & 0.3512 & 0.4672 & 0.8306 & {\bf 0.9368} & 0.9019 & 0.8805
\\
    & $\tau_K$& 0.2444 & 0.5111 & {\bf 0.8667} & 0.8222 & 0.8222 & 0.7778
\\
    & $\tau_\omega$ & 0.3987 & 0.5679 & {\bf 0.8927} & 0.7351 & 0.7351 & 0.7148
\\
\midrule
\multirow{3}{*}{SUN397}
    & $R_p$ & 0.1535 & 0.4424 & 0.3693 & {\bf 0.9169} & 0.9058 & 0.7219
\\
    & $\tau_K$& 0.0222 & 0.2889 & 0.2000 & {\bf 0.7333} & {\bf 0.7333} & 0.5111
\\
    & $\tau_\omega$ & 0.3315 & 0.5159 & 0.1194 & 0.5928 & 0.5928 & {\bf 0.6424}
\\
\bottomrule
\end{tabular}
}
\label{tab:source_si}
\end{table*}

\begin{table*}[hp]\caption{\small Transferability estimation on transferring different layers of the pre-trained model to the CIFAR-100 dataset. The best performance among all transferability measures is highlighted in bold. }
\small
\center
\begin{tabular}{c  c c c c c c  c c c c c c}
\toprule
   &
  Measures  & LFC & H-Score & LogME & TransRate 
\\
\midrule
\multirow{3}{*}{\makecell{Source: SVHN \\ Model: ResNet-20}}
    & $R_p$ & -0.1895 & -0.5320 & -0.3352 & {\bf 0.9769}
\\
    & $\tau_K$ & -0.4667 & -0.2000 & -0.0667 & {\bf 0.8667}
\\
    & $\tau_\omega$ & -0.5497 & -0.2993 & -0.2340 & {\bf 0.9265}
\\
\midrule
\multirow{3}{*}{\makecell{Source: CIFAR-10 \\ Model: ResNet-20}}
    & $R_p$ & 0.5755 & 0.6476 & {\bf 0.6551} & 0.6347
\\
    & $\tau_K$ & {\bf 0.4667} & 0.2000 & 0.2000 & {\bf 0.4667}
\\
    & $\tau_\omega$ & 0.4041 & 0.3673 & 0.3673 & {\bf 0.5224}
\\
\midrule
\multirow{3}{*}{\makecell{Source: ImageNet \\ Model: ResNet-18}}
    & $R_p$ & 0.2595 & 0.9876 & {\bf 0.9898} & 0.9866
\\
    & $\tau_K$ & 0.0 & {\bf1.0} & {\bf1.0} & {\bf1.0}
\\
    & $\tau_\omega$ & 0.0 & {\bf1.0} & {\bf1.0} & {\bf1.0}
\\
\midrule
\multirow{3}{*}{\makecell{Source: ImageNet \\ Model: ResNet-34}}
    & $R_p$ & 0.6997 & 0.9357 & 0.9370 & {\bf 0.9550}
\\
    & $\tau_K$ & 0.3333 & {\bf 0.9444} & {\bf 0.9444} & {\bf 0.9444}
\\
    & $\tau_\omega$ & 0.4834 & {\bf 0.8674} & {\bf 0.8674} & {\bf 0.8674}
\\
\midrule
\multirow{3}{*}{\makecell{Source: Aircraft \\ Model: ResNet-18}}
    & $R_p$ & 0.6299 & 0.7983 & 0.0929 & {\bf0.9560}
\\
    & $\tau_K$ & 0.3333 & {\bf 0.6667} & 0.0000 & {\bf 0.6667}
\\
    & $\tau_\omega$ & 0.3333 & {\bf 0.8133} & 0.3067 & {\bf 0.8133}
\\
\midrule
\multirow{3}{*}{\makecell{Source: Birdsnap \\ Model: ResNet-18}}
    & $R_p$ & 0.7003 & 0.3166 & -0.5207 & {\bf 0.9871}
\\
    & $\tau_K$ & {\bf 0.6667} & {\bf 0.6667} & -0.3333 & {\bf 0.6667}
\\
    & $\tau_\omega$ & 0.5200 & 0.3067 & -0.2933 & {\bf0.8133}
\\
\midrule
\multirow{3}{*}{\makecell{Source: Caltech-101 \\ Model: ResNet-18}}
    & $R_p$ & 0.9310 & 0.9015 & 0.8561 & {\bf0.9871}
\\
    & $\tau_K$ & {\bf1.0} & 0.6667 & 0.6667 & {\bf1.0}
\\
    & $\tau_\omega$ & {\bf1.0} & 0.5200 & 0.5200 & {\bf1.0}
\\
\midrule
\multirow{3}{*}{\makecell{Source: Caltech-256 \\ Model: ResNet-18}}
    & $R_p$ & 0.8395 & 0.1649 & -0.4235 & {\bf0.9763}
\\
    & $\tau_K$ & 0.6667 & -0.3333 & -0.3333 & {\bf1.0}
\\
    & $\tau_\omega$ & 0.8133 & -0.2933 & -0.2933 & {\bf1.0}
\\
\midrule
\multirow{3}{*}{\makecell{Source: Cars \\ Model: ResNet-18}}
    & $R_p$ & -0.2438 & 0.3188 & -0.4489 & {\bf0.9790}
\\
    & $\tau_K$ & -0.3333 & 0.0000 & -0.3333 &{\bf 0.6667}
\\
    & $\tau_\omega$ & -0.4400 & 0.3067 & -0.2933 & {\bf0.8133}
\\
\midrule
\multirow{3}{*}{\makecell{Source: DTD \\ Model: ResNet-18}}
    & $R_p$ & 0.9542 & 0.8818 & 0.7200 &{\bf 0.9860}
\\
    & $\tau_K$ & {\bf1.0}& 0.6667 & 0.6667 & {\bf1.0}
\\
    & $\tau_\omega$ & {\bf1.0} & 0.5200 & 0.5200 & {\bf1.0}
\\
\midrule
\multirow{3}{*}{\makecell{Source: Flowers \\ Model: ResNet-18}}
    & $R_p$ & 0.8365 & 0.6054 & 0.0392 & {\bf0.9925}
\\
    & $\tau_K$ & 0.3333 & 0.6667 & 0.0 & {\bf1.0}
\\
    & $\tau_\omega$ & 0.3333 & 0.5200 & -0.0667 & {\bf1.0}
\\
\midrule
\multirow{3}{*}{\makecell{Source: Food \\ Model: ResNet-18}}
    & $R_p$ & 0.7963 & 0.5642 & -0.4487 & {\bf0.8002}
\\
    & $\tau_K$ & {\bf1.0} & 0.3333 & -0.3333 & {\bf1.0}
\\
    & $\tau_\omega$ & {\bf1.0} & 0.3333 & -0.2933 & {\bf1.0}
\\
\midrule
\multirow{3}{*}{\makecell{Source: Pets \\ Model: ResNet-18}}
    & $R_p$ & 0.7127 & 0.8349 & 0.6301 & {\bf0.9608}
\\
    & $\tau_K$ & 0.6667 & 0.6667 & 0.3333 & {\bf1.0}
\\
    & $\tau_\omega$ & 0.8133 & 0.5200 & 0.2000 & {\bf1.0}
\\
\midrule
\multirow{3}{*}{\makecell{Source: SUN397 \\ Model: ResNet-18}}
    & $R_p$ & 0.7647 & 0.6109 & 0.0409 & {\bf0.9761}
\\
    & $\tau_K$ & 0.6667 & 0.6667 & 0.0 & {\bf1.0}
\\
    & $\tau_\omega$ & 0.8133 & 0.5200 & -0.0667 & {\bf1.0}
\\
\midrule
\multirow{3}{*}{\makecell{Source: ImageNet \\ Model: ResNet-18 \\ Targe: FMNIST}}
    & $R_p$ & -0.0361 & 0.1775 & 0.1808 & {\bf0.9169}
\\
    & $\tau_K$ & -0.3333 & 0.0 & 0.0 & {\bf1.0}
\\
    & $\tau_\omega$ & -0.4400 & -0.1733 & -0.1733 & {\bf1.0}
\\

\bottomrule
\end{tabular}
\label{tab:layer_si}
\end{table*}

\subsection{Extra Experiments on Source Selection}
\label{appsec:b1}

In Table~\ref{tab:source_si}, we summarize the results of source selection that we have
presented in Section \ref{sec:exp_transfer_results} and meanwhile include the results of source selection for 9 more target datasets. 
Since the 11 source datasets include the 9 new target datasets, we remove the model that is pre-trained on the target dataset and consider only the 10 models that are pre-trained on the remaining source datasets. 

From Table \ref{tab:source_si}, we observe that TransRate achieves 20/36 best performance and 7/36 second best performances, which indicates the superiority of TransRate across the spectrum of different target datasets.
The most competitive baseline H-Score 
\kai{achieves 16/36 best performance. TransRate performs better in predicting the top transferability models, achieving 10/12 highest weighted $\tau$ and 7/12 highest Kendall's $\tau$, while H-Score correlates better linearly with the transfer accuracies, achieving 8/12 highest Pearson correlation coefficient.}
LogME has a similar performance to H-Score, but a little bit worse. As for NCE, LEEP, and LFC, their performances are not that satisfactory.
These results indicate that TransRate serves as an excellent transferability predictor for selection of a pre-trained model from various source datasets.

\subsection{Extra Results of Layer Selection}\label{appsec:b2}

In this subsection, we provide more experiments on the selection of a layer for 15 models pre-trained on different source datasets, including the two experiments presented in Section \ref{sec:exp_transfer_results}. The settings are the same as those in Section \ref{sec:exp_transfer_results}, except that the pre-trained models are trained on different source datasets or with different architectures. Table \ref{tab:layer_si} presents the results of the 15 experiments. 

As shown in Table \ref{tab:layer_si}, TransRate correctly predicts the performance ranking of transferring different layers in 9 of 15 experiments, where its $\tau_K$ and $\tau_\omega$ even equal to $1$. In contrast, the best competitor, LFC, achieves all-correct prediction in only 3 experiments. This shows the superiority of TransRate over the three baseline methods in serving as a criterion for layer selection. 
In rare cases, H-Score and LogME achieve competitive or even better performances than TransRate, while they 
fail in most of the experiments. This hit-and-miss behavior can be explained by the assumption behind H-Score and LogME as mentioned in Section \ref{sec:related} -- they consider the penultimate layer to be transferred only so that prediction of the transferability at other layers than the penultimate layer is not accurate.

\subsection{Extra Results on Model Selection}\label{appsec:model_selection}

In this subsection, we summarize experiment results on model selection for 8 target datasets, including the experiment for CIFAR-100 presented in Figure \ref{fig:model_selection} and new experiments for 7 new target datasets. 

As shown in Table \ref{tab:model_si}, TransRate achieves 16/24 best performance and correctly predicts the transferability ranking of all models in 3 experiments (i.e., Caltech-101, Caltech-256 and Pets). NCE and LEEP both achieve 3/8 best Pearson correlation coefficient and LogME achieves 3/8 best weighted $\tau$. However, in most experiments, their performances are not as competitive as TransRate. 

We also conduct model selection experiments with 6 more candidate architectures, including DenseNet121, DenseNet169, DenseNet201, InceptionV3, NASNet0.5 and NASNet1.0  and present the results in Table \ref{tab:model_si_extra}.
As more models are considered in the ranking, the model selection becomes more difficult. Compared to the result in Table \ref{tab:model_si}, the performance in most experiments drops. Even so, TransRate still achieves 18/24 best performance, significantly outperforming the baseline measures.

\begin{table*}[t!]\caption{\small Transferability estimation on transferring models with different architectures (ResNet18, ResNet34, ResNet50, MobileNet0.25, MobileNet0.5, MobileNet0.75 and MobileNet1.0) pre-trained on ImageNet to different target datasets. The best performance among all transferability measures is highlighted in bold
.}
\small
\center
\resizebox{0.7\textwidth}{!}{
\begin{tabular}{c  c c c c c c  c c c c c c}
\toprule
  Target Datasets
  &
  Measures & NCE & LEEP & LFC & H-Score & LogME & TransRate 
\\
\midrule
\multirow{3}{*}{CIFAR-100}
    & $R_p$ & 0.9654 & {\bf 0.9696} & 0.0664 & 0.3802 & 0.5672 & 0.8055 
\\
    & $\tau_K$ & 0.8095 & 0.8095 & -0.0476 & 0.3333 & 0.5238 & {\bf 0.9048} 
\\
    & $\tau_\omega$ & 0.7322 & 0.8650 & -0.0680 & 0.5041 & 0.6186 & {\bf 0.9421}
\\
\midrule
\multirow{3}{*}{FMNIST}
    & $R_p$ & -0.5561 & -0.4857 & -0.4234 & 0.2182 & 0.1140 & {\bf 0.3649}
\\
    & $\tau_K$& -0.3333 & -0.2381 & -0.3333 & 0.3333 & 0.1429 & {\bf 0.4286}
\\
    & $\tau_\omega$ & -0.3088 & -0.3751 & -0.3581 & 0.2978 & 0.1515 & {\bf 0.4870}
\\
\midrule
\multirow{3}{*}{Aircraft}
    & $R_p$ & -0.4664 & 0.7383 & 0.6676 & 0.1350 & 0.6925 & {\bf 0.7952}
\\
    & $\tau_K$& -0.2000 & 0.4667 & 0.6000 & 0.3333 & 0.6000 & {\bf 0.7333}
\\
    & $\tau_\omega$ & -0.2639 & 0.4136 & 0.5184 & 0.5374 & {\bf 0.6857} & 0.6599
\\
\midrule
\multirow{3}{*}{Caltech-101}
    & $R_p$ & {\bf 0.9779} & 0.9748 & 0.5583 & 0.1241 & 0.7894 & 0.9648
\\
    & $\tau_K$& 0.8095 & 0.8095 & 0.3333 & 0.2381 & 0.8095 & {\bf 1.0}
\\
    & $\tau_\omega$ & 0.7322 & 0.7322 &  0.2158 & 0.4345 & 0.8939 & {\bf 1.0}
\\
\midrule
\multirow{3}{*}{Caltech-256}
    & $R_p$ & {\bf 0.9861} & 0.9851 & 0.5476 & 0.4262 & 0.6998 & {0.9626}
\\
    & $\tau_K$&  0.8095 & 0.8095 & 0.4286 & 0.4286 & 0.7143 & {\bf 1.0}
\\
    & $\tau_\omega$  & 0.7322 & 0.7322 &  0.3133 & 0.5937 & 0.7868 & {\bf 1.0}
\\
\midrule
\multirow{3}{*}{Cars}
    & $R_p$ & 0.7317 & {\bf 0.9771} & 0.6308 & 0.1161 & 0.8319  & 0.7627
\\
    & $\tau_K$& 0.2381 & {\bf 0.8095} & 0.4286 & 0.3333 & 0.7143 & {\bf 0.8095}
\\
    & $\tau_\omega$ & 0.2308 & 0.6786 &  0.3186 & 0.5416 & {\bf 0.8168} & 0.8125
\\
\midrule
\multirow{3}{*}{Pets}
    & $R_p$ & {\bf 0.9881} & 0.9867 & 0.9001 & 0.5085 & 0.8531  & 0.8643
\\
    & $\tau_K$& 0.8095 & 0.9048 & 0.7143 & 0.4286 & {\bf 1.0} & {\bf 1.0}
\\
    & $\tau_\omega$ & 0.7322 & 0.9250 &  0.5286 & 0.5937 & {\bf 1.0} & {\bf 1.0}
\\
\midrule
\multirow{3}{*}{SUN397}
    & $R_p$ & 0.9612 & {\bf 0.9638} & 0.5854 & 0.3117 & 0.7723  & 0.9609
\\
    & $\tau_K$& 0.7143 & 0.8095 & 0.5238 & 0.4286 & 0.7143 & {\bf 0.9048}
\\
    & $\tau_\omega$ & 0.5983 & 0.6786 & 0.4633 & 0.6616 & 0.8168 & {\bf 0.8929}
\\
\bottomrule
\end{tabular}
}
\label{tab:model_si}
\end{table*}

\begin{table*}[t!]\caption{\small Transferability estimation on transferring models with different architectures (ResNet18, ResNet34, ResNet50, MobileNet0.25, MobileNet0.5, MobileNet0.75, MobileNet1.0, \emph{DenseNet121, DenseNet169, DenseNet201, InceptionV3 NASNet0.5 and NASNet1.0}) pre-trained on ImageNet to different target datasets. 
The best performance among all transferability measures is highlighted in bold
. }
\small
\center
\resizebox{0.7\textwidth}{!}{
\begin{tabular}{c  c c c c c c  c c c c c c}
\toprule
  Target Datasets
  &
  Measures & NCE & LEEP & LFC & H-Score & LogME & TransRate 
\\
\midrule
\multirow{3}{*}{CIFAR-100}
    & $R_p$ & 0.7937 & 0.8506 & -0.2159 & 0.5016 & 0.4965 & {\bf 0.8780}
\\
    & $\tau_K$ & 0.7436 & 0.7179 & -0.0256 & 0.4872 & 0.4103 & {\bf 0.9231} 
\\
    & $\tau_\omega$ & 0.8315 & 0.8485 & -0.0126 & 0.6058 & 0.5130 & {\bf 0.8498}
\\
\midrule
\multirow{3}{*}{FMNIST}
    & $R_p$ & 0.2708 & 0.3522 & 0.3085 & 0.6226 & 0.1521 & {\bf 0.7086}
\\
    & $\tau_K$& 0.0769 & 0.1795 & 0.1795 & 0.5897 & 0.0 & {\bf 0.7179}
\\
    & $\tau_\omega$ & 0.2091 & 0.4351 & 0.2230 & 0.6670 & -0.1171 & {\bf 0.8592}
\\
\midrule
\multirow{3}{*}{Aircraft}
    & $R_p$ & -0.4969 & 0.7260 & -0.3718 & 0.3196 & 0.4179 & {\bf 0.6320}
\\
    & $\tau_K$& -0.3939 & 0.2727 & 0.2121 & 0.3939 & 0.4848 & {\bf 0.5758}
\\
    & $\tau_\omega$ & -0.3585 & 0.1134 & 0.0290 & 0.4971 & 0.6519 & {\bf 0.6838}
\\
\midrule
\multirow{3}{*}{Caltech-101}
    & $R_p$ & 0.7955 & 0.8445 & -0.1485 & 0.3112 & 0.5336 & 0.6392
\\
    & $\tau_K$& 0.6410 & 0.6154 & -0.1026 & 0.5385 & 0.6923 & {\bf 0.7692}
\\
    & $\tau_\omega$ & 0.6358 & 0.5380 & -0.3313 & 0.7665 & 0.8214 & {\bf 0.8511}
\\
\midrule
\multirow{3}{*}{Caltech-256}
    & $R_p$ & {\bf 0.9339} & 0.9199 & 0.3244 & 0.6125 & 0.6220 & {0.8100}
\\
    & $\tau_K$&  0.7949 & 0.7179 & 0.1538 & 0.7179 & 0.7949 & {\bf 0.8974}
\\
    & $\tau_\omega$  & 0.6922 & 0.6253 & -0.0587 & 0.8668 & 0.8673 & {\bf 0.8971}
\\
\midrule
\multirow{3}{*}{Cars}
    & $R_p$ & 0.4317 & {\bf 0.8114} & -0.1935 & 0.2936 & 0.5289  & 0.7309
\\
    & $\tau_K$& 0.3590 & 0.7692 & 0.2821 & 0.4103 & 0.6154 & {\bf 0.7949}
\\
    & $\tau_\omega$ & 0.5391 & {\bf 0.7974} & 0.2155 & 0.5620 & 0.6950 & 0.6894
\\
\midrule
\multirow{3}{*}{Pets}
    & $R_p$ & 0.9681 & {\bf 0.9787} & 0.6892 & 0.6333 & 0.7098  & 0.8143
\\
    & $\tau_K$& 0.7692 & 0.8205 & 0.5128 & 0.6154 & 0.8462 & {\bf 0.8718}
\\
    & $\tau_\omega$ & 0.7178 & 0.7741 & 0.6214 & 0.6901 & 0.8439 & {\bf 0.8530}
\\
\midrule
\multirow{3}{*}{SUN397}
    & $R_p$ & {\bf 0.9513} & 0.9166 & 0.3982 & 0.5834 & 0.7053 & 0.8380
\\
    & $\tau_K$& 0.7692 & 0.7692 & 0.3077 & 0.6667 & 0.7692 & {\bf 0.8462}
\\
    & $\tau_\omega$ & 0.7332 & 0.7368 & 0.1686 & 0.7627 & 0.7577 & {\bf 0.7998}
\\
\bottomrule
\end{tabular}
}
\label{tab:model_si_extra}
\end{table*}

\subsection{Extra Results on Self-supervised Model Selection}\label{appsec:ssl_model_selection}

In this subsection, we conduct extra experiments on model selection among 4 self-supervised algorithms for 8 new target datasets and summarize their results and the experiment result presented in Figure \ref{fig:sslselection} in Section \ref{sec:exp_transfer_results}. 

The results are presented in Table \ref{tab:ssl_model_si}. TransRate significantly outperforms the baseline methods. It achieves the best performance in all experiments except the experiment for Caltech-101. In the experiments for FMNIST, Caltech-256, Flowers and SUN397, TransRate corrrectly predict the ranking of all models, obtaining $\tau_K=1$ and $\tau_\omega=1$. The baseline methods all achieve only 3/27 best performance, but underperfom in most experiments.

\begin{table*}[t!]\caption{\small Transferability estimation on transferring models pre-trained with different self-supervised learning algorithms on ImageNet to different target datasets. The best performance among all transferability measures is highlighted in bold
. }
\small
\center
\resizebox{0.55\textwidth}{!}{
\begin{tabular}{c  c c c c c c  c c c c}
\toprule
  Target Datasets
  &
  Measures & LFC & H-Score & LogME & TransRate 
\\
\midrule
\multirow{3}{*}{CIFAR-100}
    & $R_p$ & 0.4261 & -0.9006 & -0.8595 & {\bf 0.8550}
\\
    & $\tau_K$ & 0.6667 & -0.6667 & -0.6667 & {\bf 0.6667}
\\
    & $\tau_\omega$ & 0.5200 & -0.6667 & -0.6667 & {\bf0.8133}
\\
\midrule
\multirow{3}{*}{FMNIST}
    & $R_p$ & 0.6259 & 0.2483 & 0.3307 & {\bf 0.9955}
\\
    & $\tau_K$ & {\bf1.0} & 0.6667 & 0.6667 & {\bf1.0}
\\
    & $\tau_\omega$ & {\bf1.0} & 0.8133 & 0.8133 & {\bf1.0}
\\
\midrule
\multirow{3}{*}{Aircraft}
    & $R_p$ & -0.3821 & 0.2395 & 0.0404 & {\bf 0.9688}
\\
    & $\tau_K$ & -0.3333 & 0.3333 & 0.0 & {\bf 0.6667}
\\
    & $\tau_\omega$ & -0.4400 & 0.2000 & -0.1733 & {\bf 0.7333}
\\
\midrule
\multirow{3}{*}{Birdsnap}
    & $R_p$ & 0.3656 & 0.5404 & 0.1623 & {\bf 0.6397}
\\
    & $\tau_K$ & {\bf 0.6667} & 0.3333 & 0.0 & {\bf 0.6667}
\\
    & $\tau_\omega$ & {\bf 0.5200} & 0.3333 & 0.0 & {\bf 0.5200}
\\
\midrule
\multirow{3}{*}{Caltech-101}
    & $R_p$ & 0.6620 & 0.7456 & {\bf 0.8963} & -0.0979
\\
    & $\tau_K$ & 0.3333 & {\bf 0.6667} & {\bf 0.6667} & 0.3333
\\
    & $\tau_\omega$ & 0.5333 & {\bf 0.8133} & 0.5200 & 0.5333
\\
\midrule
\multirow{3}{*}{Caltech-256}
    & $R_p$ & 0.2690 & -0.5860 & -0.4753 & {\bf 0.9239}
\\
    & $\tau_K$ & 0.3333 & -0.3333 & -0.3333 & {\bf1.0}
\\
    & $\tau_\omega$ & 0.2000 & -0.2933 & -0.2933 & {\bf1.0}
\\
\midrule
\multirow{3}{*}{Cars}
    & $R_p$ & -0.4311 & 0.3379 & 0.1503 & {\bf 0.7498}
\\
    & $\tau_K$ & -0.3333 & {\bf0.3333}& {\bf0.3333} & {\bf0.3333}
\\
    & $\tau_\omega$ & -0.4400 & 0.2000 & 0.2000 & {\bf0.5333}
\\
\midrule
\multirow{3}{*}{Flowers}
    & $R_p$ & 0.3755 & 0.7177 & 0.3576 & {\bf 0.8077}
\\
    & $\tau_K$ & -0.3333 & 0.0 & 0.0 & {\bf1.0}
\\
    & $\tau_\omega$ & -0.4400 & -0.1733 & -0.0667 & {\bf1.0}
\\
\midrule
\multirow{3}{*}{SUN397}
    & $R_p$ & -0.4770 & 0.4464 & 0.2504 & {\bf 0.8180}
\\
    & $\tau_K$ & -0.3333 & 0.0 & 0.0 & {\bf1.0}
\\
    & $\tau_\omega$ & -0.4400  & 0.0 & 0.0 & {\bf1.0}
\\
\bottomrule
\end{tabular}
}
\label{tab:ssl_model_si}
\end{table*}

\subsection{Extra Experiments on Sample Size Sensitivity Study}
\label{appsec:b3}

\begin{figure*}[h]
\centering
    \vspace{-3mm}
    \subfigure{\includegraphics[width=0.22\textwidth]{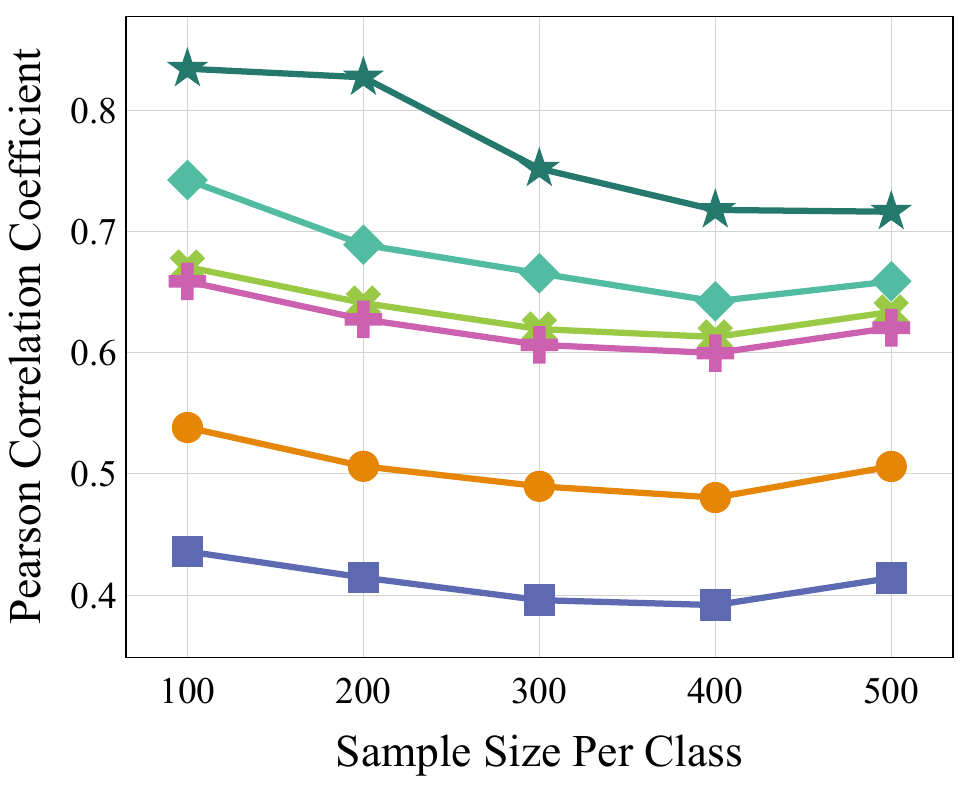} }
    \subfigure{\includegraphics[width=0.22\textwidth]{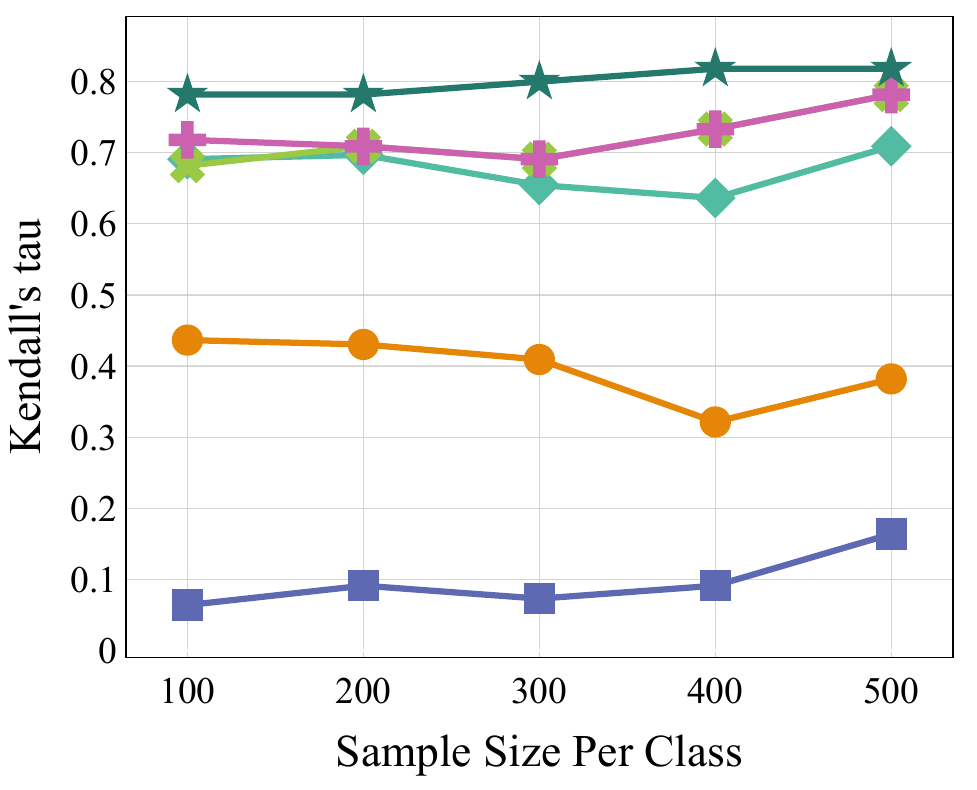}}
    \subfigure{\includegraphics[width=0.22\textwidth]{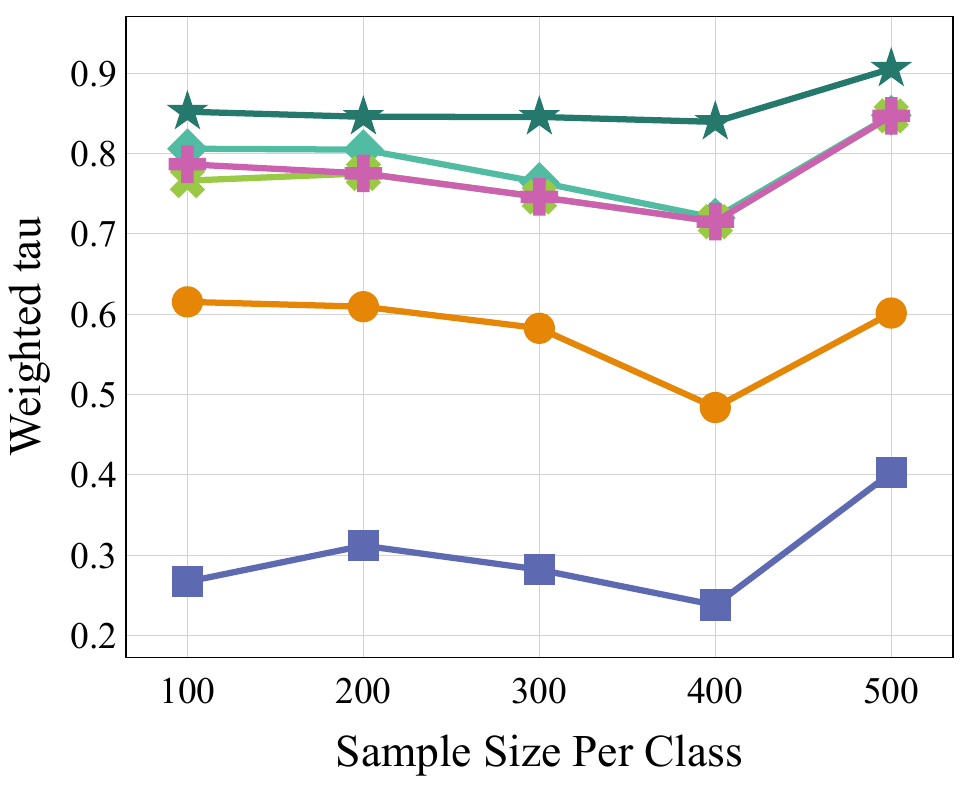} }
    \\
    \subfigure{\includegraphics[width=0.35\textwidth]{legend/sample_legend.png}}
    \caption{ \small 
    The three types of correlation coefficients between estimated transferability and test accuracy when varying the size of the target dataset. The correlation coefficient is measured by a series of transfer tasks that fine-tune
    the pre-trained ResNet-18 from $11$ different source datasets to CIFAR-100.}
\label{fig:sampel_si_cifar}
\end{figure*}

\kai{
We provide a supplementary experiment to further investigate the influence of sample size on the performance of transferability estimation algorithms. 
We vary the number of samples per classs in CIFAR-100 from $100$ to $500$, and visulize the trend
of the three types of correlation coefficients in Figure \ref{fig:sampel_si_cifar}. 

From Figure \ref{fig:sampel_si_cifar}, we observe that the ranking prediction coefficients (i.e. $\tau_K$ and $\tau_\omega$) of all algorithms generally drops when the sample size per class decreases from $500$ to $50$. The deterioration of the performance is caused by the inaccurate estimation given a small number of samples. 
Even with slight deterioration given fewer samples, TransRate still outperforms the baselines.
This shows its superiority in predicting the correct ranking of the models, even when only a small number of samples are available for estimation.
}

\subsection{Time Complexity}
\label{appsec:b4}

\begin{table*}[t]\caption{\small Comparison of the computational cost of different measures.}
\small
\resizebox{.99\textwidth}{!}{
\centering
\begin{tabular}{c   c c  c c  c c}
\toprule
   & \multicolumn{2}{c}{ResNet-18, Full Data} & \multicolumn{2}{c}{ResNet-18, Small Data} 
   & \multicolumn{2}{c}{ResNet-50, Full Data}  
\\
 \cmidrule(lr){2-3} \cmidrule(lr){4-5} \cmidrule(lr){6-7}
   & Wall-clock time (second) & Speedup & Wall-clock time (second) & Speedup & Wall-clock time (second) & Speedup
\\
\midrule
    Fine-tune & 8399.65 & 1$\times$ & 882.33 & 1$\times$ & $2.3\times10^4$ & 1$\times$
\\
\midrule
    Extract feature & 30.1416 &  & 3.2986 &  & 72.787 & 
\\
\midrule
    NCE & 0.9126 & 9,204$\times$ & 0.2119 & 4,164$\times$ & 2.1220 & 10,839$\times$
\\
    LEEP & 0.7771 & 10,808$\times$ & 0.1211 & 7,286$\times$ & 1.9152 & 12,009$\times$
\\
\midrule
    LFC & 30.1416 & 279$\times$ & 0.7987 & 1,106$\times$ & 149.3040 & 154$\times$
\\
    H-Score & 1.6285 & 5,158$\times$ & 0.3998 & 2,207$\times$ & 13.07 & 1,760$\times$
\\
    LogME & 9.2737 & 906$\times$ & 2.0224 & 436$\times$ & 50.1797 & 458$\times$
\\
    TransRate & \textbf{1.3410} & \textbf{6,264$\times$} & \textbf{0.2697} & \textbf{3,272$\times$} & \textbf{10.6498} & \textbf{2,160$\times$}
\\
\bottomrule
\end{tabular}
}
\label{tab:time}
\end{table*}

In this subsection, we compare the running time of TransRate as well as the baselines. We run the experiments on a server with 12 Intel Xeon Platinum 8255C 2.50GHz CPU and a single P40 GPU. We consider three transfer tasks: 1) transferring ResNet-18 pre-trained on ImageNet to CIFAR-100 with full data; 2) transferring ResNet-18 pre-trained on ImageNet to CIFAR-100 with $1/10$ data (50 samples per class); 3) transferring ResNet-50 pre-trained on ImageNet to CIFAR-100 with full data.
For task 1), $n=50,000$, $d=512$; for task 2), $n=5,000$, $d=512$; for task 3), $n=50,000$, $d=2048$. The batch size in all experiments is $50$.

We present the results in Table \ref{tab:time}. First of all, we can observe that
the time for fine-tuning a model is about $300$ times of the time for transferability estimation (including the time for feature extraction and the time for computing a transferability measure). Besides, it often requires more than 10 times of fine-tuning to search for the best-performing hyper-parameters. 
Therefore, running a transferability estimation algorithm can achieve $3000\times$ speedup when selecting a pre-trained model and the layer of it to transfer.
This highlights the necessity and importance of developing transferability estimation algorithms. 
Second, though LEEP and NCE computing the similarity between labels only show the highest efficiency, they suffer from the unsatisfactory performance in source selection and the incapability of accommodating unsupervised pre-trained models and layer selection.
Third, amongst all the feature-based transferability measures, TransRate takes the shortest wall-clock time. This indicates its computational efficiency. 
The time costs of both H-Score and LogME are higher than TransRate, which recognizes the necessity of developing an optimization-free estimation algorithm.

\subsection{Sensitivity to Value of $\epsilon$}
\label{appsec:b5}

\kai{
To evaluate the influence of $\epsilon$ on TransRate, we conduct experiments on the toy case presented in Section 3.2 and on the layer selection with the same settings in Section 4.3. We vary the value of $\epsilon$ from $0.01$ to 1E-15 and report the TransRate score and the performance (evaluated by $R_p$, $\tau$, and $\tau_\omega$)  of TransRate under different values of $\epsilon$ in Figure \ref{fig:epsilon}.

We have the following three observations. First, the TransRate scores hardly change when $\epsilon \le$ 1E-3 in the toy case and when $\epsilon \le $ 1E-12 in the layer selection experiment. This verifies the analysis Appendix D.4 that t the value of TransRate does not change for a sufficiently small $\epsilon$. Second, though the value of TransRate scores is still changing, their ranking does change for all $\epsilon$ in Figure \ref{fig:epsilon_1} and for $\epsilon \le $ 1E-3 in Figure \ref{fig:epsilon_2}. Third, we see in Figure \ref{fig:epsilon_3} that the that performance of TransRate remains nearly the same when $\epsilon \le$ 1E-3. 
The second and third observations verify that the value of $\epsilon$ has limited influence to the performance of TransRate.

\begin{figure*}[t]
\centering
    \vspace{-3mm}
    \subfigure[The change trend of transrate score in the toy case.]{\includegraphics[width=0.22\textwidth]{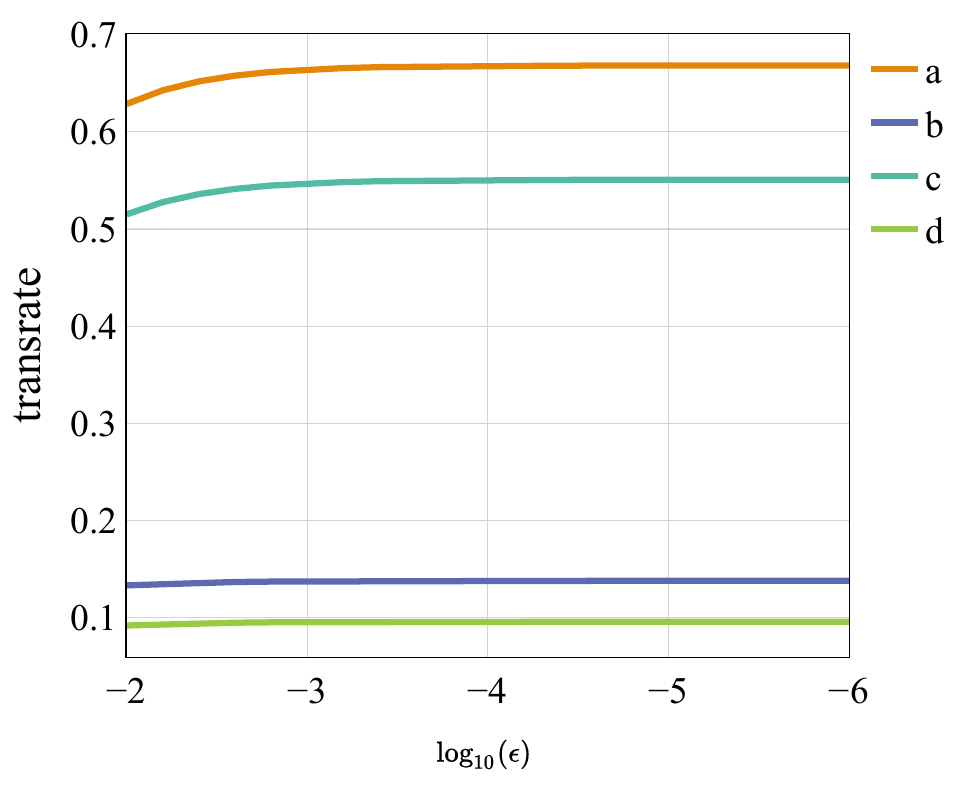} \label{fig:epsilon_1}}
    \subfigure[The change trend of transrate score in layer selection experiment.]{\includegraphics[width=0.22\textwidth]{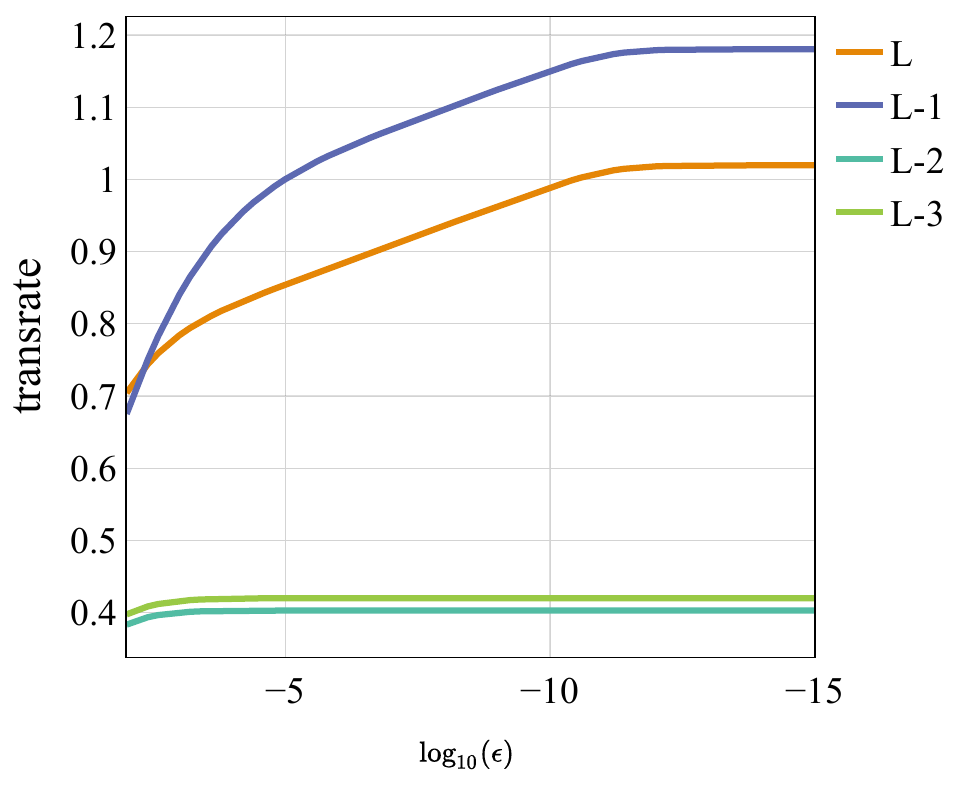}\label{fig:epsilon_2}}
    \subfigure[The change trend of 3 performance measures in layer selection experiment.]{\includegraphics[width=0.22\textwidth]{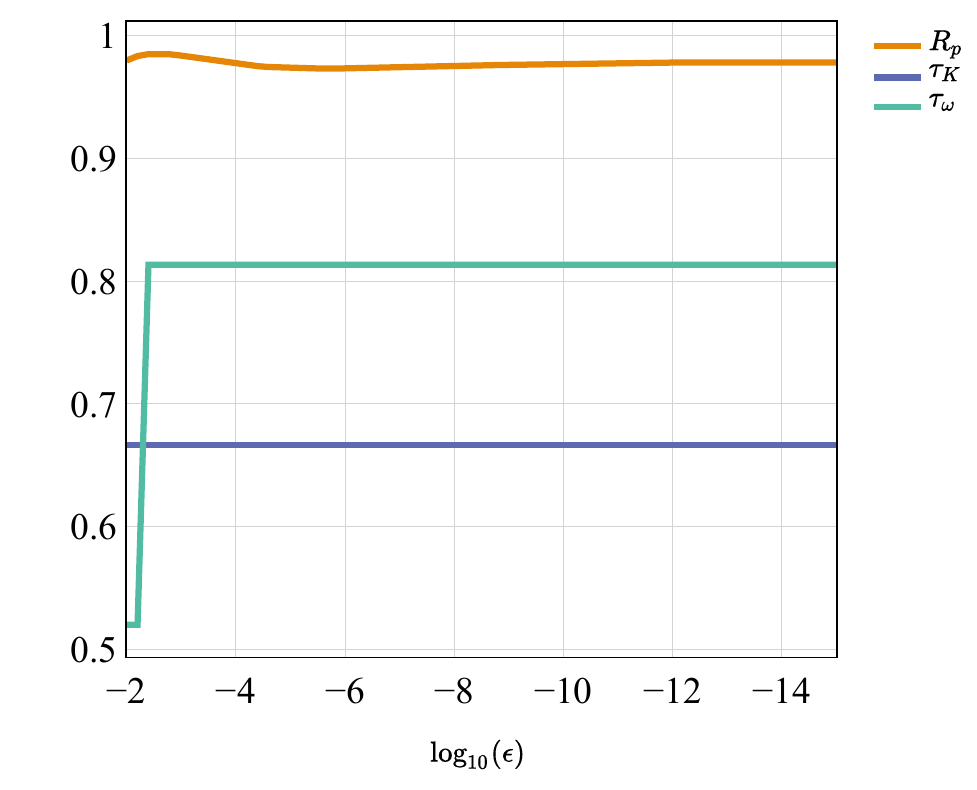} \label{fig:epsilon_3}}
    \caption{ \small 
    Sensitivity analysis of the value of $\epsilon$. The figure in the left shows results on toy example introduced in section 3.2. The figures in middle and right columns show results on a layer selection experiment from a ResNet-18 model pre-trained on Birdsnap to CIFAR-100.}
\label{fig:epsilon}
\end{figure*}

}

\begin{table*}[t]\caption{\small Transferability estimation on transferring from a pre-trained model to different target tasks.}
\small
\centering
\begin{tabular}{c |c|c c c c c c}
\toprule
      Target Tasks
      &
  Measures & NCE & LEEP & LFC & H-Score & LogME & TransRate 
\\
\midrule
  \multirow{3}{*}{ \makecell{Source: ImageNet \\ Model: ResNet-18 \\ Target: CIFAR-100} } 
  & $R_p$ & 0.9810 & 0.9781 & 0.7599 & -0.9597 & -0.7940 & {\bf 0.9841}
\\
    & $\tau_K$ & 0.9467 & 0.9467 & 0.6000 & -0.9467 & -0.6267 & {\bf 0.9733}
\\
    & $\tau_\omega$ & 0.9537 & 0.9537 & 0.7662 & -0.9600 & -0.6104 & {\bf 0.9810}
\\
\midrule
  \multirow{3}{*}{ \makecell{Source: ImageNet \\ Model: ResNet-18 \\ Target: FMNIST} } 
  & $R_p$ & 0.9258 & 0.9246 & 0.8229 & -0.8674 & 0.7200 & {\bf 0.9410}
\\
    & $\tau_K$ & 0.7067 & 0.7067 & {\bf 0.8134} & -0.7067 & 0.6533 & 0.7333
\\
    & $\tau_\omega$ & 0.7578 & 0.7578 & {\bf 0.8893} & -0.7271 & 0.7782 & 0.8122 
\\
\midrule
    \multirow{3}{*}{ \makecell{Source: CIFAR-10 \\ Model: ResNet-20 \\ Target: CIFAR-100} } 
    & $R_p$ & 0.9815 & 0.9819 & 0.5714 & -0.9192 & -0.9188 & {\bf 0.9896}
\\
    & $\tau_K$ & {\bf 1.0} & {\bf 1.0} & 0.3600 & -0.9733 & -0.8133 & {\bf 1.0}     
\\
    & $\tau_\omega$ & {\bf 1.0} & {\bf 1.0} & 0.5320 & -0.9695 & -0.7878 & {\bf 1.0} 
\\
\midrule
  \multirow{3}{*}{ \makecell{Source: CIFAR-10 \\ Model: ResNet-20\\ Target: FMNIST} } 
    & $R_p$ & 0.9622 & {\bf 0.9639} & 0.6410 & -0.8871 & 0.6687 & 0.9449
\\
    & $\tau_K$ & \textbf{0.8400} & \textbf{0.8400} & 0.4667 & -0.8133 & 0.3333 & \textbf{0.8400}
\\
    & $\tau_\omega$ & 0.8579 & 0.8579 & 0.5262 & -0.8267 & 0.5654 & \textbf{0.8797}
\\
\bottomrule
\end{tabular}
\label{tab:class}
\end{table*}

\subsection{Target Selection} 
\label{appsec:b6}

We follow \cite{nguyen2020leep} to also evaluate the correlation of the proposed TransRate and other baselines to the accuracy of transferring a pre-trained model to different target tasks. The target tasks are constructed by sampling different subsets of classes from a target dataset. We consider two target datasets: CIFAR-100 and FMNIST. For CIFAR-100, we construct the target tasks by sampling 2, 5, 10, 25, 50, and 100 classes. For FMNIST, we construct the target tasks by sampling 2, 4, 5, 6, 8, 10 classes. \kai{As shown in Proposition 1, the optimal log-likelihood is linear proportional to TransRate score minusing the entropy of $Y$. In target selection, the entropy of $Y$ can be different for different targets. Hence, we subtract $H(Y)$ from the TransRate in this experiment.} 

The results on transferring a pre-trained ResNet-18 on Imagenet to two target datasets and transferring a pre-trained ResNet-20 on CIFAR-10 to two target datasets are summarized in Table \ref{tab:class}. 

The results in Table \ref{tab:class} show that TransRate outshines the baselines in all 4 experiments. 
\kai{ In both the first and third experiments, TransRate obtains the best performance in all three metrics. In the second experiment, it gets the best $R_p$, although its $\tau_K$ and $\tau_\omega$ are a little bit lower than LFC. In the fourth experimet, TransRate achieves the best $\tau_K$ and $\tau_\omega$ and the third best $R_p$ (0.9449), which is competitive with the best $R_p$ (0.9639).
}
Generally speaking, TransRate is the best among all 6 transferability estimation algorithms. Yet we suggest that the competitors, NCE and LEEP, can be good substitutes if TransRate is not considered.


\section{Theoretical studies of TransRate}
\label{appsec:c}

\kai{

\subsection{Coding Rate and Shannon Entropy of a Quantized Continuous Random Variable} 
\label{appsec:c1}

Rate-distortion function is known as $\epsilon$-Entropy~\cite{binia1974epsilon}, which closely related to informaiton entropy. As introduced in \cite{ma2007segmentation}, coding rate is a regularizer version of rate-distortion function for the Gaussian source $\mathcal{N}(0, ZZ^\top)$. 
So coding rate is also closely related to informaiton entropy. 
Even for non-Gaussian distribution, as presented in Appendix A in \cite{ma2007segmentation}, the coding rate estimates a tight upper bound of the total number of nats needed to encoder the region spanned by the vectors in $Z$ subject to the mean squared error $\epsilon$ in each dimension. So it naturally related to the information in the quantized $Z$.
We would note that the $\epsilon$ in \cite{ma2007segmentation} considers the overall distortion rate across all dimension while in our paper, the $\epsilon$ considers only one dimension. So the $\epsilon$ in our paper equals to $\epsilon^2 /d$ in \cite{ma2007segmentation}.

Notice that the code rate $R(Z, \epsilon)$ would be infinite if $\epsilon \rightarrow 0$. This property matches the property Shannon entropy of a quantized random variable that $H(Z^\Delta) = H(Z) - \log(\Delta)$ (Theorem 8.3.1~\cite{cover1999elements}) would be infinite when $\Delta \rightarrow 0$. Such a property does not hurt the estimation of mutual information as 
\begin{equation*}
    MI(Z; Y) = H(Z) - H(Z|Y) 
    = (H(Z) - \log(\Delta)) - (H(Z|Y) - \log(\Delta)) 
    \approx H(Z^\Delta) - H(Z^\Delta|Y)
    \approx R(Z, \epsilon) - R(Z|Y, \epsilon),
\end{equation*}
where $\Delta = \sqrt{2\pi e \epsilon}$.

\subsection{TransRate Score and Transfer Performance} 
\label{appsec:c2}

In transfer learning, the pre-trained model is optimized by maximizing the log-likelihood, and the accuracy is closely related to the log-likelihood. 
As presented in Proposition 1, the ideal TransRate score highly aligns with the optimal log-likelihood of a task. This indicates that the ideal TransRate is closely related to the transfer performance. 

Notice that the practical TransRate is not an exact estimation of the ideal TransRate. But generally, it is linearly proportional to the ideal TransRate. Together with Proposition 1 and the definition of transferability (in Definition 1), we get that the pratical TransRate is larger than the transferability up to a multiplicative and/or an additive constant and smaller than the transferability up to another multiplicative and/or another additive constant. This means that:
\begin{equation*}
    \text{TrR}_{T_s \rightarrow T_t}(g, \epsilon) \propto \text{Trf}_{T_s \rightarrow T_t}(g). 
\end{equation*}

We also show in Lemma \ref{lem:bound} that the value of TransRate is related to the separability of the data from different classes.
On one hand, TransRate achieves minimal value when the data covariance matrices of all classes are the same. In this case, it is impossible to separate the data from different classes and no classifier can perform better than random guesses. On the other hand, TransRate achieves its maximal value when the data from different classes are independent. In this case, there exists an optimal classifier that can correctly predict the labels of the data from different classes. The upper and lower bound of TransRate show that TransRate is related to the separability of the data, and thus, related to the performance of the optimal classifier.
}

\section{Theoretical Details Omitted in Section 3}
\label{appsec:d}

\subsection{Proof of Proposition 1.}
\label{appsec:d1}

\textbf{Proposition 1.} \emph{Assume the target task has a uniform label distribution, i.e. $p(y=y^c) = \frac{1}{C}$ holds for all $c=1,2, ...,C$. We then have: }
\begin{equation*}
	\mathrm{TrR}_{T_s \rightarrow T_t}(g) 
	- H(Y) \gtrapprox  \mathcal{L}(g, w^*) \gtrapprox \mathrm{TrR}_{T_s \rightarrow T_t}(g)
	- H(Y) -H(Z^\Delta).
\end{equation*}

\begin{proof}
Firstly, we note that 
\begin{equation}\label{eq:1}
\begin{split}
	\trR_{T_s \rightarrow T_t}(g) \kai{=}H(Z) - H(Z|Y) = \text{MI}(Y;Z) \approx H(Z^\Delta) - H(Z^\Delta|Y)
\end{split} 
\end{equation}
As presented in \cite{agakov2004algorithm}, the mutual information has a variational lower bound as $MI(Y;Z) \ge \mathbbm{E}_{Z,Y} \log \frac{Q(z,y)}{P(z) P(y)}$ for a variational distribution $Q$.
Following \cite{qin2019rethinking}, we choose $Q$ as 
\begin{equation}
	Q(z,y) = P(z)P(y)\frac{y \exp(w^*(z))}{\mathbbm{E}_{y'} {y'}\exp(w^*(z))},
\end{equation}
where $w^*(z)$ is the output of optimal classifier before softmax, $y$ is the one-hot label of $z$, $y'$ is any possible one-hot label. 
If $p(y' =y^c) = \frac{1}{C}$ holds for all $c=1,2, ...,C$, we have
\begin{equation}\label{eq:2}
\begin{split}
	\text{MI}(Y;Z) & \ge \mathbbm{E}_{Z,Y} \log \frac{y \exp(w^*(z))}{\mathbbm{E}_{y'} y' \exp(w^*(z))}
	\\
	& = \mathbbm{E}_{Z,Y} \log \frac{y \exp(w^*(z))}{\frac{1}{C} \sum_{c=1}^C y^c \exp(w^*(z))}
	\\
	& = \mathbbm{E}_{Z,Y} \log \frac{y \exp(w^*(z))}{\sum_{c=1}^C y^c \exp(w^*(z))} - \log(\frac{1}{C})
	\\
	& \approx \mathcal{L}(g, w^*) + \log(C)
	\\
	& =  \mathcal{L}(g, w^*) + H(Y)
\end{split}
\end{equation}
The last equality comes from the definition of the negative log-likelihood loss $\mathcal{L}(g,w^*)$, which is an empirical estimation of $\mathbbm{E}_{Z,Y} \log \frac{y \exp(w^*(z))}{\sum_{c=1}^C y^c \exp(w^*(z))}$. 
Combining Eqs.\eqref{eq:1} and \eqref{eq:2}, we have the first inequality in Proposition \ref{prop:1}.

For proving the second inequality, we consider a classifier $\bar{w}$ that predicts the label $y$ for any data by $p(y) = \int_{z\in Z} p(y|z) p(z) dz$. The the empirical loss computed on this classifier is
\begin{equation}
\begin{split}
    \mathcal{L}(g, \bar{w}) = \frac{1}{n} \sum_{i=1}^n \log p(y_i) & = \frac{1}{n} \sum_{i=1}^n \log \left( \int_{z\in Z} p(y_i|z) p(z) \right) dz \\
    & \ge \frac{1}{n} 
    \sum_{i=1}^n
    \log p(y_i|z_i)
    + \frac{1}{n} \sum_{i=1}^n \log p(z_i) 
\end{split}
\end{equation}
The inequality comes from replacing the integral by one of its elements.
It is easy to verify that the first term is an empirical estimation of $-H(Y|Z^\Delta)$ and the second term is an empirical estimation of $-H(Z^\Delta)$.
By the definition of $\mathcal{L}(g, w^*)$, we have 
\begin{equation}\label{eq:prop2:RHS}
    \mathcal{L}(g, w^*) \ge \mathcal{L}(g, \bar{w}) \gtrapprox -H(Y|Z^\Delta) - H(Z^\Delta) = \mathrm{TrR}_{T_s \rightarrow T_t}(g ) - H(Y) - H(Z^\Delta).
\end{equation}
Combining Eqns. \eqref{eq:2} and \eqref{eq:prop2:RHS}, we complete the proof.

\end{proof}

\kai{
\textbf{Remark: } Since the maximal log-likelihood is a variational form of mutual information between inputs and labels, the gap between the $MI(Y; Z)$ and the optimal log-likelihood is small. That is Eq. \eqref{eq:2} is a tight upper bound of the optimal log-likelihood. The lower bound is proved by constructing a classifier $\bar{w}$ without considering the feature. Such a classifer generally does not performs well in practice. The performance gap between the optimal classifier $w^*$ and $\bar{w}$ is larger. So the lower bound is loose. The lower bounds provided in NCE~\cite{tran2019transferability} and LEEP~\cite{nguyen2020leep} are also proved through a similar technique. This means the lower bounds in our paper and in NCE, LEEP are all loose.
But only \ours is proved to be a tight upper bound of the maximal log-likelihood. 
}

\subsection{Properties of Coding Rate and TransRate Score}
\label{appsec:d2}

In this part, we discuss the properties of \kai{coding rate} and TransRate Score.

\begin{lemma}\label{lem:commutative}
    For any $\hat{Z} \in \mathbbm{R}^{d \times n}$, we have
    \begin{equation*}
        R(\hat{Z}, \epsilon) = \frac{1}{2} \logdet (I_d + \frac{1}{n \epsilon} \hat{Z} \hat{Z}^\top )
        = \frac{1}{2} \logdet (I_n + \frac{1}{n \epsilon} \hat{Z}^\top \hat{Z})
    \end{equation*}
\end{lemma}

Lemma \ref{lem:commutative} presents the commutative property of the coding rate, which is known in \cite{ma2007segmentation}. Based on this lemma, we can reduce the complexity of $\logdet(\cdot)$ computation in $R(\hat{Z}, \epsilon)$ from $\mathcal{O}(d^{2.373})$ to $\mathcal{O}(n^{2.373})$ if $n < d$.

\begin{lemma}\label{lem:invariant}
    For any $\hat{Z} \in \mathbbm{R}^{d \times n}$ having 
    $r$ singular values, denoted by $\sigma_1$, $\sigma_2$, .., $\sigma_r$, we have
    \begin{equation*}
        R(\hat{Z}, \epsilon) = \frac{1}{2} \sum_{i=1}^r
        \log(1 + \frac{1}{n \epsilon} \sigma_i^2) 
    \end{equation*}
\end{lemma}

Lemma \ref{lem:invariant} is an inference from the invariant property of the coding rate presented in \cite{ma2007segmentation}. Both lemmas are inferred from Sylvester's determinant theorem.

\begin{lemma}\label{lem:bound} (Upper and lower bounds of TransRate)
    For any $\hat{Z}^c \in \mathbbm{R}^{d \times n_c}$ for $c = 1, 2, ..., C$, let \kai{$\hat{Z} = [\hat{Z}^1, \hat{Z}^2, ..., \hat{Z}^C]$ which is the concatenation of all $\hat{Z}^c $.} We then have
    \begin{equation*}
        \mathrm{TrR}_{T_s \rightarrow T_t}(g,\epsilon) = R(\hat{Z}, \epsilon) - \sum_{c=1}^C \frac{n_c}{n} R(\hat{Z}^c, \epsilon) \ge 0 
    \end{equation*}
    The equality holds when $\frac{\hat{Z} \hat{Z} ^\top}{n} = \frac{\hat{Z}^c ({\hat{Z}^c})^\top}{n_c}$ for all $c$.
    \begin{equation*}
    \begin{split}
        \mathrm{TrR}_{T_s \rightarrow T_t}(g,\epsilon) & = R(\hat{Z}, \epsilon) - \sum_{c=1}^C \frac{n_c}{n} R(\hat{Z}^c, \epsilon) \\
        & \le \frac{1}{2} \sum_{c=1}^{C}   \left( \logdet (I_d + \frac{1}{n \epsilon} \hat{Z}^c ({\hat{Z}^c})^\top ) - \frac{n_c}{n} \logdet (I_d + \frac{1}{n_c \epsilon} \hat{Z}^c ({\hat{Z}^c})^\top ) \right)
    \end{split}
    \end{equation*}
    The equality holds when $ \hat{Z}^{c_1} ({\hat{Z}^{c_2}})^\top = 0$ for all $1 \le c_1 < c_2  \le C$.
\end{lemma}

The proof follows the upper and lower bound of $R(\hat{Z}, \epsilon)$ in Lemma A.4 of \cite{yu2020learning}.

\subsection{Proof of the toy case in Section 3.2}
\label{appsec:d3}

In the end of Section 3.2, we present a toy case of a binary classification problem with $\hat{Z} = [\hat{Z}^1,\hat{Z}^2]$ where $\hat{Z}^1 \in \mathbbm{R}^{d \times n/2}$ and $\hat{Z}^2 \in \mathbbm{R}^{d \times n/2}$. The lower and upper bound of the TransRate in this toy case can be derived by Lemma \ref{lem:bound}. But in Sec. 3.2, we derived the bounds in another way. Here, we provide details of the derivation.

By Lemma \ref{lem:commutative}, we have
\begin{align}\label{eq:toy}
    R(\hat{Z}, \epsilon)  = & \frac{1}{2}\logdet(I_{2n} + \alpha \hat{Z} ^\top \hat{Z})
    \nonumber\\
     = & \frac{1}{2}\logdet\left(I_{n} + \alpha 
    \begin{bmatrix} (\hat{Z}^1)^\top \hat{Z}^1 & (\hat{Z}^1)^\top \hat{Z}^2 \\ (\hat{Z}^2)^\top \hat{Z}^1 & (\hat{Z}^2)^\top \hat{Z}^2 \end{bmatrix}
    \right)
    \nonumber\\
      =  & \frac{1}{2}\log\det \Big\{(I_{n/2} +\alpha (\hat{Z}^1)^\top \hat{Z}^1 +\alpha (\hat{Z}^2)^\top \hat{Z}^2) \nonumber \\
     &+\alpha \left[(\hat{Z}^1)^\top \hat{Z}^1 (\hat{Z}^2)^\top \hat{Z}^2 - (\hat{Z}^1)^\top \hat{Z}^2 (\hat{Z}^2)^\top \hat{Z}^1\right] \Big\}
    \nonumber\\
     \ge &  \frac{1}{2}\log\det \left \{(I_{n/2} +\alpha (\hat{Z}^1)^\top \hat{Z}^1 +\alpha (\hat{Z}^2)^\top \hat{Z}^2) \right \}  \nonumber \\
     &+ \frac{\alpha}{2} \logdet \left\{ (\hat{Z}^1)^\top \hat{Z}^1 (\hat{Z}^2)^\top \hat{Z}^2 - (\hat{Z}^1)^\top \hat{Z}^2 (\hat{Z}^2)^\top \hat{Z}^1 \right \}
\end{align}
The first equality comes from Lemma \ref{lem:commutative}; the third equality follows the property of matrix determinant that for any square matrices $A, B, C, D$ with the same size, i.e.,
\begin{equation*}
    \det \begin{pmatrix} A & B \\ C & D \end{pmatrix} = \det(AD - BC).
\end{equation*}
The inequality follows that for any positive definite symmetric matrices $A$ and $B$, $\det(A + B) \ge \det(A) + \det(B)$. 

From Eqn. \eqref{eq:toy}, we know that when $(\hat{Z}^1)^\top \hat{Z}^1$ and $(\hat{Z}^2)^\top \hat{Z}^2$ are fixed, the lower bound of TransRate is related to $\logdet \left\{ \left[(\hat{Z}^1)^\top \hat{Z}^1 (\hat{Z}^2)^\top \hat{Z}^2 - (\hat{Z}^1)^\top \hat{Z}^2 (\hat{Z}^2)^\top \hat{Z}^1\right] \right \}$. The value of this term is determined by $(\hat{Z}^1)^\top \hat{Z}^2$. So the value of TransRate is also determined by $(\hat{Z}^1)^\top \hat{Z}^2$, i.e., the overlap between the two classes. When $\hat{Z}^1$ and $\hat{Z}^2$ are completely the same, this term becomes zero and TransRate achieves its minimal value as $\frac{1}{2}\log\det \left \{(I_{n/2} +\alpha (\hat{Z}^1)^\top \hat{Z}^1 +\alpha (\hat{Z}^2)^\top \hat{Z}^2) \right \}$.
When $\hat{Z}^1$ and $\hat{Z}^2$ are independent, this term achieves its maximal value as $\logdet \left ( (\hat{Z}^1)^\top \hat{Z}^1 (\hat{Z}^2)^\top \hat{Z}^2\right)$ and TransRate achieves its maximal value as well.

\subsection{The Influence of $\epsilon$}
\label{appsec:d4}

In Section 3.2, we mention that the value of $\epsilon$ has minimal influence on the performance of TransRate. Here we provide more details about the influence of the choice of $\epsilon$ on the value of TransRate.

Assume that we scale $\epsilon$ by a positive scaler $\alpha$. After scaling, the value of $R(\hat{Z}, \alpha \epsilon)$ and $R(\hat{Z}, \alpha \epsilon |Y)$ is different from $R(\hat{Z}, \epsilon)$ and $R(\hat{Z}, \epsilon |Y)$. 
By Lemma \ref{lem:invariant}, we have
\begin{equation*}
    R(\hat{Z}, \alpha \epsilon) = \frac{1}{2} \sum_{i=1}^r
    \log(1 + \frac{1}{n \alpha \epsilon} \sigma_i^2).
\end{equation*}
If $\frac{1}{n \alpha \epsilon} \sigma_i^2 \gg 1$, which holds for a sufficiently small $\epsilon$, we have $\log(1 + \frac{1}{n \alpha \epsilon} \sigma_i^2) \approx \log(\frac{1}{n \alpha \epsilon} \sigma_i^2) =  \log(\frac{1}{n \epsilon} \sigma_i^2) -2\log(\alpha)$. Then for sufficiently small $\epsilon$, $\alpha$ and sufficiently large $\sigma_i$, we have $R(\hat{Z}, \alpha 
\epsilon) \approx R(\hat{Z}, 
\epsilon) - r \log(\alpha)$ and $R(\hat{Z}, \alpha \epsilon |Y ) \approx R(\hat{Z}, \epsilon | Y) - r \log(\alpha)$. Therefore, the influence of $\alpha$ is nearly canceled in calculating TransRate as we subtract $R(\hat{Z}, \alpha \epsilon)$ by $R(\hat{Z}, \alpha \epsilon |Y)$. 
Though this assumption may not hold in practice, we further verify the influence empirically in Appendix \ref{appsec:b5}.



\subsection{Time Complexity}
\label{appsec:d6}

When $d<n$ and $d<n_c$, the computational costs of $R(\hat{Z}, \epsilon)$ and $R(\hat{Z}, \epsilon | Y)$ are $\mathcal{O}(d^{2.373} + nd^2)$ and $\mathcal{O}(C d^{2.373} + C n_c d^2)$, respectively. Thus, the total computation cost of TransRate is $\mathcal{O}((C+1)d^{2.373} + 2nd^2)$. By Lemma \ref{lem:commutative}, when $n<d$ or $n_c <d$, we can further reduce the cost of computing $R(\hat{Z}, \epsilon)$ or $R(\hat{Z}, \epsilon | Y)$. Besides, we can implement the computation of $R(\hat{Z}, \epsilon)$ and $R(\hat{Z}^c, \epsilon)$ in parallel, so that we can reduce the computation cost of TransRate to $\mathcal{O}(\min\{d^{2.373} + nd^2, n^{2.373} + dn^2 \})$.

\end{document}